%% file: main.tex
\documentclass[11pt]{article}

\usepackage[in]{fullpage}

\usepackage{times}
\usepackage{subfigure} 

\usepackage{algorithm}
\usepackage{algorithmic}

\usepackage[cmex10]{amsmath}
\usepackage[mathscr]{euscript}
\usepackage{bm}
\usepackage{amsfonts}

\usepackage{enumitem}

\usepackage{mathtools}

\usepackage{multirow}

\usepackage{hyperref}

\newcommand{\BlackBox}{\rule{1.5ex}{1.5ex}}  % end of proof
\newenvironment{proof}{\par\noindent{\bf Proof\ }}{\hfill\BlackBox\\[2mm]}
\newenvironment{proof2}{\par\noindent{\bf Proof of Theorem 2\ }}{\hfill\BlackBox\\[2mm]}
\newenvironment{proof4}{\par\noindent{\bf Proof of Theorem 4\ }}{\hfill\BlackBox\\[2mm]}
 
\newtheorem{theorem}{Theorem}
\newtheorem{lemma}[theorem]{Lemma}

\newtheorem{corollary}[theorem]{Corollary}

\newcommand{\subscript}[2]{$#1#2$}

\usepackage{cases}

\begin{document}

\title{A Spectral Algorithm for Inference in Hidden semi-Markov Models}

\author{Igor Melnyk \\
Dept of Computer Science \& Engineering\\
University of Minnesota, Twin Cities\\
melnyk@cs.umn.edu
\and
Arindam Banerjee\\
Dept of Computer Science \& Engineering\\
University of Minnesota, Twin Cities\\
banerjee@cs.umn.edu}

\providecommand{\keywords}[1]{\textbf{Keywords:} #1}

\maketitle

\begin{abstract} 

Hidden semi-Markov models (HSMMs) are latent variable models which allow latent state persistence and can be viewed as a generalization of the popular hidden Markov models (HMMs).
In this paper, we introduce a novel spectral algorithm to perform inference in HSMMs. Unlike expectation maximization (EM), our approach correctly estimates the probability of given observation sequence based on a set of training sequences.
Our approach is based on estimating moments from the sample, whose number of dimensions depends only logarithmically on the maximum length of the hidden state persistence. Moreover, the algorithm requires only a few matrix inversions and is therefore computationally efficient. Empirical evaluations on synthetic and real data demonstrate the advantage of the algorithm over EM in terms of speed and accuracy, especially for large datasets.
\end{abstract} 

\section{Introduction}
\label{sec:introduction}
\input{Introduction}

\section{Notation and Preliminaries}
\label{sec:notation}
\input{Notations}

\section{Problem Formulation}
\label{sec:problem}
\input{Formulation}

\section{Spectral Algorithm for Inference in HSMM}
\label{sec:tensorForm}
\input{TensorForm}

\subsection{Estimation of Observable Tensors}
\label{sec:obsTensors}
\input{ObservableRepresentation}

\subsection{Basic Version of Spectral Algorithm}
\label{sec:specAlgorithm}
\input{basicSpectralAlg}

\subsection{Efficient Version of Spectral Algorithm}
\label{sec:effcAlgorithm}
\input{effSpectralAlg}

\section{Rank Analysis of Observable Tensors}
\label{sec:observationVars}
\input{FullRankStructure}

\section{Experiments}
\label{sec:experiments}
\input{Experiments}

\section{Conclusion}
\label{sec:conclusion}
\input{Conclusion}

%\section*{Acknowledgments}
%{\bf Acknowledgements:} This work was supported by NASA grant NNX12AQ39A, and NSF Grants IIS-0953274, IIS-1029711, IIS-0916750. We thank %Nikunj Oza and Bryan Matthews at NASA for their helpful comments and suggestions, and computing support from the Minnesota %Supercomputing Institute (MSI). A.~B.~acknowledges support from IBM and Yahoo.

%\newpage
\section*{Appendix}

\appendix
\section{Analysis of Tensor Rank Structure}
\label{sec:proofs}
\input{proofs}

\section{Initial and Final Parts of HSMM}
\label{sec:appendix}
\input{startEndModel}

\vspace*{3mm}

{\bf Acknowledgements:} This work was supported by NASA grant NNX12AQ39A, and NSF Grants IIS-0953274, IIS-1029711, IIS-0916750. We thank Nikunj Oza and Bryan Matthews at NASA for their helpful comments and suggestions, and computing support from the Minnesota Supercomputing Institute (MSI). A.~B.~acknowledges support from IBM and Yahoo.

\bibliographystyle{abbrv}
\bibliography{main}

\end{document}

%% file: Introduction.tex
Hidden semi-Markov models (HSMMs) are discrete latent variable models which allow temporal persistence of latent states, and can be viewed as a generalization of the popular hidden Markov models (HMMs)~\cite{chiappa14,murphy02, yu10}. In HSMMs, the stochastic model for the unobservable process is defined by a semi-Markov chain: latent state at the next time step is determined by the current latent state as well as time elapsed since the entry into the current state. 
The ability to flexibly model such latent state persistence turns out to be useful in many application areas, including anomaly detection \cite{tan08, xie09}, activity recognition \cite{van10}, and speech synthesis \cite{ze07}. Such state persistence is in contrast to HMMs, which use a Markov chain over latent state transitions and hence have an implicit geometric distribution for the state duration \cite{rabiner89}.

Given a set of training sequences, one can formulate two distinct but related problems: \emph{learning}, i.e., estimating model parameters and \emph{inference}, i.e., computing the probability of an observed and/or latent variable sequence. The methods proposed for learning HSMMs usually follow the initial idea due to Rabiner \cite{rabiner89} based on the modifications of the Baum-Welch algorithm \cite{baum66}, which are all variants of the expectation maximization (EM) framework, presented in \cite{delr77}. Once the parameters are estimated, we can then perform inference using, e.g., the forward-backward algorithm of \cite{yu03}. However, since EM, in general, has no guarantees in estimating the parameters correctly and can suffer from slow convergence, such methods can be inefficient and/or inconsistent.

Approaches based on hierarchical Dirichlet processes have also been proposed for HMMs \cite{fsjw08} and HSMMs \cite{jowi13}, which are the nonparametric Bayesian models avoiding the need to specify the size of the latent space and learn it from data. However, in practice, the accuracy of such algorithms is usually sensitive to initialization and may suffer from slow convergence.

In recent years, there has been an increased interest in spectral algorithms, which provide computationally efficient, local-minimum-free, provably consistent inference and/or parameter estimation algorithms for latent variable models. For example, \cite{anima13, aghk12b, anandkumar13} have proposed spectral methods for learning the parameters of a wide class of tree-structured latent graphical models, including Gaussian mixture models, topic models, and latent Dirichlet allocation. The main idea is based on a tensor decomposition of certain low order moments, computable directly from data, in order to extract the model parameters. 

In many problems, however, the end goal is not the recovery of model parameters but the statistical inference, in which case the parameter estimation step is unnecessary. 
%Therefore, we would want to avoid the parameter estimation step and directly compute the quantity of interest. 
In this regard, \cite{hskz12} have proposed an efficient spectral algorithm for inference in HMMs. It is based on the idea of expressing the probability of the observed sequence in a representation, which does not depend on the model parameters and uses easily computable second and third order sample moments to perform inference. However, their approach was specific to HMMs and not easily extendable to other latent variable graphical models. \cite{pasx11} then introduced a spectral algorithm to perform inference in latent tree graphical models with arbitrary topology, and later in \cite{psit12} a general spectral inference framework for latent junction trees. 

In this paper, we utilize the framework of \cite{psit12} and introduce a novel spectral algorithm for inference in HSMMs. Since we address a more specific problem than \cite{psit12}, our results shed  more light into the details of the spectral framework for HSMMs, allow for a sharper analysis, and yield a significantly more efficient algorithm than the general framework in \cite{psit12}. There are two main technical contributions in this work:
\begin{itemize}
\item By exploiting the {\em homogeneity} of HSMMs we make our algorithm more efficient and accurate than an algorithm, which directly follows from the recipe in \cite{psit12} for general graphs. In particular, our approach ensures that the number of matrix multiplications and inverses needed to estimate the probability of an observed sequence is fixed and independent of sequence length. 
%   
%Moreover, the training step of the proposed algorithm requires a fixed number of spectral decompositions of the estimated tensors, significantly reducing the computational complexity over the general algorithm in \cite{parikh12}.
\item We show that the number of dimensions in the sample moments (represented as a multidimensional matrix or a tensor) in estimated observable representation depends only {\em logarithmically} on the maximum length of latent state persistence. 
\end{itemize}
In experiments, comparing our method with EM on both synthetic and real datasets, two observations stand out: (i) the spectral method gets similar or better performance than EM as the number of samples increases, and (ii) the spectral method is orders of magnitude faster than EM for the datasets we consider.

Few remarks are in order about the proposed algorithm. Note that our method does not estimate model parameters explicitly but rather learns alternative representation to perform inference on observable variables. Moreover, our formulation cannot be directly used to infer hidden states, although methods such as in \cite{mossel05} can be potentially utilized to recover original HSMM parameters from the learned representation.
%on synthetic and real data, illustrating the accuracy and efficiency of our algorithm. Interestingly, while the spectral and EM algorithms give comparable performance in terms of inference, the spectral algorithm,  somewhat surprisingly, turns out to be orders of magnitude faster than EM for the problems we consider.

The rest of the paper is organized as follows: We introduce notation in Section~\ref{sec:notation}. In Section~\ref{sec:problem}, we present HSMM inference from a tensor product perspective and in Section~\ref{sec:tensorForm} introduce the spectral algorithm for inference. In Section~\ref{sec:observationVars}, we present a careful technical analysis to establish logarithmic dependence of the number of modes in the tensor on maximum latent state persistence. We present experimental results in Section~\ref{sec:experiments} and conclude in Section~\ref{sec:conclusion}.

%% file: Notations.tex
%One of the main steps in the derivation of the spectral algorithm is to express the model parameters in tensor notations in order to represent the HSMM in the observable form. 
In this section, we cover basic facts about tensor algebra, a detailed tutorial on tensors can be found in  \cite{kiers00} or \cite{koba09}. A tensor is defined as a multidimensional array of data, which will be denoted by boldface Euler script letters, e.g., $\underset{m_1,\ldots,m_N}{\bm{\mathscr{X}}} \in \mathbb{R}^{I_{m_1}\times \cdots \times I_{m_N}}$, which is $N$-mode tensor of dimensions $I_{m_1}\times \cdots \times I_{m_N}$. A specific mode is denoted by the subscript variable $m_i$, whose dimension is $I_{m_i}$.

Any tensor can be matrisized (or flattened) into a matrix. This mapping can be done in multiple ways, the only requirement is that the number of elements is preserved and the mapping is one-to-one. If we split the modes into two disjoint sets, one corresponding to rows and the other to columns, e.g., $\{m_1,\ldots,m_N\} = \{p_1,\ldots,p_K\}\cup\{q_1,\ldots,q_L\}$, then a matrisization of $\bm{\mathscr{X}}$ is denoted by a corresponding capital boldface letter, e.g.,  $\underset{p_1,\ldots,p_K q_1,\ldots,q_L}{\mathbf{X}} \in \mathbb{R}^{I_{p_1}\cdots I_{p_K}\times I_{q_1}\cdots I_{q_L}}$.

\textbf{Tensor Multiplication} Multiplication of two tensors is performed along specific modes. For this, we flatten each tensor to a matrix, perform the usual matrix multiplication and  transform the result back to a tensor. The multiplication is denoted by a symbol $\times$ with an optional subscript representing the modes along which the operation is performed, e.g.,: 
\begin{align*}
\underset{p_1,\ldots,p_K,r_1,\ldots,r_M}{\bm{\mathscr{Z}}} = \hspace{-3pt} \underset{p_1,\ldots,p_K,q_1,\ldots,q_L}{\bm{\mathscr{X}}}\hspace{-3pt}\times_{q_1,\ldots,q_L}\hspace{-1pt}\underset{q_1,\ldots,q_L,r_1,\ldots,r_M}{\bm{\mathscr{Y}}},
\end{align*}
\noindent where $\underset{q_1,\ldots,q_L,r_1,\ldots,r_M}{\bm{\mathscr{Y}}} \in \mathbb{R}^{I_{q_1}\times\cdots \times I_{q_L}\times I_{r_1}\times \cdots \times I_{r_M}}$ and the resulting tensor on the left hand side is of the form $ \underset{p_1,\ldots,p_K,r_1,\ldots,r_M}{\bm{\mathscr{Z}}} \in \mathbb{R}^{I_{p_1}\times\cdots \times I_{p_K}\times I_{r_1}\times \cdots \times I_{r_M}}$. Observe that in the above, we can flatten the tensors $\bm{\mathscr{X}}$ and $\bm{\mathscr{Y}}$ in multiple different ways as long as the matrix multiplication remains valid. For example, we could assign the multiplication modes in both tensors to columns, in this case the matrix product becomes $\mathbf{Z} = \mathbf{X}\mathbf{Y}^{T}$. Alternatively, the tensor  $\bm{\mathscr{Y}}$ could be matrisized with the multiplication modes corresponding to rows, resulting in the product $\mathbf{Z} = \mathbf{X}\mathbf{Y}$.

An important fact about tensor multiplication is that in a series of tensor multiplications the order is irrelevant as long as the multiplication is performed along the matching modes, e.g, 

\begin{align*}
\underset{sp}{\bm{\mathscr{X}}}\times_{s}\left(\underset{tr}{\bm{\mathscr{Y}}}\times_{r}\underset{rs}{\bm{\mathscr{Z}}}\right) = \left(\underset{sp}{\bm{\mathscr{X}}}\times_{s}\underset{rs}{\bm{\mathscr{Z}}}\right)\times_{r}\underset{tr}{\bm{\mathscr{Y}}}.
\end{align*}
\noindent If we let the matrisized tensors to be $\mathbf{X} \in \mathbb{R}^{I_{p}\times I_{s}}$, $\mathbf{Y} \in \mathbb{R}^{I_{t}\times I_{r}}$ and $\mathbf{Z} \in \mathbb{R}^{I_{r}\times I_{s}}$, then the above can be verified to be true since
\begin{align*}
\mathbf{X}\left(\mathbf{Y}\mathbf{Z}\right) = \left(\mathbf{X}\mathbf{Z}^{T}\right)\mathbf{Y}^T.
\end{align*}

Note that to reduce clutter, in many places we will drop the multiplication subscripts. The implied modes of multiplication can then be inferred from the subscripts of the tensors. Specifically, when two tensors are multiplied, we first check their modes and then multiply along the modes which are common to both of them. For example, in the product $\underset{pqr}{\bm{\mathscr{X}}}\times\underset{qsr}{\bm{\mathscr{Y}}}$, the implied multiplication is performed along the common modes, i.e., $q$ and $r$.

\textbf{Tensor Inversion} We also discuss the operation of tensor inversion. Tensor inverse $\bm{\mathscr{X}}^{-1}$ is always defined with respect to a certain subset of modes and can be written as follows:
\begin{align*}
\underset{p_1,\ldots,p_K,q_1,\ldots,q_L}{\bm{\mathscr{X}}}\hspace{-3pt}\hspace{-2pt}\times_{q_1,\ldots,q_L}\hspace{-1pt}\underset{p_1,\ldots,p_K,q_1,\ldots,q_L}{\bm{\mathscr{X}}^{-1}} \hspace{-2pt}=\hspace{-2pt} \underset{p_1,\ldots,p_K,p_1,\ldots,p_K}{\bm{\mathscr{I}}},
\end{align*}
\noindent where the inversion is performed along the modes $q_1,\ldots,q_L$, and $\underset{p_1,\ldots,p_K,p_1,\ldots,p_K}{\bm{\mathscr{I}}}$ denotes an identity tensor, whose elements are everywhere zero, except $\bm{\mathscr{I}}(i_1,\ldots,i_K,i_1,\ldots,i_K)=1$. To perform inversion, we first convert tensor to a matrix, i.e., matrisize tensor. If the modes to be inverted along are associated with columns of the matrix, we compute the right matrix inverse, so that these modes get eliminated after the product. Otherwise, if those modes associated with rows, we compute left matrix inverse. Obviously, for the full rank square matrices both choices would produce the same result.  For example, in the above equation the matrisized tensor might be of the form $\underset{p_1,\ldots,p_K q_1,\ldots,q_L}{\mathbf{X}} \in \mathbb{R}^{I_{p_1}\cdots I_{p_K}\times I_{q_1}\cdots I_{q_L}}$, therefore, we would compute the right matrix inverse so that the modes $q_1,\ldots,q_L$ are eliminated. If the matrisized $\mathbf{X}$ has full row rank, then the inverse can be computed, otherwise we could only compute its pseudo-inverse. Tensorizing the matrix $\mathbf{X}^{-1}$ gives us the desired tensor inverse.

\textbf{Mode Duplication} Observe that in the above, the tensor $\underset{p_1,\ldots,p_K,p_1,\ldots,p_K}{\bm{\mathscr{I}}}$  has duplicate modes. In general, if a tensor has duplicate modes, the corresponding sub-tensor can be interpreted as a hyper-diagonal. For example, if for a tensor $\underset{pq}{\bm{\mathscr{X}}}$ we construct a tensor $\underset{pppq}{\overline{\bm{\mathscr{X}}}}$, which has its mode $p$ duplicated three times, then for a fixed index $i$, the sub-tensor $\overline{\bm{\mathscr{X}}}(:,:,:,i)$ is a hypercube with elements $\bm{\mathscr{X}}(:,i)$ on the diagonal. 

Mode duplication enables us to multiply several tensors along the same mode. For example, if we need to multiply tensors $\underset{sp}{\bm{\mathscr{X}}}$, $\underset{pr}{\bm{\mathscr{Y}}}$ and $\underset{tp}{\bm{\mathscr{Z}}}$ along the mode $p$, then a simple product of the form 
\begin{align*}
\underset{sp}{\bm{\mathscr{X}}}\times_{p}\underset{pr}{\bm{\mathscr{Y}}}\times_{p}\underset{tp}{\bm{\mathscr{Z}}}
\end{align*}
\noindent  cannot be done since any product of two tensors along the mode $p$ would eliminate it, preventing any further multiplications. In general, if there are $N$ multiplications along the specific mode, then there are must be cumulatively $2N$ number of times such a mode is encountered in the participating tensors. In our example, we might duplicate the mode $p$ in, say, tensor ${\bm{\mathscr{Z}}}$ to have 
\begin{align*}
\underset{sp}{\bm{\mathscr{X}}}\times_{p}\left(\underset{pr}{\bm{\mathscr{Y}}}\times_{p}\underset{tpp}{\bm{\mathscr{Z}}}\right),
\end{align*}
so that there are two multiplications over mode $p$ and cumulatively there are four times such a mode is encountered in the participating tensors.
To reduce clutter, we sometimes do not explicitly show the duplicated variables in the subscripts; the implied mode repetition will be evident from the context or explicitly stated in cases when there is a confusion. For example, the identity tensor will often be written as $\underset{p_1,\ldots,p_K}{\bm{\mathscr{I}}}$.

%% file: Formulation.tex
\begin{figure}[!t]
\centering
   \includegraphics[width=0.35\textwidth]{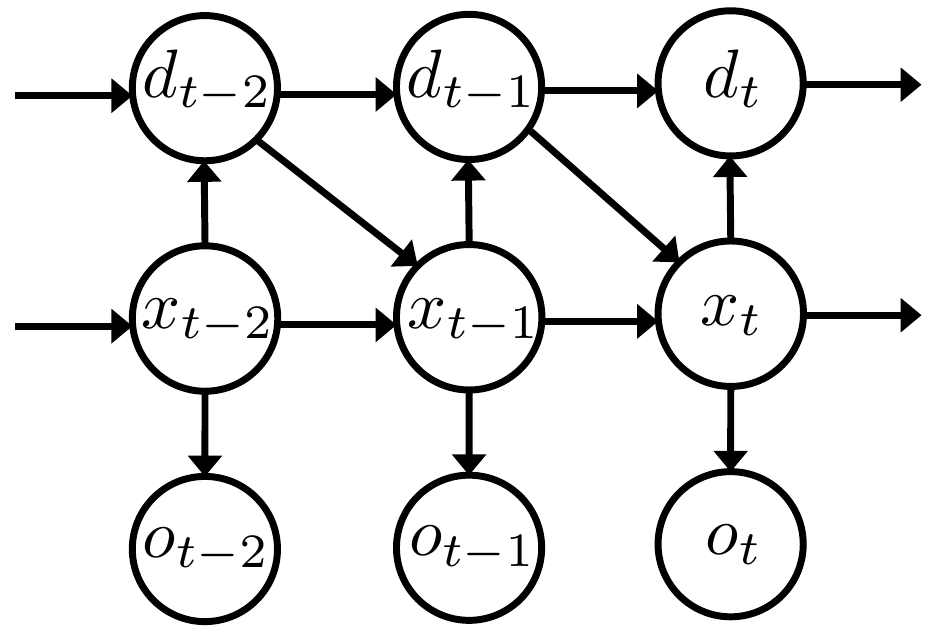}
   \caption{Hidden Semi-Markov Model (HSMM). Here $o_t$ denotes an observation at time step $t$, $x_t$ is a latent state and $d_t$ is the length of state persistence at time step $t$. See text for more details.}
   \label{fig:hsmm}
\end{figure}

In this paper, we consider the problem of inference in HSMM\footnote{Note: to reduce clutter, in the main part of the paper we only consider the part of the model for a general time stamp $t$ and ignore the initial and final steps of the model, whose representation differs slightly from what is shown in Figure~\ref{fig:hsmm}. The details for these parts are presented separately in Appendix \ref{sec:appendix}.} (see Figure~\ref{fig:hsmm}).
Unlike the popular HMM, which has a geometric probability for state persistence, i.e., the probability of persisting in the same state over $t$ time steps decreases as $p^t$, where $p$ is the probability of persistence for one time step, HSMM explicitly models state persistence. 
From a graphical model perspective, HSMM has three sets of variables: the observations $o_t \in \{1,\ldots,n_o\}$, the latent states $x_t \in \{1,\ldots,n_x\}$, and another latent variable $d_t \in \{1,\ldots,n_d\}$ which determines the length of state persistence. HSMM is specified by three conditional probability tables (CPTs): the observation/emission probability $p(o_t|x_t)$ and the state transition and the duration probabilities given by:
\begin{align}
p(d_t|x_t, d_{t-1}) &= \begin{cases} p(d_t|x_t) &\mbox{if } d_{t-1}=1 \\
\delta(d_t, d_{t-1}-1) & \mbox{if } d_{t-1} > 1 \end{cases} \label{durProb}\\
p(x_t|x_{t-1},d_{t-1}) &= \begin{cases} p(x_t|x_{t-1}) & ~~~~~\mbox{if } d_{t-1}=1 \\
\delta(x_t, x_{t-1}) & ~~~~~\mbox{if } d_{t-1} > 1 \end{cases} \label{transProb},
\end{align}
where $\delta(a,b)$ denotes the Dirac delta function: $\delta(a,b)=1$ if $a=b$ and 0 otherwise. In addition, one can consider suitable prior probabilities $p(x_0)$ and $p(d_0)$. In essence, $d_t$ works as a down counter for state persistence. When $d_{t-1} > 1$, the model remains in the same state $x_t = x_{t-1}$, while when $d_{t-1} = 1$, one samples a new state $x_t$ and the new duration in that state $d_t|x_t$. For our analysis, we assume $p(d_t|x_t, d_{t-1}=1)$ to be a discrete multinomial distribution over $\{1,\ldots,n_d\}$ where $n_d$ denotes the largest duration of state persistence.

The considered inference problem can be posed as follows: given a set of sequences $\{\mathbf{S}^1,\ldots,\mathbf{S}^N\}$ drawn independently from the HSMM model, where each sequence is  $\mathbf{S}^{i} = \{o_1^i,\ldots,o_{T_i}^i\}, i=1,\ldots,N$, our goal is to compute the probability $p(\mathbf{S}^{test})$ of any given test sequence $\mathbf{S}^{test} = (o_1^{test},\ldots, o_T^{test})$. A traditional approach would be to estimate the CPTs using the EM algorithm, and use the estimates to compute $p(\mathbf{S}^{test})$. However, the EM algorithm is not guaranteed to estimate the parameters optimally, and hence the computation of $p(\mathbf{S}^{test})$ may be incorrect. The focus of our work is to develop a provably correct spectral algorithm for computing $p(\mathbf{S}^{test})$.

\subsection{HSMM in Tensor Notations}
We start by considering the matrix forms of the HSMM parameters and writing the computations in tensor notation, as introduced in Section~\ref{sec:notation}. Specifically, $p(d_t|x_t,d_{t-1}=1)$ is denoted as $D \in \mathbb{R}^{n_d\times n_x}$, $p(x_t|x_{t-1},d_{t-1}=1)$ is denoted as $X \in \mathbb{R}^{n_x\times n_x}$, and $p(o_t|x_t)$ as $O \in \mathbb{R}^{n_o\times n_x}$. We make the following assumptions on the HSMM parameters:
\\

\noindent {\bf{Assumptions}}

%\begin{assumption}
%\label{assumptions}
%\leavevmode
%{\leftmargini=2.6ex
\begin{enumerate}[topsep=0pt,itemsep=-1ex,partopsep=1ex,parsep=1ex,label=\subscript{A}{\arabic*}.]
\item $\mathcal{X}$ is full rank and has non-zero probability of visiting any state from any other state.
\item $D$ has a non-zero probability of any duration in any state.
\item $O$ is full column rank and, as a consequence, $n_x \leq n_o$.
\end{enumerate} %}
%\end{assumption}

We provide some comments on the above assumptions. We note that the assumption $A1$ can be relaxed to allow zero entries (while still ensuring full rank structure) and thus prevent certain states to be directly reachable from other states; however, this would require more involved analysis based on the mixing time of the corresponding Markov chain~\cite{lepw09}, and is not pursued in this work. Also, observe that the assumption of $n_x \leq n_o$ is needed in order to ensure that hidden states are identifiable, although recent work is showing that such an assumption can be relaxed in some cases \cite{ahjk13}. Intuitively, it means that the number of different observations coming from each state is large enough, so that one hidden state can be differentiated from the other. 

To express the joint probability $p(o_1,\ldots,o_T)$ for any possible observation sequence in tensor form, we utilize the junction tree algorithm \cite{barb12}. The resulting tree is shown in Figure \ref{fig:jthsmm} and it corresponds to the graphical model of HSMM in Figure \ref{fig:hsmm}. Recall, that the junction tree is a tree-structured representation of an arbitrary graph enabling efficient inference. It can be constructed by forming a maximal spanning tree from the cliques of the graph. The cliques then represent vertices in the junction tree and the edges connecting the vertices are labeled with variables common to two cliques it connects. The set of variables on the edges are referred to as separators. For example, in Figure \ref{fig:jthsmm} the cliques $\mathbb{X}_t$ and $\mathbb{D}_t$ have two variables in common, $x_{t-1}$ and $d_{t-1}$, and which define the sepatator between  $\mathbb{X}_t$ and $\mathbb{D}_t$.

We proceed by representing the clique CPTs of the junction tree as tensors. For example, the clique $\mathbb{X}_t$, containing the CPT of $p(x_t|x_{t-1}, d_{t-1})$ is represented as tensor $\underset{x_t|x_{t-1}d_{t-1}}{{\bm{\mathscr{X}}}}$. For ease of exposition, the tensor's modes are named based on the variables on which the tensor depends. We also keep the conditioning symbol $|$, for clarity. Similarly, we represent the clique $\mathbb{D}_{t}$ with its CPT $p(d_t|x_t,d_{t-1})$ as tensor $\underset{d_{t}|x_{t}d_{t-1}}{\bm{\mathscr{D}}}$, and $\mathbb{O}_{t}$ containing $p(o_t|x_t)$ as tensor $\underset{o_{t}|x_{t}}{\bm{\mathscr{O}}}$.
\begin{figure*}[!t]
\centering
   \includegraphics[width=\textwidth]{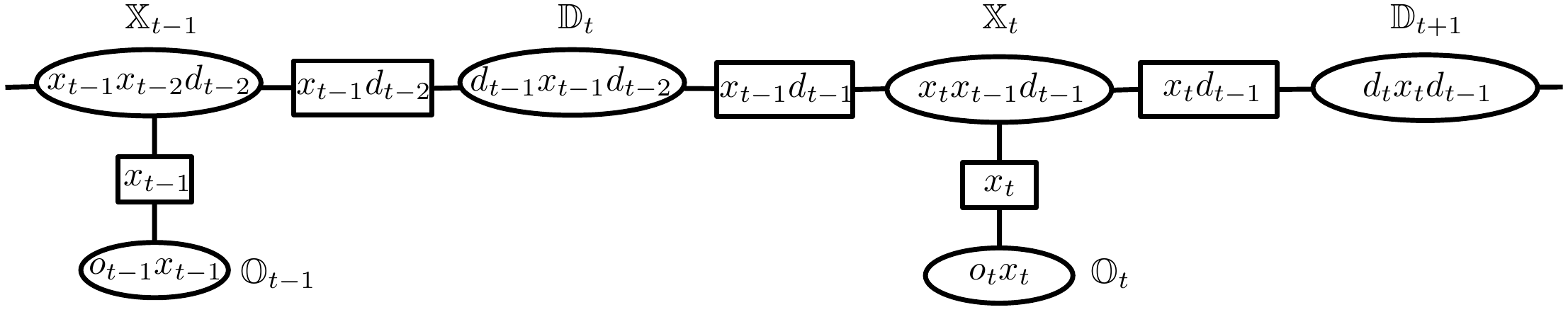}
   \caption{Junction Tree for Hidden Semi-Markov Model. The ovals represent cliques, which are denoted by capital blackboard bold variables; the rectangles denote separators. Symbols within the shapes represent the variables on which the corresponding potentials depend.}
   \label{fig:jthsmm}
\end{figure*}

If we denote the joint probability of the observed sequence $p(o_1,\ldots,o_T)$ as $\underset{o_1,\ldots,o_T}{\bm{\mathscr{P}}}$ then the message passing for the junction tree algorithm in Figure \ref{fig:jthsmm} can be represented as tensor multiplications:

%\begin{align*}
%p(o_1,\ldots, o_T) = \ldots p(x_{t-1}|x_{t-2},d_{t-2})\sum_{d_{t-1}}p(d_{t-1}|x_{t-1}d_{t-2}) \times\\
%\times\left(\sum_{x_t}p(x_t|x_{t-1},d_{t-1})p(o_t|x_t)\right)\times\sum_{d_t}p(d_t|x_td_{t-1})\nonumber\\
%\label{mesScalar}
%\end{align*}

\begin{align}
\underset{o_1,\ldots,o_T}{\bm{\mathscr{P}}} =
\prod_{t} \underset{d_{t-1}|x_{t-1}x_{t-1}d_{t-2}}{\bm{\mathscr{D}}}
\times_{x_{t-1}d_{t-1}}\left(\underset{x_tx_{t}|x_{t-1}d_{t-1}d_{t-1}}{\bm{\mathscr{X}}}\times_{x_t}\underset{o_t|x_t}{\bm{\mathscr{O}}}\right),
\label{eq:mesTensor}
\end{align}
where, for simplicity, we denoted by $\prod_{t}$ the tensor product over multiple time steps.

Note that in \eqref{eq:mesTensor} the neighboring tensors are multiplied along the modes which are the separator variables between two corresponding neighboring cliques in Figure \ref{fig:jthsmm}. Therefore, as we discussed in Section \ref{sec:notation}, if a certain mode of a tensor is to participate multiple times in products with other tensor, the mode must be duplicated for the expression to remain correct. It can easily be seen from the junction tree that the number of times the mode is duplicated depends on the number of times such a variable appears in separators adjacent to the clique. For example, the tensor $\underset{x_tx_t|x_{t-1}d_{t-1}d_{t-1}}{{\bm{\mathscr{X}}}}$ has a mode $x_{t-1}$ appearing once in the separator connecting $\mathbb{X}_t$ and $\mathbb{D}_t$ in Figure \ref{fig:jthsmm}, while $x_t$ appears a total of two times - once in the separator connecting $\mathbb{X}_t$ and $\mathbb{O}_t$, and once in the separator connecting $\mathbb{X}_t$ and $\mathbb{D}_{t+1}$. Finally, $d_{t-1}$ appears in the separator between $\mathbb{D}_t$ and $\mathbb{X}_t$, and between $\mathbb{D}_{t+1}$ and $\mathbb{X}_t$. Applying the same reasoning to tensors $\bm{\mathscr{D}}$ and $\bm{\mathscr{O}}$ results in the expression \eqref{eq:mesTensor}.

\subsection{Summary of Results}

In this work, we represent expression \eqref{eq:mesTensor}, which is defined in terms of unknown model parameters, in a different form, called observable representation, where all the factors can be estimated directly from data using certain sample moments without knowledge of model parameters. Such an observable form is derived in Sections \ref{sec:obsTensRepresentation} and \ref{sec:obsTensors}. Based on the obtained representation, we propose in Section \ref{sec:specAlgorithm} a simple spectral algorithm, which requires estimating $\bm{\mathscr{X}}$, $\bm{\mathscr{D}}$ and $\bm{\mathscr{O}}$ for all the time stamps $t$. This estimation process is expensive as it involves costly tensor operations to be performed at each time index $t$. Moreover, the accurate estimation of these tensors requires large number of training sequences which might not be available, leading to inaccurate and unstable computations. However, exploiting the homogeneity property of HSMMs, i.e., the fact that the probability distributions, which the above tensors represent, are independent of time index $t$, we derive computationally more efficient and accurate spectral algorithm in Section \ref{sec:effcAlgorithm} requiring estimation of only three tensors for all the time stamps $t$. Although the computational complexity of inference, i.e., the evaluation of expression \eqref{eq:mesTensor}, is not affected by the introduced modifications, the overall algorithm becomes faster and more accurate. 
In Section \ref{sec:observationVars} we return to the results of Sections \ref{sec:obsTensRepresentation} and establish the conditions under which the derived observable representation  exists. In particular, our analysis shows that the number of dimensions of the required sample moments has logarithmic dependence on the longest state persistence $n_d$. Such conclusion is in contrast to the analysis, which would follow from the work of \cite{psit12}, in which case the required number of dimensions in the estimated sample moments would have had linear dependence on $n_d$. The exponential reduction in the size of the sample moments represents significant improvement in algorithm's efficiency and accuracy. Finally, we evaluated the proposed algorithm using synthetic and real datasets and compared its performance with the traditional EM approach. The main conclusion from such evaluations is that for large enough datasets the spectral method gets similar or better performance than EM, while at the same time being orders of magnitude faster than EM.

%% file: TensorForm.tex
In this Section we present the details of the spectral inference approach. In particular, in Sections \ref{sec:obsTensRepresentation} and \ref{sec:obsTensors} we derive observable tensor representation and show how to estimate each of its factors directly from data. Practical algorithms implementing these ideas are then derived in Sections \ref{sec:specAlgorithm} and \ref{sec:effcAlgorithm}.

\subsection{Observable Tensor Representation}
\label{sec:obsTensRepresentation}
%\begin{figure*}[th]%[!thb]
%\centering
%   \includegraphics[width=0.9\textwidth]{JTtensors.pdf}
%   \caption{Tree representation of the spectral HSMM algorithm}
%   \label{fig:jttensors}
%\end{figure*}
Observe that the computation of the joint probability in \eqref{eq:mesTensor} requires knowledge of the unknown model parameters. Our goal is to change the tensor representation such that $\underset{o_1,\ldots,o_T}{\bm{\mathscr{P}}}$ can be written in terms of the quantities directly computable from data. To that end, we follow \cite{psit12} and between every two factors in \eqref{eq:mesTensor} introduce an identity tensor with the modes corresponding to the modes along which the multiplication is performed. For example, consider a part of \eqref{eq:mesTensor} after introducing identity tensors:
\begin{align}
\times\hspace{-0pt}\underset{{x_{t\hspace{-1pt}-\hspace{-1pt}1}d_{t\hspace{-1pt}-\hspace{-1pt}2}}}{\bm{\mathscr{I}}}\hspace{-0pt}\times_{{x_{t\hspace{-1pt}-\hspace{-1pt}1}d_{t\hspace{-1pt}-\hspace{-1pt}2}}}\hspace{-2pt}\underset{d_{t\hspace{-1pt}-\hspace{-1pt}1}|x_{t\hspace{-1pt}-\hspace{-1pt}1}x_{t\hspace{-1pt}-\hspace{-1pt}1}d_{t\hspace{-1pt}-\hspace{-1pt}2}}{\bm{\mathscr{D}}}\hspace{-6pt}\times_{{x_{t\hspace{-1pt}-\hspace{-1pt}1}d_{t\hspace{-1pt}-\hspace{-1pt}1}}}\hspace{-2pt}\underset{{x_{t\hspace{-1pt}-\hspace{-1pt}1}d_{t\hspace{-1pt}-\hspace{-1pt}1}}}{\bm{\mathscr{I}}}
\hspace{-0pt}\times_{{x_{t\hspace{-1pt}-\hspace{-1pt}1}d_{t\hspace{-1pt}-\hspace{-1pt}1}}}\hspace{-0pt}\left(\hspace{-0pt}\underset{x_tx_t|x_{t\hspace{-1pt}-\hspace{-1pt}1}d_{t\hspace{-1pt}-\hspace{-1pt}1}d_{t\hspace{-1pt}-\hspace{-1pt}1}}{\bm{\mathscr{X}}}\hspace{-0pt}\times_{x_t}\hspace{-2pt}\underset{x_t}{\bm{\mathscr{I}}}\hspace{-2pt}\times_{x_t}\hspace{-2pt}\underset{o_tx_t}{\bm{\mathscr{O}}}\hspace{-0pt}\right)\hspace{-0pt}\times_{x_td_{t\hspace{-1pt}-\hspace{-1pt}1}}\hspace{-2pt}\underset{{x_td_{t\hspace{-1pt}-\hspace{-1pt}1}}}{\bm{\mathscr{I}}}\hspace{-0pt}\times,
\label{eq:mesTensorIdent}
\end{align}
where all the identity tensors have duplicated modes which are not shown.

%\begin{align*}
% \underset{x_{t}|x_{t-1}d_{t-1}}{\bm{\mathscr{X}}}\times_{x_t}\underset{o_t|x_t}{\bm{\mathscr{O}}} = \underset{x_{t}|x_{t-1}d_{t-1}}{\bm{\mathscr{X}}}\times_{x_t}\underset{x_t}{\bm{\mathscr{I}}}\times_{x_t}\underset{o_t| x_t}{\bm{\mathscr{O}}}
%\end{align*}
%\noindent where we introduced an identity tensor $\underset{x_t}{\bm{\mathscr{I}}}$ which does not change anything. Then if we write $\underset{x_t}{\bm{\mathscr{I}}} = \underset{x_tv}{\bm{\mathscr{F}}}\times_{v}\underset{x_tv}{\bm{\mathscr{F}}^{-1}}$ for some variable $v$, then we can select invertible tensor $\underset{x_tv}{\bm{\mathscr{F}}}$
Now rewrite each of the identity tensors in \eqref{eq:mesTensorIdent} as a multiplication of some factor times its inverse. For example, 
\begin{align*}
%\label{eq:identitySplit}
\underset{x_t}{\bm{\mathscr{I}}} = \underset{\omega_{x_t}x_t}{\bm{\mathscr{F}}}\times_{\omega_{x_t}}\underset{\omega_{x_t}x_t}{\bm{\mathscr{F}}^{-1}},
\end{align*}
for some invertible factor $\underset{\omega_{x_t}x_t}{\bm{\mathscr{F}}}$, whose modes are $x_t$ and $\omega_{x_t}$. Note that the choice of mode $x_t$ is fixed and is determined by the modes of the identity tensor $\underset{x_t}{\bm{\mathscr{I}}}$, while the mode $\omega_{x_t}$ is not fixed and we have a freedom in selecting it. Moreover, observe that since the tensor inversion is done along the mode $\omega_{x_t}$ and the matrix ${{\mathbf{F}}}$ has its rows associated with mode $\omega_{x_t}$, we need to ensure such a matrix has full column rank for the inverse to exist and for the product $\mathbf{F}^{-1}\mathbf{F}$ to be the identity matrix (see Section \ref{sec:notation} for more details on tensor inversion). Based on the above discussion, we choose tensor $\bm{\mathscr{F}}$ such that (i) $\omega_{x_t}$ are the observed variables, (ii)  $\underset{\omega_{x_t}x_t}{\bm{\mathscr{F}}}$ is invertible and (iii) we interpret the factor $\underset{\omega_{x_t}x_t}{\bm{\mathscr{F}}}$ as corresponding to a conditional probability distribution, i.e., $p(\omega_{x_t}|x_t)$ and therefore write $\underset{\omega_{x_t}|x_t}{\bm{\mathscr{F}}}$.

After expanding each of the identity tensors, regrouping the factors and recalling that in a series of tensor multiplication the order is irrelevant, we can identify three modified tensors:
\begin{align*}
%\tilde{\bm{\mathscr{D}}}
\underset{\omega_{x_{t\hspace{-1pt}-\hspace{-1pt}1}d_{t\hspace{-1pt}-\hspace{-1pt}2}}\omega_{x_{t\hspace{-1pt}-\hspace{-1pt}1}d_{t\hspace{-1pt}-\hspace{-1pt}1}}}{\tilde{\bm{\mathscr{D}}}}&= \underset{\omega_{x_{t\hspace{-1pt}-\hspace{-1pt}1}d_{t\hspace{-1pt}-\hspace{-1pt}2}}|x_{t\hspace{-1pt}-\hspace{-1pt}1}d_{t\hspace{-1pt}-\hspace{-1pt}2}}{\bm{\mathscr{F}}^{-1}}\times_{x_{t\hspace{-1pt}-\hspace{-1pt}1}d_{t\hspace{-1pt}-\hspace{-1pt}2}}\underset{d_{t\hspace{-1pt}-\hspace{-1pt}1}|x_{t\hspace{-1pt}-\hspace{-1pt}1}x_{t\hspace{-1pt}-\hspace{-1pt}1}d_{t\hspace{-1pt}-\hspace{-1pt}2}}{\bm{\mathscr{D}}}\times_{x_{t\hspace{-1pt}-\hspace{-1pt}1}d_{t\hspace{-1pt}-\hspace{-1pt}1}}\underset{\omega_{x_{t\hspace{-1pt}-\hspace{-1pt}1}d_{t\hspace{-1pt}-\hspace{-1pt}1}}|x_{t\hspace{-1pt}-\hspace{-1pt}1}d_{t\hspace{-1pt}-\hspace{-1pt}1}}{\bm{\mathscr{F}}}\\
\underset{\omega_{x_{t\hspace{-1pt}-\hspace{-1pt}1}d_{t\hspace{-1pt}-\hspace{-1pt}1}}\omega_{x_t}\omega_{x_{t}d_{t\hspace{-1pt}-\hspace{-1pt}1}}}{\tilde{\bm{\mathscr{X}}}} &=
\underset{\omega_{x_{t\hspace{-1pt}-\hspace{-1pt}1}d_{t\hspace{-1pt}-\hspace{-1pt}1}}|x_{t\hspace{-1pt}-\hspace{-1pt}1}d_{t\hspace{-1pt}-\hspace{-1pt}1}}{\bm{\mathscr{F}}^{-1}}\times_{x_{t\hspace{-1pt}-\hspace{-1pt}1}d_{t\hspace{-1pt}-\hspace{-1pt}1}}
\hspace{-2pt}\left(\underset{x_{t}x_t|x_{t\hspace{-1pt}-\hspace{-1pt}1}d_{t\hspace{-1pt}-\hspace{-1pt}1}d_{t\hspace{-1pt}-\hspace{-1pt}1}}{\bm{\mathscr{X}}}\hspace{-2pt}\times_{x_t}\hspace{-2pt}\underset{\omega_{x_t}|x_{t}}{\bm{\mathscr{F}}}\right)
\hspace{-2pt}\times_{x_td_{t\hspace{-1pt}-\hspace{-1pt}1}}\hspace{-3pt}\underset{\omega_{x_td_{t\hspace{-1pt}-\hspace{-1pt}1}}|x_td_{t\hspace{-1pt}-\hspace{-1pt}1}}{\bm{\mathscr{F}}}\\
\underset{\omega_{x_t}o_t}{\tilde{\bm{\mathscr{O}}}} &= \underset{\omega_{x_t}|x_{t}}{\bm{\mathscr{F}}^{-1}}\times_{x_t}\underset{o_t|x_t}{\bm{\mathscr{O}}}.
\end{align*}
Note that although each of the above tensors depends only on the observed variables $\omega$, how to estimate them is not clear yet: the expressions on the right depend on the unknown model parameters, while the tensors on the left do not correspond to valid probability distributions (due to the presence of inverses ${\bm{\mathscr{F}}^{-1}}$), and so cannot be estimated from data using sample moments. For example, $\underset{\omega_{x_{t\hspace{-1pt}-\hspace{-1pt}1}d_{t\hspace{-1pt}-\hspace{-1pt}2}}\omega_{x_{t\hspace{-1pt}-\hspace{-1pt}1}d_{t\hspace{-1pt}-\hspace{-1pt}1}}}{\tilde{\bm{\mathscr{D}}}}$ is not a tensor form of $p(\omega_{x_{t\hspace{-1pt}-\hspace{-1pt}1}d_{t\hspace{-1pt}-\hspace{-1pt}2}}, \omega_{x_{t\hspace{-1pt}-\hspace{-1pt}1}d_{t\hspace{-1pt}-\hspace{-1pt}1}})$.

Next, we discuss the choice of the observable set $\omega$ in the factors ${\bm{\mathscr{F}}}$. From Figure \ref{fig:jthsmm} we can see that there are three types of separators which depend on $x_{t-1}d_{t-1}$, $x_{t}d_{t-1}$ and $x_t$, consequently, there are three types of identity tensors which we introduced in \eqref{eq:mesTensorIdent}, i.e., $\underset{{x_{t-1}d_{t-1}}}{\bm{\mathscr{I}}}$, $\underset{{x_{t}d_{t-1}}}{\bm{\mathscr{I}}}$ and $\underset{x_{t}}{\bm{\mathscr{I}}}$. Therefore, we need to define three types of observable sets $\omega_{x_{t-1}d_{t-1}}$, $\omega_{x_{t}d_{t-1}}$ and $\omega_{x_{t}}$. There could be multiple choices for these sets, one of them is $\omega_{x_{t-1}d_{t-1}} = \omega_{x_{t}d_{t-1}}= \{o_{t+1}, o_{t+2}, \ldots\}$ for all $t$ (see Figure \ref{fig:condind} for an illustration). Ideally, we want these sets to be of minimal size, since they need to be estimated from observations. The detailed description of how many and which of these observations to select to get a minimal set is deferred until Section \ref{sec:observationVars}, where we also show that we can set $\omega_{x_t} = o_t$.

In what follows, we define $\bm{\mathsf{O}}_{R_{t}} := \{o_{t+1}, o_{t+2}, \ldots\}$, to emphasize that this is a fixed set of observations whose length is yet to be determined, starting after time stamp $t$ and going to the right (or forward in time) in the graphical model in Figure \ref{fig:hsmm}. With these definitions, setting $\omega_{x_{t-1}d_{t-1}} = \bm{\mathsf{O}}_{R_{t}}$, $\omega_{x_{t}d_{t-1}} = \bm{\mathsf{O}}_{R_{t}}$, $\omega_{x_{t-1}d_{t-2}} = \bm{\mathsf{O}}_{R_{t-1}}$ and $\omega_{x_t} = o_t$, we can now rewrite \eqref{eq:mesTensor} in the form:
\begin{align}
\label{eq:mesTensorObserv}
\underset{o_1,\ldots,o_T}{\bm{\mathscr{P}}} \hspace{-3pt}=
\prod_{t} \underset{\bm{\mathsf{O}}_{R_{t-1}}\bm{\mathsf{O}}_{R_{t}}}{\tilde{\bm{\mathscr{D}}}}
\hspace{-3pt}\times_{\bm{\mathsf{O}}_{R_t}}\left(\underset{\bm{\mathsf{O}}_{R_{t}}o_t\bm{\mathsf{O}}_{R_{t}}}{\tilde{\bm{\mathscr{X}}}}\hspace{-3pt}\times_{o_t}\underset{o_to_t}{\tilde{\bm{\mathscr{O}}}}\right).
\end{align}
Comparing \eqref{eq:mesTensor} and \eqref{eq:mesTensorObserv} we see that the above equation expresses the joint probability distribution in the observable form. As noted above, we cannot yet use this formula in practice since we do not know how to compute the transformed tensors. In what follows, we show how to estimate such tensors directly from data, without the need for the model parameters.

%% file: ObservableRepresentation.tex
In this Section we express each of the tensors in \eqref{eq:mesTensorObserv} in the form suitable for estimation directly from the observed sequences. 

\subsubsection[Computation of Tensor D]{Computation of Tensor $\underset{\bm{\mathsf{O}}_{R_{t\hspace{-1pt}-\hspace{-1pt}1}}\bm{\mathsf{O}}_{R_{t}}}{\tilde{\bm{\mathscr{D}}}}$ 
}
\label{sec:compD}

\begin{figure*}[!t]
\centering
   \includegraphics[width=0.4\textwidth]{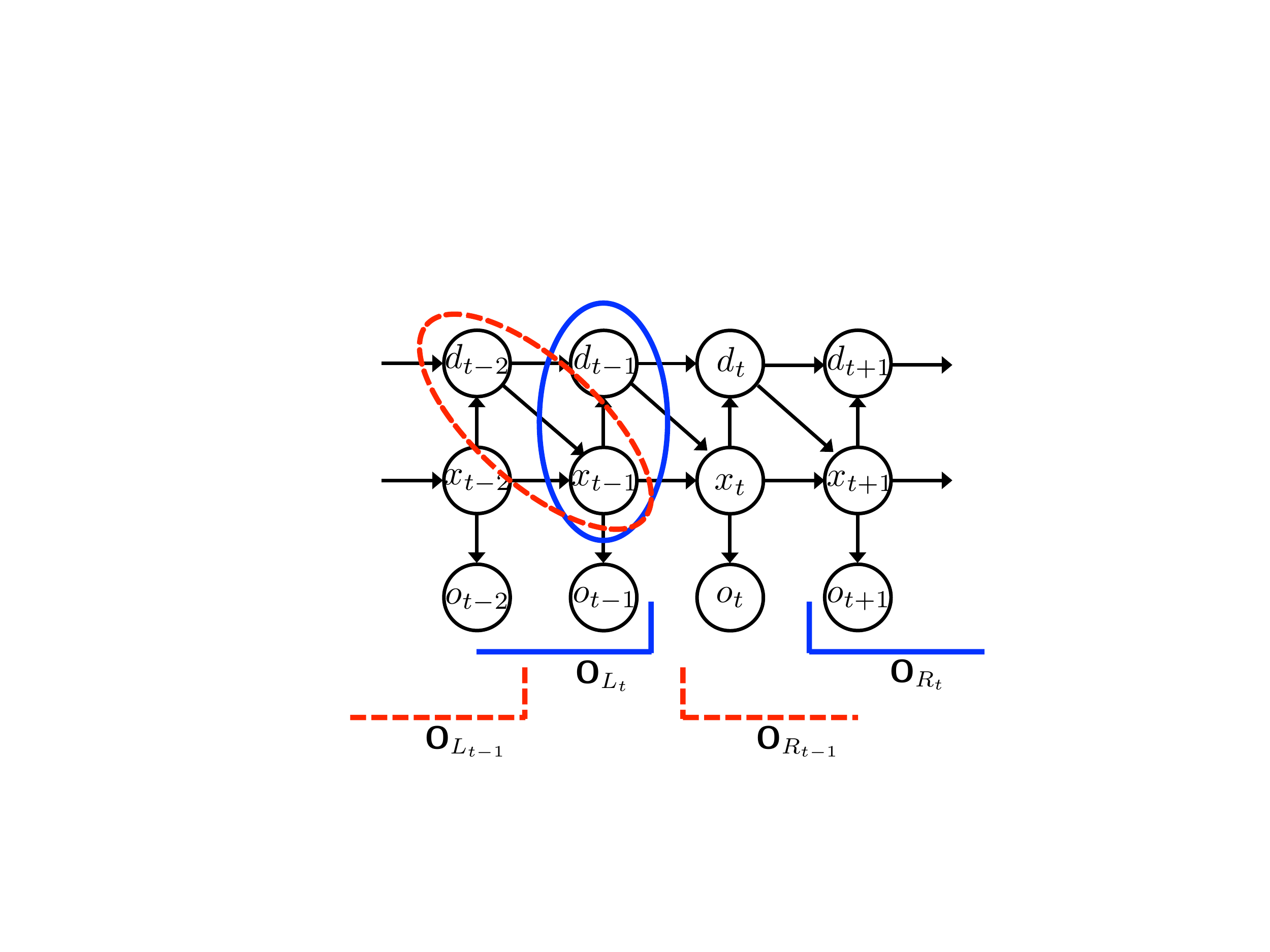}
   \caption{Conditional independence in HSMM. The figure depicts two sets of relationships: $\bm{\mathsf{O}}_{L_{t}}$ and $\bm{\mathsf{O}}_{R_{t}}$ are independent conditioned on $x_{t-1}d_{t-1}$, similarly,  $\bm{\mathsf{O}}_{L_{t-1}}$ and $\bm{\mathsf{O}}_{R_{t-1}}$ are conditionally independent given $x_{t-1}d_{t-2}$.  We defined $\bm{\mathsf{O}}_{L_{t}} =\{\ldots, o_{t-2},o_{t-1}\}$ and $\bm{\mathsf{O}}_{R_{t}} =\{o_{t+1}, o_{t+2}, \ldots\}$.}
   \label{fig:condind}
\end{figure*}

Consider the tensor from Section \ref{sec:obsTensRepresentation}
\begin{align}
\label{eq:tensorD}
\underset{\bm{\mathsf{O}}_{R_{t\hspace{-1pt}-\hspace{-1pt}1}}\bm{\mathsf{O}}_{R_{t}}}{\tilde{\bm{\mathscr{D}}}}\hspace{-2pt}=\hspace{-2pt} \underset{\bm{\mathsf{O}}_{R_{t\hspace{-1pt}-\hspace{-1pt}1}}|x_{t\hspace{-1pt}-\hspace{-1pt}1}d_{t\hspace{-1pt}-\hspace{-1pt}2}}{\bm{\mathscr{F}}^{-1}}\hspace{-2pt}\times_{x_{t\hspace{-1pt}-\hspace{-1pt}1}d_{t\hspace{-1pt}-\hspace{-1pt}2}}\underset{d_{t\hspace{-1pt}-\hspace{-1pt}1}|x_{t\hspace{-1pt}-\hspace{-1pt}1}x_{t\hspace{-1pt}-\hspace{-1pt}1}d_{t\hspace{-1pt}-\hspace{-1pt}2}}{\bm{\mathscr{D}}}\hspace{-3pt}\times_{x_{t\hspace{-1pt}-\hspace{-1pt}1}d_{t\hspace{-1pt}-\hspace{-1pt}1}}\underset{\bm{\mathsf{O}}_{R_{t}}|x_{t\hspace{-1pt}-\hspace{-1pt}1}d_{t\hspace{-1pt}-\hspace{-1pt}1}}{\bm{\mathscr{F}}},
\end{align}
whose modes are the observable variables $\bm{\mathsf{O}}_{R_{t-1}}$ and $\bm{\mathsf{O}}_{R_{t}}$. To estimate this tensor from data, consider $\bm{\mathsf{O}}_{L_{t-1}}$, a set of the observed variables such that $\bm{\mathsf{O}}_{L_{t-1}}$ and $\bm{\mathsf{O}}_{R_{t-1}}$ are independent, conditioned on $x_{t-1}d_{t-2}$ (see Figure \ref{fig:condind}): 
\begin{align}
\label{eq:CondInd}
p(\bm{\mathsf{O}}_{L_{t\hspace{-1pt}-\hspace{-1pt}1}}, \bm{\mathsf{O}}_{R_{t\hspace{-1pt}-\hspace{-1pt}1}}) = \sum_{x_{t\hspace{-1pt}-\hspace{-1pt}1}d_{t\hspace{-1pt}-\hspace{-1pt}2}}p(\bm{\mathsf{O}}_{L_{t\hspace{-1pt}-\hspace{-1pt}1}}|x_{t\hspace{-1pt}-\hspace{-1pt}1}d_{t\hspace{-1pt}-\hspace{-1pt}2})p(\bm{\mathsf{O}}_{R_{t\hspace{-1pt}-\hspace{-1pt}1}}|x_{t\hspace{-1pt}-\hspace{-1pt}1}d_{t\hspace{-1pt}-\hspace{-1pt}2})p(x_{t\hspace{-1pt}-\hspace{-1pt}1}d_{t\hspace{-1pt}-\hspace{-1pt}2}).
\end{align}
The above conditional independence relationship can be written in tensor form: 
\begin{align}
\label{eq:tensCondInd}
\underset{\bm{\mathsf{O}}_{L_{t\hspace{-1pt}-\hspace{-1pt}1}}\bm{\mathsf{O}}_{R_{t\hspace{-1pt}-\hspace{-1pt}1}}}{\bm{\mathscr{M}}} \hspace{-3pt}=\hspace{-3pt} \underset{\bm{\mathsf{O}}_{L_{t\hspace{-1pt}-\hspace{-1pt}1}}|x_{t\hspace{-1pt}-\hspace{-1pt}1}d_{t\hspace{-1pt}-\hspace{-1pt}2}}{\bm{\mathscr{F}}}
\times_{x_{t\hspace{-1pt}-\hspace{-1pt}1}d_{t\hspace{-1pt}-\hspace{-1pt}2}}\underset{\bm{\mathsf{O}}_{R_{t\hspace{-1pt}-\hspace{-1pt}1}}|x_{t\hspace{-1pt}-\hspace{-1pt}1}d_{t\hspace{-1pt}-\hspace{-1pt}2}}{\bm{\mathscr{F}}}
\times_{x_{t\hspace{-1pt}-\hspace{-1pt}1}d_{t\hspace{-1pt}-\hspace{-1pt}2}}\underset{x_{t\hspace{-1pt}-\hspace{-1pt}1}d_{t\hspace{-1pt}-\hspace{-1pt}2}}{\bm{\mathscr{K}}},
\end{align}
\noindent where tensor $\bm{\mathscr{K}}$ represents the marginal $p(x_{t-1},d_{t-2})$. Note that, though not shown, the modes $x_{t-1}$ and $d_{t-2}$ need to appear twice in $\bm{\mathscr{K}}$, since it interacts with both other terms (see the discussion on mode duplication in Section \ref{sec:notation}).  The set $\bm{\mathsf{O}}_{L_{t-1}}$ is defined in a way similar to $\bm{\mathsf{O}}_{R_{t}}$ but with the set of observations starting at time stamp $t-2$ and going to the left (or backward in time), i.e., $\bm{\mathsf{O}}_{L_{t-1}} := \{\ldots, o_{t-3}, o_{t-2}\}$ (see Figure \ref{fig:condind}).

Next, we express the inverse of the tensor $\underset{\bm{\mathsf{O}}_{R_{t-1}}|x_{t-1}d_{t-2}}{\bm{\mathscr{F}}}$ from \eqref{eq:tensCondInd} and substitute back to \eqref{eq:tensorD}. For this, we observe that in \eqref{eq:tensorD} the tensor $\bm{\mathscr{F}^{-1}}$ is inverted with respect to mode $\bm{\mathsf{O}}_{R_{t-1}}$, therefore, we do the following:

\begin{align}
\underset{\bm{\mathsf{O}}_{L_{t\hspace{-1pt}-\hspace{-1pt}1}}\bm{\mathsf{O}}_{R_{t\hspace{-1pt}-\hspace{-1pt}1}}}{\bm{\mathscr{M}}}\times_{\bm{\mathsf{O}}_{R_{t-1}}} \underset{\bm{\mathsf{O}}_{R_{t\hspace{-1pt}-\hspace{-1pt}1}}|x_{t\hspace{-1pt}-\hspace{-1pt}1}d_{t\hspace{-1pt}-\hspace{-1pt}2}}{\bm{\mathscr{F}}^{-1}} &= \underset{\bm{\mathsf{O}}_{L_{t\hspace{-1pt}-\hspace{-1pt}1}}|x_{t\hspace{-1pt}-\hspace{-1pt}1}d_{t\hspace{-1pt}-\hspace{-1pt}2}}{\bm{\mathscr{F}}}
\times_{x_{t\hspace{-1pt}-\hspace{-1pt}1}d_{t\hspace{-1pt}-\hspace{-1pt}2}}\underset{x_{t\hspace{-1pt}-\hspace{-1pt}1}d_{t\hspace{-1pt}-\hspace{-1pt}2}}{\bm{\mathscr{I}}}
\times_{x_{t\hspace{-1pt}-\hspace{-1pt}1}d_{t\hspace{-1pt}-\hspace{-1pt}2}}\underset{x_{t\hspace{-1pt}-\hspace{-1pt}1}d_{t\hspace{-1pt}-\hspace{-1pt}2}}{\bm{\mathscr{K}}}\nonumber \\ \underset{\bm{\mathsf{O}}_{R_{t\hspace{-1pt}-\hspace{-1pt}1}}|x_{t\hspace{-1pt}-\hspace{-1pt}1}d_{t\hspace{-1pt}-\hspace{-1pt}2}}{\bm{\mathscr{F}}^{-1}} &= \underset{\bm{\mathsf{O}}_{L_{t\hspace{-1pt}-\hspace{-1pt}1}}\bm{\mathsf{O}}_{R_{t\hspace{-1pt}-\hspace{-1pt}1}}}{\bm{\mathscr{M}}^{-1}}\times_{\bm{\mathsf{O}}_{L_{t-1}}} \underset{\bm{\mathsf{O}}_{L_{t\hspace{-1pt}-\hspace{-1pt}1}}|x_{t\hspace{-1pt}-\hspace{-1pt}1}d_{t\hspace{-1pt}-\hspace{-1pt}2}}{\bm{\mathscr{F}}}
\times_{x_{t\hspace{-1pt}-\hspace{-1pt}1}d_{t\hspace{-1pt}-\hspace{-1pt}2}}\underset{x_{t\hspace{-1pt}-\hspace{-1pt}1}d_{t\hspace{-1pt}-\hspace{-1pt}2}}{\bm{\mathscr{K}}},
\label{eq:Finv}
\end{align}
\noindent where $\underset{\bm{\mathsf{O}}_{L_{t\hspace{-1pt}-\hspace{-1pt}1}}\bm{\mathsf{O}}_{R_{t\hspace{-1pt}-\hspace{-1pt}1}}}{\bm{\mathscr{M}}^{-1}}$ is inverted with respect to mode $\bm{\mathsf{O}}_{L_{t-1}}$. Next, substituting \eqref{eq:Finv} back to \eqref{eq:tensorD}, we get

\begin{align}
\underset{\bm{\mathsf{O}}_{R_{t\hspace{-1pt}-\hspace{-1pt}1}}\bm{\mathsf{O}}_{R_{t}}}{\tilde{\bm{\mathscr{D}}}} &\hspace{-7pt}=\hspace{-4pt} \underset{\bm{\mathsf{O}}_{L_{t\hspace{-1pt}-\hspace{-1pt}1}}\bm{\mathsf{O}}_{R_{t\hspace{-1pt}-\hspace{-1pt}1}}}{\bm{\mathscr{M}}^{-1}}\hspace{-6pt}\times_{\bm{\mathsf{O}}_{L_{t\hspace{-1pt}-\hspace{-1pt}1}}}\hspace{-2pt}
\overbracket{\underset{\bm{\mathsf{O}}_{L_{t\hspace{-1pt}-\hspace{-1pt}1}}|x_{t\hspace{-1pt}-\hspace{-1pt}1}d_{t\hspace{-1pt}-\hspace{-1pt}2}}{\bm{\mathscr{F}}}\hspace{-5pt}\times_{x_{t\hspace{-1pt}-\hspace{-1pt}1}d_{t\hspace{-1pt}-\hspace{-1pt}2}}\hspace{-2pt}
\underset{x_{t\hspace{-1pt}-\hspace{-1pt}1}d_{t\hspace{-1pt}-\hspace{-1pt}2}}{\bm{\mathscr{K}}}
 \hspace{-5pt}\times_{x_{t\hspace{-1pt}-\hspace{-1pt}1}d_{t\hspace{-1pt}-\hspace{-1pt}2}}\hspace{-2pt}\underset{d_{t\hspace{-1pt}-\hspace{-1pt}1}|x_{t\hspace{-1pt}-\hspace{-1pt}1}x_{t\hspace{-1pt}-\hspace{-1pt}1}d_{t\hspace{-1pt}-\hspace{-1pt}2}}{\bm{\mathscr{D}}}\hspace{-5pt}\times_{x_{t\hspace{-1pt}-\hspace{-1pt}1}d_{t\hspace{-1pt}-\hspace{-1pt}1}}\hspace{-2pt}\underset{\bm{\mathsf{O}}_{R_{t}}|x_{t\hspace{-1pt}-\hspace{-1pt}1}d_{t\hspace{-1pt}-\hspace{-1pt}1}}{\bm{\mathscr{F}}}}\nonumber \\
 \label{eq:DtensorObs}
 &\hspace{-7pt}=\hspace{-4pt}\underset{\bm{\mathsf{O}}_{L_{t\hspace{-1pt}-\hspace{-1pt}1}}\bm{\mathsf{O}}_{R_{t\hspace{-1pt}-\hspace{-1pt}1}}}{\bm{\mathscr{M}}^{-1}}\times_{\bm{\mathsf{O}}_{L_{t\hspace{-1pt}-\hspace{-1pt}1}}}~~~\underset{\bm{\mathsf{O}}_{L_{t\hspace{-1pt}-\hspace{-1pt}1}}\bm{\mathsf{O}}_{R_{t}}}{\bm{\mathscr{M}}},
\end{align}
\noindent where we have eliminated all the latent variables by multiplying the last four terms on the first line.

Observe that the tensors $\underset{\bm{\mathsf{O}}_{L_{t\hspace{-1pt}-\hspace{-1pt}1}}\bm{\mathsf{O}}_{R_{t\hspace{-1pt}-\hspace{-1pt}1}}}{\bm{\mathscr{M}}}$ and $\underset{\bm{\mathsf{O}}_{L_{t\hspace{-1pt}-\hspace{-1pt}1}}\bm{\mathsf{O}}_{R_{t}}}{\bm{\mathscr{M}}}$ represent valid joint probability distributions over a subset of observations $p(\bm{\mathsf{O}}_{L_{t\hspace{-1pt}-\hspace{-1pt}1}}, \bm{\mathsf{O}}_{R_{t\hspace{-1pt}-\hspace{-1pt}1}})$ and $p(\bm{\mathsf{O}}_{L_{t\hspace{-1pt}-\hspace{-1pt}1}}, \bm{\mathsf{O}}_{R_{t}})$, respectively, and though they are defined with respect to unknown model parameters (as, for example, in \eqref{eq:CondInd}), we can readily estimate them from data. For example, $\underset{\bm{\mathsf{O}}_{L_{t\hspace{-1pt}-\hspace{-1pt}1}}\bm{\mathsf{O}}_{R_{t}}}{\bm{\mathscr{M}}}$ is a tensor, where each entry is computed from the frequency of co-occurrence of tuples of the observed symbols $\{\ldots, o_{t-3}, o_{t-2}, o_{t+1}, o_{t+2}, \ldots\}$. Ideally, we want a small number of observation symbols since we need to estimate their co-occurrence frequency from the training data. A precise characterization of how many and which of these symbols suffices for the analysis will be done in Section \ref{sec:observationVars}. 
\begin{figure*}[!t]
\centering
   \includegraphics[width=\textwidth]{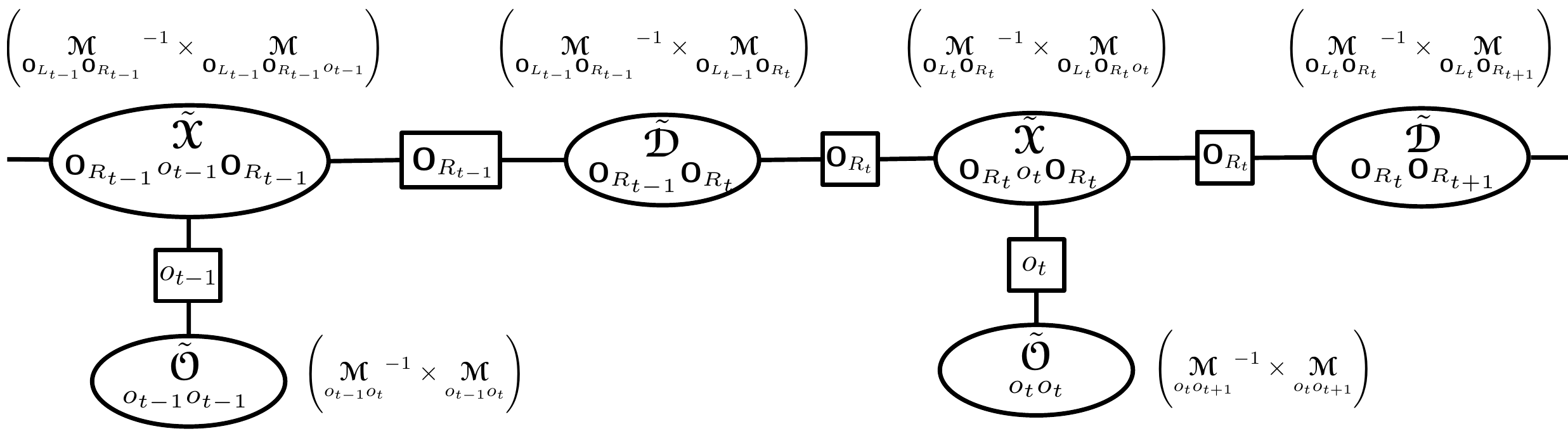}
   \caption{Graphical representation of the HSMM spectral algorithm for inference in Algorithm \ref{algSpectralBasic}. As compared to junction tree in Figure \ref{fig:jthsmm}, the cliques and separators are now defined in terms of the tensors, which are defined with respect to the observed data. The expressions in the parenthesis show the observable representation of the corresponding tensors.}
   \label{fig:jttensors}
\end{figure*}

\subsubsection[Computation of Tensor X]{Computation of Tensor $\underset{\bm{\mathsf{O}}_{R_{t}}o_t\bm{\mathsf{O}}_{R_{t}}}{\tilde{\bm{\mathscr{X}}}}$}

The form of this tensor was established at the beginning of Section \ref{sec:obsTensors} to be:
\begin{align}
\label{eq:tensorX}
\underset{\bm{\mathsf{O}}_{R_{t}}o_t\bm{\mathsf{O}}_{R_{t}}}{\tilde{\bm{\mathscr{X}}}} =
\underset{\bm{\mathsf{O}}_{R_{t}}|x_{t\hspace{-1pt}-\hspace{-1pt}1}d_{t\hspace{-1pt}-\hspace{-1pt}1}}{\bm{\mathscr{F}}^{-1}}\hspace{-5pt}\times_{x_{t\hspace{-1pt}-\hspace{-1pt}1}d_{t\hspace{-1pt}-\hspace{-1pt}1}}
\hspace{-2pt}\left(\underset{x_tx_{t}|x_{t\hspace{-1pt}-\hspace{-1pt}1}d_{t\hspace{-1pt}-\hspace{-1pt}1}d_{t\hspace{-1pt}-\hspace{-1pt}1}}{\bm{\mathscr{X}}}\hspace{-2pt}\times_{x_t}\hspace{-2pt}\underset{o_{t}|x_{t}}{\bm{\mathscr{F}}}\right)
\hspace{-2pt}\times_{x_td_{t-1}}\hspace{-3pt}\underset{\bm{\mathsf{O}}_{R_{t}}|x_td_{t-1}}{\bm{\mathscr{F}}}.
\end{align}
Consider the following conditional independence relationship (see Figure \ref{fig:condind}):
\begin{align}
\label{eq:tensCondInd2}
\underset{\bm{\mathsf{O}}_{L_{t}}\bm{\mathsf{O}}_{R_{t}}}{\bm{\mathscr{M}}} \hspace{-3pt}=\hspace{-3pt} \underset{\bm{\mathsf{O}}_{L_{t}}|x_{t\hspace{-1pt}-\hspace{-1pt}1}d_{t\hspace{-1pt}-\hspace{-1pt}1}}{\bm{\mathscr{F}}}
\times_{x_{t\hspace{-1pt}-\hspace{-1pt}1}d_{t\hspace{-1pt}-\hspace{-1pt}1}}\underset{\bm{\mathsf{O}}_{R_{t}}|x_{t\hspace{-1pt}-\hspace{-1pt}1}d_{t\hspace{-1pt}-\hspace{-1pt}1}}{\bm{\mathscr{F}}}
\times_{x_{t\hspace{-1pt}-\hspace{-1pt}1}d_{t\hspace{-1pt}-\hspace{-1pt}1}}\underset{x_{t\hspace{-1pt}-\hspace{-1pt}1}d_{t\hspace{-1pt}-\hspace{-1pt}1}}{\bm{\mathscr{K}}},
\end{align}
where  $\underset{x_{t\hspace{-1pt}-\hspace{-1pt}1}d_{t\hspace{-1pt}-\hspace{-1pt}1}}{\bm{\mathscr{K}}} = \underset{x_{t\hspace{-1pt}-\hspace{-1pt}1}d_{t\hspace{-1pt}-\hspace{-1pt}1}x_{t\hspace{-1pt}-\hspace{-1pt}1}d_{t\hspace{-1pt}-\hspace{-1pt}1}}{\bm{\mathscr{K}}}$ and we omitted the duplicated modes.

We express the inverse of tensor $\underset{\bm{\mathsf{O}}_{R_{t}}|x_{t-1}d_{t-1}}{\bm{\mathscr{F}}}$ from the above equation 
\begin{align*}
\underset{\bm{\mathsf{O}}_{R_{t}}|x_{t\hspace{-1pt}-\hspace{-1pt}1}d_{t\hspace{-1pt}-\hspace{-1pt}1}}{\bm{\mathscr{F}}^{-1}}
= \underset{\bm{\mathsf{O}}_{L_{t}}\bm{\mathsf{O}}_{R_{t}}}{\bm{\mathscr{M}}^{-1}}\times_{\bm{\mathsf{O}}_{L_{t}}} \underset{\bm{\mathsf{O}}_{L_{t}}|x_{t\hspace{-1pt}-\hspace{-1pt}1}d_{t\hspace{-1pt}-\hspace{-1pt}1}}{\bm{\mathscr{F}}}
\times_{x_{t\hspace{-1pt}-\hspace{-1pt}1}d_{t\hspace{-1pt}-\hspace{-1pt}1}}\underset{x_{t\hspace{-1pt}-\hspace{-1pt}1}d_{t\hspace{-1pt}-\hspace{-1pt}1}}{\bm{\mathscr{K}}},
\end{align*}
\noindent where tensor  $\underset{\bm{\mathsf{O}}_{R_{t}}|x_{t\hspace{-1pt}-\hspace{-1pt}1}d_{t\hspace{-1pt}-\hspace{-1pt}1}}{\bm{\mathscr{F}}}$ is inverted with respect to mode $\bm{\mathsf{O}}_{R_{t}}$, while $\underset{\bm{\mathsf{O}}_{L_{t}}\bm{\mathsf{O}}_{R_{t}}}{\bm{\mathscr{M}}}$ is inverted with respect to mode $\bm{\mathsf{O}}_{L_{t}}$. Substituting back to \eqref{eq:tensorX}, we get
\begin{align*}
\underset{\bm{\mathsf{O}}_{R_{t}}o_t\bm{\mathsf{O}}_{R_{t}}}{\tilde{\bm{\mathscr{X}}}} \hspace{-6pt}=
\underset{\bm{\mathsf{O}}_{L_{t}}\bm{\mathsf{O}}_{R_{t}}}{\bm{\mathscr{M}}^{-1}}
\times_{\bm{\mathsf{O}}_{L_{t}}}\underset{\bm{\mathsf{O}}_{L_{t}}|x_{t\hspace{-1pt}-\hspace{-1pt}1}d_{t\hspace{-1pt}-\hspace{-1pt}1}}{\bm{\mathscr{F}}}
\times_{x_{t\hspace{-1pt}-\hspace{-1pt}1}d_{t\hspace{-1pt}-\hspace{-1pt}1}}\underset{x_{t\hspace{-1pt}-\hspace{-1pt}1}d_{t\hspace{-1pt}-\hspace{-1pt}1}}{\bm{\mathscr{K}}}\times_{x_{t\hspace{-1pt}-\hspace{-1pt}1}d_{t\hspace{-1pt}-\hspace{-1pt}1}}
\hspace{-2pt}\left(\underset{x_{t}x_t|x_{t\hspace{-1pt}-\hspace{-1pt}1}d_{t\hspace{-1pt}-\hspace{-1pt}1}d_{t\hspace{-1pt}-\hspace{-1pt}1}}{\bm{\mathscr{X}}}\hspace{-2pt}\times_{x_t}\hspace{-2pt}\underset{o_{t}|x_{t}}{\bm{\mathscr{F}}}\right)
\hspace{-2pt}\times_{x_td_{t-1}}\hspace{-3pt}\underset{\bm{\mathsf{O}}_{R_{t}}|x_td_{t-1}}{\bm{\mathscr{F}}}.
\end{align*}

\noindent Considering the last five factors and multiplying them together, we obtain 
\begin{align*}
\underset{\bm{\mathsf{O}}_{L_{t}}\bm{\mathsf{O}}_{R_{t}}o_t}{\bm{\mathscr{M}}} = 
\underset{\bm{\mathsf{O}}_{L_{t}}|x_{t\hspace{-1pt}-\hspace{-1pt}1}d_{t\hspace{-1pt}-\hspace{-1pt}1}}{\bm{\mathscr{F}}}
\times_{x_{t\hspace{-1pt}-\hspace{-1pt}1}d_{t\hspace{-1pt}-\hspace{-1pt}1}}\underset{x_{t\hspace{-1pt}-\hspace{-1pt}1}d_{t\hspace{-1pt}-\hspace{-1pt}1}}{\bm{\mathscr{K}}}\times_{x_{t-1}d_{t-1}}
\hspace{-2pt}\left(\underset{x_{t}x_t|x_{t\hspace{-1pt}-\hspace{-1pt}1}d_{t\hspace{-1pt}-\hspace{-1pt}1}d_{t\hspace{-1pt}-\hspace{-1pt}1}}{\bm{\mathscr{X}}}\hspace{-2pt}\times_{x_t}\hspace{-2pt}\underset{o_{t}|x_{t}}{\bm{\mathscr{F}}}\right)
\hspace{-2pt}\times_{x_td_{t-1}}\hspace{-3pt}\underset{\bm{\mathsf{O}}_{R_{t}}|x_td_{t-1}}{\bm{\mathscr{F}}}.
\end{align*}
Finally, \eqref{eq:tensorX} can now be written as 
\begin{align}
\label{eq:tensorXobs}
\underset{\bm{\mathsf{O}}_{R_{t}}o_t\bm{\mathsf{O}}_{R_{t}}}{\tilde{\bm{\mathscr{X}}}} = \underset{\bm{\mathsf{O}}_{L_{t}}\bm{\mathsf{O}}_{R_{t}}}{\bm{\mathscr{M}}^{-1}} \times_{\bm{\mathsf{O}}_{L_{t}}}\underset{\bm{\mathsf{O}}_{L_{t}}\bm{\mathsf{O}}_{R_{t}}o_t}{\bm{\mathscr{M}}},
\end{align}
\noindent where the right hand side can now be estimated directly from data, without the need for the model parameters.

\subsubsection[Computation of Tensor O]{Computation of Tensor $\underset{o_to_t}{\tilde{\bm{\mathscr{O}}}}$}

Finally, we consider the tensor 
\begin{align}
\label{eq:tensorO}
\underset{o_to_t}{\tilde{\bm{\mathscr{O}}}} = \underset{o_{t}|x_{t}}{\bm{\mathscr{F}}^{-1}}\times_{x_t}\underset{o_t|x_t}{\bm{\mathscr{O}}}.
\end{align}

\noindent The conditional independence relationship can take the form
\begin{align*}
%\label{eq:tensCondInd3}
\underset{o_to_{t+1}}{\bm{\mathscr{M}}} \hspace{-3pt}=\hspace{-3pt} \underset{o_t|x_t}{\bm{\mathscr{F}}}
\times_{x_t}\underset{o_{t+1}|x_t}{\bm{\mathscr{F}}}
\times_{x_t}\underset{x_t}{\bm{\mathscr{K}}}.
\end{align*}

\noindent Expressing the inverse of $\underset{o_t|x_t}{\bm{\mathscr{F}}}$ 
\begin{align*}
\underset{o_t|x_t}{\bm{\mathscr{F}}^{-1}} = \underset{o_to_{t+1}}{\bm{\mathscr{M}}^{-1}} 
\times_{o_{t+1}}\underset{o_{t+1}|x_t}{\bm{\mathscr{F}}}
\times_{x_t}\underset{x_t}{\bm{\mathscr{K}}},
\end{align*}

\noindent and substituting in \eqref{eq:tensorO}, we get

\begin{align}
\label{eq:tensorOobs}
\underset{o_to_t}{\tilde{\bm{\mathscr{O}}}} &= \underset{o_to_{t+1}}{\bm{\mathscr{M}}^{-1}}
\times_{o_{t+1}}\underset{o_{t+1}|x_t}{\bm{\mathscr{F}}}
\times_{x_t}\underset{x_t}{\bm{\mathscr{K}}}\times_{x_t}\underset{o_t|x_t}{\bm{\mathscr{O}}}\nonumber\\
&= \underset{o_to_{t+1}}{\bm{\mathscr{M}}^{-1}}
\times_{o_{t+1}}\underset{o_to_{t+1}}{\bm{\mathscr{M}}}.
\end{align}

%% file: basicSpectralAlg.tex
The basic version of the spectral HSMM algorithm to compute $\underset{o_1,\ldots,o_T}{\bm{\mathscr{P}}}$ entirely using the observed variables can be described as a two step process: in the learning step, compute $\underset{\bm{\mathsf{O}}_{R_{t-1}}\bm{\mathsf{O}}_{R_{t}}}{\tilde{\bm{\mathscr{D}}}}$, $\underset{\bm{\mathsf{O}}_{R_{t-1}}o_t\bm{\mathsf{O}}_{R_{t}}}{\tilde{\bm{\mathscr{X}}}}$ and $\underset{o_to_t}{\tilde{\bm{\mathscr{O}}}}$ for each $t$ using \eqref{eq:DtensorObs}, \eqref{eq:tensorXobs} and \eqref{eq:tensorOobs}  from the training data. In the inference step, use \eqref{eq:mesTensorObserv} to compute $p(\mathbf{S}^{test})$. Algorithm \ref{algSpectralBasic} shows its basic version and Figure \ref{fig:jttensors} shows the graphical representation of this algorithm in terms of the transformed junction tree of Figure \ref{fig:jthsmm}.

As an example, consider the learning step of the algorithm and the computation of tensor in \eqref{eq:DtensorObs}, i.e.,
\begin{align*}
\underset{\bm{\mathsf{O}}_{R_{t\hspace{-1pt}-\hspace{-1pt}1}}\bm{\mathsf{O}}_{R_{t}}}{\tilde{\bm{\mathscr{D}}}} = \underset{\bm{\mathsf{O}}_{L_{t\hspace{-1pt}-\hspace{-1pt}1}}\bm{\mathsf{O}}_{R_{t\hspace{-1pt}-\hspace{-1pt}1}}}{\bm{\mathscr{M}}^{-1}}\times_{\bm{\mathsf{O}}_{L_{t\hspace{-1pt}-\hspace{-1pt}1}}}~~~\underset{\bm{\mathsf{O}}_{L_{t\hspace{-1pt}-\hspace{-1pt}1}}\bm{\mathsf{O}}_{R_{t}}}{\bm{\mathscr{M}}}.
\end{align*}
For a fixed $t$, we estimate each entry of $\underset{\bm{\mathsf{O}}_{L_{t\hspace{-1pt}-\hspace{-1pt}1}}\bm{\mathsf{O}}_{R_{t\hspace{-1pt}-\hspace{-1pt}1}}}{\bm{\mathscr{M}}}$ from the frequency of co-occurrence of tuples of the observed symbols $\{\ldots, o_{t-3}, o_{t-2}, o_{t+1}, o_{t+2}, \ldots\}$ in the given dataset (the sets $\bm{\mathsf{O}}_{L_{t-1}}$ and $\bm{\mathsf{O}}_{R_{t-1}}$ were defined at the beginning of Section \ref{sec:obsTensors}). Next, following our discussion after the equation \eqref{eq:Finv}, we invert $\underset{\bm{\mathsf{O}}_{L_{t\hspace{-1pt}-\hspace{-1pt}1}}\bm{\mathsf{O}}_{R_{t\hspace{-1pt}-\hspace{-1pt}1}}}{\bm{\mathscr{M}}^{-1}}$ along the modes $\bm{\mathsf{O}}_{L_{t-1}}$. For this, we matrisize the tensor so that the modes $\bm{\mathsf{O}}_{L_{t-1}}$ are associated with columns and $\bm{\mathsf{O}}_{R_{t-1}}$ with rows in matrix  $\underset{\bm{\mathsf{O}}_{R_{t\hspace{-1pt}-\hspace{-1pt}1}}\bm{\mathsf{O}}_{L_{t\hspace{-1pt}-\hspace{-1pt}1}}}{{\mathbf{M}}}$ (see Section \ref{sec:notation} for the discussion on tensor matrisization and inversion). Finally, we compute the right inverse of the matrix to obtain $\underset{\bm{\mathsf{O}}_{R_{t\hspace{-1pt}-\hspace{-1pt}1}}\bm{\mathsf{O}}_{L_{t\hspace{-1pt}-\hspace{-1pt}1}}}{{\mathbf{M}}^{-1}}$. Similarly, we estimate the tensor $\underset{\bm{\mathsf{O}}_{L_{t\hspace{-1pt}-\hspace{-1pt}1}}\bm{\mathsf{O}}_{R_{t}}}{\bm{\mathscr{M}}}$ using the corresponding co-occurrences of the observed symbols. Matrisizing the result, so that the rows correspond to the modes $\bm{\mathsf{O}}_{L_{t-1}}$ and the columns to $\bm{\mathsf{O}}_{R_{t}}$, we get the matrix 
$\underset{\bm{\mathsf{O}}_{L_{t\hspace{-1pt}-\hspace{-1pt}1}}\bm{\mathsf{O}}_{R_{t}}}{{\mathbf{M}}}$. The multiplication $\underset{\bm{\mathsf{O}}_{R_{t\hspace{-1pt}-\hspace{-1pt}1}}\bm{\mathsf{O}}_{L_{t\hspace{-1pt}-\hspace{-1pt}1}}}{{\mathbf{M}}^{-1}}~~\cdot~~\underset{\bm{\mathsf{O}}_{L_{t\hspace{-1pt}-\hspace{-1pt}1}}\bm{\mathsf{O}}_{R_{t}}}{{\mathbf{M}}} = \underset{\bm{\mathsf{O}}_{R_{t\hspace{-1pt}-\hspace{-1pt}1}}\bm{\mathsf{O}}_{R_{t}}}{\tilde{{\mathbf{D}}}}$ produces a matrix, which is then converted to a tensor to get the final result in \eqref{eq:DtensorObs}.

In the inference step we perform tensor multiplications for each $t$ running along the length of the testing sequence. The only nuance here is that before multiplying the tensor $\underset{o_to_t}{\tilde{\bm{\mathscr{O}}}}$ with others, the second mode $o_t$, whose dimension is $n_o$ is collapsed into a scalar. This operation is denoted as $\underset{o_to_t}{\tilde{\bm{\mathscr{O}}}}\Big|_{o_t=o_t^{test}}$, which means that based on the value of the $t$th symbol in testing sequence, we select the column corresponding to the element $o_t^{test}$. For example, if $\underset{o_to_t}{\tilde{\bm{\mathscr{O}}}} \in \mathbb{R}^{10\times10}$ and $o_t^{test} = 3$ then $\underset{o_to_t}{\tilde{\bm{\mathscr{O}}}}\Big|_{o_t=o_t^{test}} \in \mathbb{R}^{10\times1}$, a third column in the original matrix.   

\begin{algorithm}[!t]
   \caption{Basic Spectral Algorithm for HSMM inference}
   \label{algSpectralBasic}
\begin{algorithmic}
   \STATE {\bfseries Input:} Training sequences: $\mathbf{S}^{i} = \{o_1^i,\ldots,o_{T_i}^i\}, i=1,\ldots,N$.\\ Testing sequence: $\mathbf{S}^{test} = \{o_1^{test},\ldots,o_{T}^{test}\}$.   %\strut
   \STATE {\bfseries Output:} $p(\mathbf{S}^{test})$
   \STATE
   \STATE {\bfseries Learning phase:}
   \FOR {{\bfseries all} $t$}
   \STATE Estimate $\underset{\bm{\mathsf{O}}_{R_{t\hspace{-1pt}-\hspace{-1pt}1}}\bm{\mathsf{O}}_{R_{t}}}{\tilde{\bm{\mathscr{D}}}}$, $\underset{\bm{\mathsf{O}}_{R_{t}}o_t\bm{\mathsf{O}}_{R_{t}}}{\tilde{\bm{\mathscr{X}}}}$ and $\underset{o_to_t}{\tilde{\bm{\mathscr{O}}}}$ from data $\{\mathbf{S}^1,\ldots,\mathbf{S}^N\}$ using equations \eqref{eq:DtensorObs}, \eqref{eq:tensorXobs} and \eqref{eq:tensorOobs}.
   \ENDFOR   
   \STATE
   \STATE {\bfseries Inference phase:}
   \STATE $p(\mathbf{S}^{test}) = 1$
   \FOR {$t=T$ {\bfseries down to} $t=1$}
   \STATE $p(\mathbf{S}^{test}) = p(\mathbf{S}^{test}) \times\underset{\bm{\mathsf{O}}_{R_{t\hspace{-1pt}-\hspace{-1pt}1}}\bm{\mathsf{O}}_{R_{t}}}{\tilde{\bm{\mathscr{D}}}} \times_{\bm{\mathsf{O}}_{R_{t}}} \left(\underset{\bm{\mathsf{O}}_{R_{t}}o_t\bm{\mathsf{O}}_{R_{t}}}{\tilde{\bm{\mathscr{X}}}}\times_{o_t}\underset{o_to_t}{\tilde{\bm{\mathscr{O}}}}\Big|_{o_t=o_t^{test}}\right)$
   \ENDFOR
\end{algorithmic}
\end{algorithm}

Analyzing \eqref{eq:DtensorObs}, \eqref{eq:tensorXobs} and \eqref{eq:tensorOobs}, we see that the computational complexity of the learning phase of the algorithm is determined by the tensor inverses and multiplications. For example, if in \eqref{eq:DtensorObs} we denote $|\bm{\mathsf{O}}_{R}| = |\bm{\mathsf{O}}_{L}| = \ell$ (in Section \ref{sec:observationVars} we will show that $\ell = \lceil 1 + \frac{\log n_d}{\log n_x}\rceil$), then $\underset{\bm{\mathsf{O}}_{L_{t\hspace{-1pt}-\hspace{-1pt}1}}\bm{\mathsf{O}}_{R_{t-1}}}{{{\mathbf{M}}}} \in \mathbb{R}^{n_o^\ell\times n_o^\ell}$ and $\underset{\bm{\mathsf{O}}_{L_{t\hspace{-1pt}-\hspace{-1pt}1}}\bm{\mathsf{O}}_{R_{t}}}{{{\mathbf{M}}}} \in \mathbb{R}^{n_o^\ell\times n_o^\ell}$. The computational complexity of the multiplications and inversions would then be $\mathcal{O}(n_o^{3\ell})$. Performing this computations for all $t$ and assuming that the length of the sequences is $T$, would result in $\mathcal{O}\left(n_o^{3\ell}T\right)$. Additionally, with $N$ training examples there will be a cost of $\mathcal{O}\left(\ell N T\right)$ to estimate the sample moments $\bm{\mathscr{M}}$, which is based on counting the co-occurrences of certain observable symbols. In the inference phase of the algorithm, we perform a series of tensor multiplications with the cost of $\mathcal{O}(n_o^{3\ell}T)$. 
%Thus, the overall cost of Algorithm \ref{algSpectralBasic} is then $\mathcal{O}\left((n_o^{3\ell} + \ell N)T\right)$.

%% file: effSpectralAlg.tex
Note that for large $\ell$ the accurate estimation of tensors $\bm{\mathscr{M}}$ for each $t$ will require large number of training sequences which  might not be available, leading to inaccurate and unstable computations. Observe, however, that for example the estimated sample-based tensors $\underset{\bm{\mathsf{O}}_{L_{t\hspace{-1pt}-\hspace{-1pt}1}}\bm{\mathsf{O}}_{R_{t}}}{\bm{\mathscr{M}}}$ in \eqref{eq:DtensorObs} for each $t$ estimate the same population quantity due to homogeneity of HSMM.  Thus, a \emph{novel} aspect of our work is the improvement of the accuracy and efficiency of the basic algorithm \ref{alg1} by exploiting the homogeneity property of HSMM and estimating the tensors ${\tilde{\bm{\mathscr{X}}}}$, ${\tilde{\bm{\mathscr{D}}}}$ and ${\tilde{\bm{\mathscr{O}}}}$ in the batch, by pooling the samples across different $t$ and then averaging the result. Thus, we compute only three tensors for all $t$, as opposed to computing these tensors for each $t$.

We show the details for computing the tensors ${\tilde{\bm{\mathscr{D}}}}$ in the batch form.    The derivations for other tensors ${\tilde{\bm{\mathscr{X}}}}$ and ${\tilde{\bm{\mathscr{O}}}}$ can be computed in a similar manner. Recall from \eqref{eq:DtensorObs} the form of  $\underset{\bm{\mathsf{O}}_{R_{t\hspace{-1pt}-\hspace{-1pt}1}}\bm{\mathsf{O}}_{R_{t}}}{\tilde{\bm{\mathscr{D}}}}$, and consider the following alternative expression, based on the sum over all $t$:
\begin{align}
\label{eq:efficStruct}
{\tilde{\bm{\mathscr{D}}}}= 
\left(\sum_{t}\underset{\bm{\mathsf{O}}_{L_{t-1}}\bm{\mathsf{O}}_{R_{t-1}}}{\bm{\mathscr{M}}}\right)^{-1}\times_{\bm{\mathsf{O}}_{L}}\left(\sum_{t}\underset{\bm{\mathsf{O}}_{L_{t-1}}\bm{\mathsf{O}}_{R_{t}}}{\bm{\mathscr{M}}}\right),
\end{align}
\noindent where $\bm{\mathsf{O}}_{L}$ denotes a generic mode of the averaged tensor $\bm{\mathscr{M}}$, corresponding to $\bm{\mathsf{O}}_{L_{t-1}}$ for all $t$. Note that in practice, instead of summation, we use averaging to avoid numerical overflow problems. It is equivalent to the considered expression in \eqref{eq:efficStruct}, since the term $\frac{1}{T}$ then cancels out. 
Since 
\begin{align}
\underset{\bm{\mathsf{O}}_{L_{t\hspace{-1pt}-\hspace{-1pt}1}}\bm{\mathsf{O}}_{R_{t\hspace{-1pt}-\hspace{-1pt}1}}}{\bm{\mathscr{M}}} \hspace{-3pt}=\hspace{-3pt} \underset{\bm{\mathsf{O}}_{L_{t\hspace{-1pt}-\hspace{-1pt}1}}|x_{t\hspace{-1pt}-\hspace{-1pt}1}d_{t\hspace{-1pt}-\hspace{-1pt}2}}{\bm{\mathscr{F}}}
\times_{x_{t\hspace{-1pt}-\hspace{-1pt}1}d_{t\hspace{-1pt}-\hspace{-1pt}2}}\underset{\bm{\mathsf{O}}_{R_{t\hspace{-1pt}-\hspace{-1pt}1}}|x_{t\hspace{-1pt}-\hspace{-1pt}1}d_{t\hspace{-1pt}-\hspace{-1pt}2}}{\bm{\mathscr{F}}}
\times_{x_{t\hspace{-1pt}-\hspace{-1pt}1}d_{t\hspace{-1pt}-\hspace{-1pt}2}}\underset{x_{t\hspace{-1pt}-\hspace{-1pt}1}d_{t\hspace{-1pt}-\hspace{-1pt}2}}{\bm{\mathscr{K}}},
\end{align}
the first term inside brackets can be rewritten as:
\begin{align}
\label{eq:DtensorLeft}
\sum_{t}&\underset{\bm{\mathsf{O}}_{L_{t\hspace{-1pt}-\hspace{-1pt}1}}|x_{t\hspace{-1pt}-\hspace{-1pt}1}d_{t\hspace{-1pt}-\hspace{-1pt}2}}{\bm{\mathscr{F}}}
\times_{x_{t\hspace{-1pt}-\hspace{-1pt}1}d_{t\hspace{-1pt}-\hspace{-1pt}2}}\underset{\bm{\mathsf{O}}_{R_{t\hspace{-1pt}-\hspace{-1pt}1}}|x_{t\hspace{-1pt}-\hspace{-1pt}1}d_{t\hspace{-1pt}-\hspace{-1pt}2}}{\bm{\mathscr{F}}}
\times_{x_{t\hspace{-1pt}-\hspace{-1pt}1}d_{t\hspace{-1pt}-\hspace{-1pt}2}}\underset{x_{t\hspace{-1pt}-\hspace{-1pt}1}d_{t\hspace{-1pt}-\hspace{-1pt}2}}{\bm{\mathscr{K}}}\nonumber\\
&\stackrel{(a)}{=}
\sum_{t}\underset{\bm{\mathsf{O}}_{R_{t-1}}|x_{t-1}d_{t-2}}{\bm{\mathscr{F}}}\times_{x_{t-1}d_{t-2}}\underset{\bm{\mathsf{O}}_{L_{t-1}}x_{t-1}d_{t-2}}{\overline{\bm{\mathscr{F}}}}\nonumber\\
&\stackrel{(b)}{=}
\underset{\bm{\mathsf{O}}_{R_{2}}|x_2d_{1}}{\bm{\mathscr{F}}}\times\left(\sum_{t}\underset{\bm{\mathsf{O}}_{L_{t-1}}x_{t-1}d_{t-2}}{\overline{\bm{\mathscr{F}}}}\right),
\end{align}
\noindent where in $(a)$ we combined the two factors, i.e., $\underset{\bm{\mathsf{O}}_{L_{t-1}}x_{t-1}d_{t-2}}{\overline{\bm{\mathscr{F}}}} = \underset{\bm{\mathsf{O}}_{L_{t\hspace{-1pt}-\hspace{-1pt}1}}|x_{t\hspace{-1pt}-\hspace{-1pt}1}d_{t\hspace{-1pt}-\hspace{-1pt}2}}{\bm{\mathscr{F}}}
\times_{x_{t\hspace{-1pt}-\hspace{-1pt}1}d_{t\hspace{-1pt}-\hspace{-1pt}2}}\underset{x_{t\hspace{-1pt}-\hspace{-1pt}1}d_{t\hspace{-1pt}-\hspace{-1pt}2}x_{t\hspace{-1pt}-\hspace{-1pt}1}d_{t\hspace{-1pt}-\hspace{-1pt}2}}{\bm{\mathscr{K}}}$ and in $(b)$ we used the homogeneity property of HSMM, i.e., the fact that $\underset{\bm{\mathsf{O}}_{R_{t-1}}|x_{t-1}d_{t-2}}{\bm{\mathscr{F}}}$ does not depend on time stamp $t$, and extracted one of the common factors, in fact, the first factor. Note that the term $\underset{\bm{\mathsf{O}}_{L_{t-1}}x_{t-1}d_{t-2}}{\overline{\bm{\mathscr{F}}}}$, on the other hand, does depend on $t$ since the factor $\underset{x_{t\hspace{-1pt}-\hspace{-1pt}1}d_{t\hspace{-1pt}-\hspace{-1pt}2}}{\bm{\mathscr{K}}}$, which represents the probability $p(x_{t-1},d_{t-2})$, changes as the time stamp $t$ changes.
%i.e., the fact that $\underset{\bm{\mathsf{O}}_{R_{t-1}}|x_{t-1}d_{t-2}}{\bm{\mathscr{F}}}$ does not depend on time $t$.

Similarly, since
\begin{align}
\label{eq:Mterm}
\underset{\bm{\mathsf{O}}_{L_{t\hspace{-1pt}-\hspace{-1pt}1}}\bm{\mathsf{O}}_{R_{t}}}{\bm{\mathscr{M}}} = \underset{\bm{\mathsf{O}}_{L_{t\hspace{-1pt}-\hspace{-1pt}1}}|x_{t-1}d_{t-2}}{\bm{\mathscr{F}}}\times_{x_{t-1}d_{t-2}}
\underset{x_{t-1}d_{t-2}}{\bm{\mathscr{K}}}
 \times_{x_{t-1}d_{t-2}}\underset{d_{t-1}|x_{t-1}x_{t-1}d_{t-2}}{\bm{\mathscr{D}}}\times_{x_{t-1}d_{t-1}}\underset{\bm{\mathsf{O}}_{R_{t}}|x_{t-1}d_{t-1}}{\bm{\mathscr{F}}},
\end{align}
rewrite the second term in \eqref{eq:efficStruct} as
\begin{align}
\label{eq:DtensorRight}
\sum_{t}&\underset{\bm{\mathsf{O}}_{L_{t\hspace{-1pt}-\hspace{-1pt}1}}|x_{t-1}d_{t-2}}{\bm{\mathscr{F}}}\times_{x_{t-1}d_{t-2}}\underset{x_{t-1}d_{t-2}}{\bm{\mathscr{K}}}
 \times_{x_{t-1}d_{t-2}}\underset{d_{t-1}|x_{t-1}x_{t-1}d_{t-2}}{\bm{\mathscr{D}}}\times_{x_{t-1}d_{t-1}}\underset{\bm{\mathsf{O}}_{R_{t}}|x_{t-1}d_{t-1}}{\bm{\mathscr{F}}} \nonumber\\
 &= \sum_{t}\underset{\bm{\mathsf{O}}_{L_{t-1}}x_{t-1}d_{t-2}}{\overline{\bm{\mathscr{F}}}}\times_{x_{t-1}d_{t-2}}
\underset{d_{t-1}|x_{t-1}x_{t-1}d_{t-2}}{\bm{\mathscr{D}}}\times_{x_{t-1}d_{t-1}}\underset{\bm{\mathsf{O}}_{R_{t}}|x_{t-1}d_{t-1}}{\bm{\mathscr{F}}}\nonumber\\
&=\left(\sum_{t}\underset{\bm{\mathsf{O}}_{L_{t-1}}x_{t-1}d_{t-2}}{\overline{\bm{\mathscr{F}}}}\right)\times
\underset{d_{2}|x_{2}x_2d_{1}}{\bm{\mathscr{D}}}\times_{x_2d_2}\underset{\bm{\mathsf{O}}_{R_{3}}|x_2d_{2}}{\bm{\mathscr{F}}},
\end{align}
\noindent where we used the transformations similar as in \eqref{eq:DtensorLeft}, i.e., the fact that the factors $\underset{d_{t-1}|x_{t-1}x_{t-1}d_{t-2}}{\bm{\mathscr{D}}}$ and $\underset{\bm{\mathsf{O}}_{R_{t}}|x_{t-1}d_{t-1}}{\bm{\mathscr{F}}}$ are homogeneous, independent of $t$.  Now if we multiply the inverse of \eqref{eq:DtensorLeft} with \eqref{eq:DtensorRight}, we get 
\begin{align}
\label{eq:eff1}
&\underset{\bm{\mathsf{O}}_{R_{2}}|x_2d_{1}}{\bm{\mathscr{F}}^{-1}}\times\left(\sum_{t}\underset{\bm{\mathsf{O}}_{L_{t-1}}x_{t-1}d_{t-2}}{\overline{\bm{\mathscr{F}}}}\right)^{-1}\times\left(\sum_{t}\underset{\bm{\mathsf{O}}_{L_{t-1}}x_{t-1}d_{t-2}}{\overline{\bm{\mathscr{F}}}}\right)\times
\underset{d_{2}|x_{2}x_2d_{1}}{\bm{\mathscr{D}}}\times\underset{\bm{\mathsf{O}}_{R_{3}}|x_2d_{2}}{\bm{\mathscr{F}}}\\
&=\underset{\bm{\mathsf{O}}_{R_{2}}|x_2d_{1}}{\bm{\mathscr{F}}^{-1}}\times_{x_2d_{1}}\underset{d_{2}|x_{2}x_2d_{1}}{\bm{\mathscr{D}}}\times_{x_2d_{2}}\underset{\bm{\mathsf{O}}_{R_{3}}|x_2d_{2}}{\bm{\mathscr{F}}}\nonumber\\
\label{eq:eff2}
& = \underset{\bm{\mathsf{O}}_{R_{2}}\bm{\mathsf{O}}_{R_{3}}}{\tilde{\bm{\mathscr{D}}}} = \underset{\bm{\mathsf{O}}_{R_{t\hspace{-1pt}-\hspace{-1pt}1}}\bm{\mathsf{O}}_{R_{t}}}{\tilde{\bm{\mathscr{D}}}},
\end{align}
\noindent where in \eqref{eq:eff1} we used the fact that the order in which tensors are multiplied is irrelevant and also the fact that the terms in parenthesis are invertible. This is due to the fact that the set of observations $\bm{\mathsf{O}}_{L_{t-1}}$ for all $t$ is selected so as to make each of the summand invertible (see Section  \ref{sec:observationVars} for the details about the choice of $\bm{\mathsf{O}}_{L_{t-1}}$). Moreover, in \eqref{eq:eff2} we used the definition of   $\underset{\bm{\mathsf{O}}_{R_{t-1}}\bm{\mathsf{O}}_{R_{t}}}{\tilde{\bm{\mathscr{D}}}}$
\begin{align*}
\underset{\bm{\mathsf{O}}_{R_{t\hspace{-1pt}-\hspace{-1pt}1}}\bm{\mathsf{O}}_{R_{t}}}{\tilde{\bm{\mathscr{D}}}}\hspace{-2pt}=\hspace{-2pt} \underset{\bm{\mathsf{O}}_{R_{t\hspace{-1pt}-\hspace{-1pt}1}}|x_{t\hspace{-1pt}-\hspace{-1pt}1}d_{t\hspace{-1pt}-\hspace{-1pt}2}}{\bm{\mathscr{F}}^{-1}}\hspace{-2pt}\times\underset{d_{t\hspace{-1pt}-\hspace{-1pt}1}|x_{t\hspace{-1pt}-\hspace{-1pt}1}d_{t\hspace{-1pt}-\hspace{-1pt}2}}{\bm{\mathscr{D}}}\hspace{-3pt}\times\underset{\bm{\mathsf{O}}_{R_{t}}|x_{t\hspace{-1pt}-\hspace{-1pt}1}d_{t\hspace{-1pt}-\hspace{-1pt}1}}{\bm{\mathscr{F}}},
\end{align*}
together with the homogeneity property of HSMM.
%Note that
%\begin{align*}
%\left(\sum_{t}\underset{\bm{\mathsf{O}}_{L_{t-1}}x_{t-1}d_{t-2}}{\bm{\mathscr{F}}}\right)^{-1}\left(\sum_{t}\unde %rset{\bm{\mathsf{O}}_{L_{t-1}}x_{t-1}d_{t-2}}{\bm{\mathscr{F}}}\right) = \bm{\mathscr{I}}
%\end{align*}
%
 
\noindent Therefore, we can conclude that the batch form of the tensor takes the form:
\begin{align}
\label{eq:DtensorAcc}
\tilde{\bm{\mathscr{D}}}
 =\left(\sum_{t}\underset{\bm{\mathsf{O}}_{L_{t-1}}\bm{\mathsf{O}}_{R_{t-1}}}{\bm{\mathscr{M}}}\right)^{-1}\times_{\bm{\mathsf{O}}_{L}}\left(\sum_{t}\underset{\bm{\mathsf{O}}_{L_{t-1}}\bm{\mathsf{O}}_{R_{t}}}{\bm{\mathscr{M}}}\right).
\end{align} 
 
%
%\begin{align*}
%\underset{\bm{\mathsf{O}}_{L_{t-1}}x_{t-1}d_{t-2}}{\bm{\mathscr{F}}} = %\underset{\bm{\mathsf{O}}_{L_{t-1}}|x_{t-1}d_{t-2}}{\bm{\mathscr{F}}}\times\underset{x_{t-1}d_{t-2}}{\bm{\mathscr{ %T}}}
%\end{align*}
%
%\noindent where the first term is full rank  and second term is a diagonal tensor and is also full rank because of %assumptions (1) and (2) on model parameters.

\begin{algorithm}[!t]
   \caption{Efficient Spectral Algorithm for HSMM inference}
   \label{alg3}
\begin{algorithmic}
   \STATE {\bfseries Input:} Training sequences: $\mathbf{S}^{i} = \{o_1^i,\ldots,o_{T_i}^i\}, i=1,\ldots,N$.\\ Testing sequence: $\mathbf{S}^{test} = \{o_1^{test},\ldots,o_{T}^{test}\}$.   %\strut
   \STATE {\bfseries Output:} $p(\mathbf{S}^{test})$
   \STATE
   \STATE {\bfseries Learning phase:}
   \STATE Estimate $\tilde{\bm{\mathscr{D}}}, \tilde{\bm{\mathscr{X}}}$ and $\tilde{\bm{\mathscr{O}}}$ from data $\{\mathbf{S}^1,\ldots,\mathbf{S}^N\}$ using equations \eqref{eq:DtensorAcc}, \eqref{eq:XtensorAcc} and \eqref{eq:OtensorAcc}.
   \STATE
   \STATE {\bfseries Inference phase:}
   \STATE $p(\mathbf{S}^{test}) = 1$
   \FOR {$i=T$ {\bfseries down to} $i=1$}
   \STATE $p(\mathbf{S}^{test}) = p(\mathbf{S}^{test}) \times\tilde{\bm{\mathscr{D}}} \times \left(\tilde{\bm{\mathscr{X}}}\times\tilde{\bm{\mathscr{O}}}|_{o=o_i^{test}}\right)$
   \ENDFOR
\end{algorithmic}
\end{algorithm}

\noindent Similar derivations can be carried out to obtain the rest of the tensors in the batch form:
\begin{align}
\label{eq:XtensorAcc}
\tilde{\bm{\mathscr{X}}} &= \left(\sum_{t}\underset{\bm{\mathsf{O}}_{L_{t}}\bm{\mathsf{O}}_{R_{t}}}{\bm{\mathscr{M}}}\right)^{-1}\times_{\bm{\mathsf{O}}_{L}}\left(\sum_{t}\underset{\bm{\mathsf{O}}_{L_{t}}\bm{\mathsf{O}}_{R_{t}}o_t}{\bm{\mathscr{M}}}\right)\\
\label{eq:OtensorAcc}
\tilde{\bm{\mathscr{O}}} &= \left(\sum_{t}\underset{o_to_{t+1}}{\bm{\mathscr{M}}}\right)^{-1}\times_{o}\left(\sum_{t}\underset{o_to_{t+1}}{\bm{\mathscr{M}}}\right).
\end{align}

\noindent where in the last expression the mode $o$ corresponds to the mode $o_{t_{t+1}}$ after averaging of tensor $\underset{o_to_{t+1}}{\bm{\mathscr{M}}}$ for all~$t$.

Analyzing \eqref{eq:DtensorAcc}, \eqref{eq:XtensorAcc} and \eqref{eq:OtensorAcc}, we see that the computational complexity of the learning phase of the algorithm is now $\mathcal{O}\left((n_o^{2\ell} + \ell N)T\right)$, mainly determined by the tensor additions and the estimation of the sample moments $\bm{\mathscr{M}}$.  The number of inverses and multiplications is now fixed and independent of sequence length $T$. Specifically, there will be three tensor multiplications and inversions for a total cost of $\mathcal{O}(n_o^{3\ell})$. The computational complexity of the inference phase is $\mathcal{O}(n_o^{3\ell}T)$, which is the same as for Algorithm \ref{algSpectralBasic}. 
%Thus, the overall cost of Algorithm \ref{alg3} is $\mathcal{O}\left((n_o^{3\ell} + \ell N)T\right)$.

Note that such a batch tensor computation significantly improves the accuracy of the resulting spectral algorithm. In part, this is due to the fact that we now use more data to estimate the tensors as compared to the original form \eqref{eq:mesTensorObserv}. The estimates obtained in this form have lower variance, which in turn ensures that the inverses we compute in \eqref{eq:DtensorAcc}, \eqref{eq:XtensorAcc} and \eqref{eq:OtensorAcc} are more stable and accurate.

%The next section addresses the question, which we left unanswered - how and what number of observations to select for sets $\bm{\mathsf{O}}_{R}$ and $\bm{\mathsf{O}}_{L}$.

%% file: FullRankStructure.tex
In Section \ref{sec:compD}, when we derived the equations \eqref{eq:DtensorObs}, \eqref{eq:tensorXobs} and \eqref{eq:tensorOobs}, we glossed over the question of the existence of tensor inverses $\underset{\bm{\mathsf{O}}_{L_{t\hspace{-1pt}-\hspace{-1pt}1}}\bm{\mathsf{O}}_{R_{t\hspace{-1pt}-\hspace{-1pt}1}}}{\bm{\mathscr{M}}^{-1}}$, $\underset{\bm{\mathsf{O}}_{L_{t}}\bm{\mathsf{O}}_{R_{t}}}{\bm{\mathscr{M}}^{-1}}$ and $\underset{o_to_{t+1}}{\bm{\mathscr{M}}^{-1}}$. In this section, our task is to analyze the rank structure of these tensors and impose restrictions on the sets $\bm{\mathsf{O}}_{L}$ and $\bm{\mathsf{O}}_{R}$ to ensure that the rank conditions are satisfied. For example, consider equation \eqref{eq:DtensorObs} and expand all its terms using \eqref{eq:tensCondInd} to get
\begin{align*}
\underset{\bm{\mathsf{O}}_{R_{t\hspace{-1pt}-\hspace{-1pt}1}}\bm{\mathsf{O}}_{R_{t}}}{\tilde{\bm{\mathscr{D}}}} \hspace{-0pt}=\hspace{-0pt}
\underset{\bm{\mathsf{O}}_{R_{t\hspace{-1pt}-\hspace{-1pt}1}}|x_{t\hspace{-1pt}-\hspace{-1pt}1}d_{t\hspace{-1pt}-\hspace{-1pt}2}}{\bm{\mathscr{F}}^{-1}}\hspace{-0pt}\times\hspace{-0pt}\overbracket{\underset{\bm{\mathsf{O}}_{L_{t\hspace{-1pt}-\hspace{-1pt}1}}|x_{t\hspace{-1pt}-\hspace{-1pt}1}d_{t\hspace{-1pt}-\hspace{-1pt}2}}{\bm{\mathscr{F}}^{-1}}\hspace{-0pt}\times\hspace{-0pt}\underbracket{\underset{x_{t\hspace{-1pt}-\hspace{-1pt}1}d_{t\hspace{-1pt}-\hspace{-1pt}2}}{\bm{\mathscr{K}}^{-1}}\hspace{-0pt}\times\hspace{-0pt} \underset{x_{t\hspace{-1pt}-\hspace{-1pt}1}d_{t\hspace{-1pt}-\hspace{-1pt}2}}{\bm{\mathscr{K}}}}\hspace{-2pt}\times\hspace{-2pt} \underset{\bm{\mathsf{O}}_{L_{t\hspace{-1pt}-\hspace{-1pt}1}}|x_{t\hspace{-1pt}-\hspace{-1pt}1}d_{t\hspace{-1pt}-\hspace{-1pt}2}}{\bm{\mathscr{F}}}}
\hspace{-3pt}\times\hspace{-0pt}\underset{d_{t\hspace{-1pt}-\hspace{-1pt}1}|x_{t\hspace{-1pt}-\hspace{-1pt}1}x_{t\hspace{-1pt}-\hspace{-1pt}1}d_{t\hspace{-1pt}-\hspace{-1pt}2}}{\bm{\mathscr{D}}}\hspace{-0pt}\times\hspace{-0pt}\underset{\bm{\mathsf{O}}_{R_{t}}|x_{t\hspace{-1pt}-\hspace{-1pt}1}d_{t\hspace{-1pt}-\hspace{-1pt}1}}{\bm{\mathscr{F}}},
\end{align*}
\noindent where we dropped the multiplication subscripts and some of the duplicated modes, which can be inferred from the context. Observe, that in order for the above equation to produce \eqref{eq:tensorD}, the terms in the middle must multiply out into identity tensor
\begin{align}
\label{eq:ident1}
\underset{x_{t-1}d_{t-2}}{\bm{\mathscr{I}}} = \underset{x_{t\hspace{-1pt}-\hspace{-1pt}1}d_{t\hspace{-1pt}-\hspace{-1pt}2}}{\bm{\mathscr{K}}^{-1}}\times_{x_{t\hspace{-1pt}-\hspace{-1pt}1}d_{t\hspace{-1pt}-\hspace{-1pt}2}}
\underset{x_{t-1}d_{t-2}}{\bm{\mathscr{K}}}
\quad\quad
\underset{x_{t-1}d_{t-2}}{\bm{\mathscr{I}}} =\underset{\bm{\mathsf{O}}_{L_{t\hspace{-1pt}-\hspace{-1pt}1}}|x_{t\hspace{-1pt}-\hspace{-1pt}1}d_{t\hspace{-1pt}-\hspace{-1pt}2}}{\bm{\mathscr{F}}^{-1}}\times_{\bm{\mathsf{O}}_{L_{t\hspace{-1pt}-\hspace{-1pt}1}}} \underset{\bm{\mathsf{O}}_{L_{t\hspace{-1pt}-\hspace{-1pt}1}}|x_{t-1}d_{t-2}}{\bm{\mathscr{F}}}.
\end{align}
\noindent Moreover, recall that $\underset{\bm{\mathsf{O}}_{R_{t\hspace{-1pt}-\hspace{-1pt}1}}|x_{t\hspace{-1pt}-\hspace{-1pt}1}d_{t\hspace{-1pt}-\hspace{-1pt}2}}{\bm{\mathscr{F}}}$ was originally introduced as part of the identity tensor
\begin{align}
\label{eq:ident3}
\underset{x_{t\hspace{-1pt}-\hspace{-1pt}1}d_{t\hspace{-1pt}-\hspace{-1pt}2}}{\bm{\mathscr{I}}} = \underset{\bm{\mathsf{O}}_{R_{t\hspace{-1pt}-\hspace{-1pt}1}}|x_{t\hspace{-1pt}-\hspace{-1pt}1}d_{t\hspace{-1pt}-\hspace{-1pt}2}}{\bm{\mathscr{F}}^{-1}}\times_{\bm{\mathsf{O}}_{R_{t\hspace{-1pt}-\hspace{-1pt}1}}}
\underset{\bm{\mathsf{O}}_{R_{t\hspace{-1pt}-\hspace{-1pt}1}}|x_{t\hspace{-1pt}-\hspace{-1pt}1}d_{t\hspace{-1pt}-\hspace{-1pt}2}}{\bm{ \mathscr{F}}},
\end{align}
\noindent therefore, we can conclude that for \eqref{eq:DtensorObs} to exist, the identity statements in \eqref{eq:ident1} and \eqref{eq:ident3} must be satisfied. These statements have implications for the ranks of $\underset{x_{t\hspace{-1pt}-\hspace{-1pt}1}d_{t\hspace{-1pt}-\hspace{-1pt}2}}{\bm{\mathscr{K}}}$, $\underset{\bm{\mathsf{O}}_{L_{t\hspace{-1pt}-\hspace{-1pt}1}}|x_{t\hspace{-1pt}-\hspace{-1pt}1}d_{t\hspace{-1pt}-\hspace{-1pt}2}}{\bm{\mathscr{F}}}$ and $\underset{\bm{\mathsf{O}}_{R_{t\hspace{-1pt}-\hspace{-1pt}1}}|x_{t\hspace{-1pt}-\hspace{-1pt}1}d_{t\hspace{-1pt}-\hspace{-1pt}2}}{\bm{\mathscr{F}}}$, which in turn determine the length of the observation sequences $\bm{\mathsf{O}}_{L_{t\hspace{-1pt}-\hspace{-1pt}1}}$ and $\bm{\mathsf{O}}_{R_{t\hspace{-1pt}-\hspace{-1pt}1}}$.

Since $\underset{x_{t\hspace{-1pt}-\hspace{-1pt}1}d_{t\hspace{-1pt}-\hspace{-1pt}2}}{\bm{\mathscr{K}}}$ represents a distribution $p(x_{t-1}d_{t-2})$, its matrisized version is a diagonal matrix with probability $p(x_{t-1}d_{t-2})$ on the diagonal.  Using  assumptions $A1$ and $A2$, it can be concluded that the diagonal elements in this matrix are non-zero and it has rank $n_xn_d$, it is thus invertible and so the first equation in \eqref{eq:ident1} is satisfied.

Next, consider the second equation in \eqref{eq:ident1} and recall from Section \ref{sec:notation} that if we matrisize the tensor as $\underset{\bm{\mathsf{O}}_{L_{t\hspace{-1pt}-\hspace{-1pt}1}}|x_{t\hspace{-1pt}-\hspace{-1pt}1}d_{t\hspace{-1pt}-\hspace{-1pt}2}}{{ \mathbf{F}}} \in \mathbb{R}^{n_o^{|\bm{\mathsf{O}}_{L_{t\hspace{-1pt}-\hspace{-1pt}1}}|}\times n_xn_d}$ then $\mathbf{F}$ must have full column rank $n_xn_d$ for the proper inverse to exist, implying $n_o^{|\bm{\mathsf{O}}_{L_{t\hspace{-1pt}-\hspace{-1pt}1}}|} \geq n_xn_d$. Similarly, $\underset{\bm{\mathsf{O}}_{R_{t\hspace{-1pt}-\hspace{-1pt}1}}|x_{t\hspace{-1pt}-\hspace{-1pt}1}d_{t\hspace{-1pt}-\hspace{-1pt}2}}{\bm{\mathscr{F}}}$ in \eqref{eq:ident3} must have rank $n_xn_d$. As a consequence of the above, the tensor 
\begin{align}
\underset{\bm{\mathsf{O}}_{L_{t\hspace{-1pt}-\hspace{-1pt}1}}\bm{\mathsf{O}}_{R_{t\hspace{-1pt}-\hspace{-1pt}1}}}{\bm{\mathscr{M}}} \hspace{-3pt}=\hspace{-3pt} \underset{\bm{\mathsf{O}}_{L_{t\hspace{-1pt}-\hspace{-1pt}1}}|x_{t\hspace{-1pt}-\hspace{-1pt}1}d_{t\hspace{-1pt}-\hspace{-1pt}2}}{\bm{\mathscr{F}}}
\times\underset{\bm{\mathsf{O}}_{R_{t\hspace{-1pt}-\hspace{-1pt}1}}|x_{t\hspace{-1pt}-\hspace{-1pt}1}d_{t\hspace{-1pt}-\hspace{-1pt}2}}{\bm{\mathscr{F}}}
\times\underset{x_{t\hspace{-1pt}-\hspace{-1pt}1}d_{t\hspace{-1pt}-\hspace{-1pt}2}}{\bm{\mathscr{K}}}
\end{align}

\noindent will have rank $n_xn_d$ and, in general, is rank-deficient.

The argument above can also be used  to show that $\underset{\bm{\mathsf{O}}_{L_{t}}\bm{\mathsf{O}}_{R_{t}}}{\bm{\mathscr{M}}}$ has rank $n_xn_d$ since in \eqref{eq:tensCondInd2} the tensors $\underset{x_{t-1}d_{t-1}}{\bm{\mathscr{K}}}$, $\underset{\bm{\mathsf{O}}_{L_{t}}|x_{t-1}d_{t-1}}{\bm{\mathscr{F}}}$ and $\underset{\bm{\mathsf{O}}_{R_{t}}|x_{t\hspace{-1pt}-\hspace{-1pt}1}d_{t\hspace{-1pt}-\hspace{-1pt}1}}{\bm{ \mathscr{F}}}$ all have rank $n_xn_d$. Similarly, $\underset{o_to_{t+1}}{\bm{\mathscr{M}}}$ will have rank $n_x$ because in \eqref{eq:tensorOobs} the rank of the participating tensors $\underset{x_t}{\bm{\mathscr{K}}}$, $\underset{o_{t+1}|x_t}{\bm{\mathscr{F}}}$ and $\underset{o_{t}|x_t}{\bm{\mathscr{F}}}$ is $n_x$. In particular, note that the tensor $\underset{o_{t}|x_t}{\bm{\mathscr{F}}}$ is the observation matrix $O \in \mathbb{R}^{n_o\times n_x}$ of the model and it has rank $n_x$ according to assumption $A3$. This conclusion also justifies our choice for $\omega_{x_t} = o_{t}$ at the end of Section \ref{sec:obsTensRepresentation}.

The key unknowns now are the sets of the observed variables $\bm{\mathsf{O}}_{R}$ and $\bm{\mathsf{O}}_{L}$ that must be appropriately selected for the corresponding tensors to have rank $n_xn_d$. Recall that we defined $\bm{\mathsf{O}}_{R_{t-1}} = \{o_{t}, o_{t+1}, \ldots\}$. As one of the new key results of our work, we established that if we select the observations $o_t$ non-sequentially with gaps that grow exponentially with the state size $n_x$ then the following result holds for all $t$:
\begin{theorem}
\label{mainTheorem}
Let the number of observations be $|\bm{\mathsf{O}}_{R_{t-1}}| = \ell$ and define the set of indices \\$\mathscr{S} = \left\{\max\left[t, ~t\hspace{-1pt}+\hspace{-1pt}(n_d\hspace{-1pt}-\hspace{-1pt}1)\hspace{-1pt} -\hspace{-1pt} (n_x^i\hspace{-1pt}-\hspace{-1pt}1)\right] ~|~ i=0,\ldots, \ell-1\right\}$, such that $\bm{\mathsf{O}}_{R_{t-1}} = \{o_k | k \in \mathscr{S}\}$ then the rank of tensor  $\underset{\bm{\mathsf{O}}_{R_{{t-1}}}|x_{t-1}d_{t-2}}{\bm{\mathscr{F}}}$ is $\min[n_x^{\ell}, ~n_xn_d]$.
\end{theorem}

As a consequence of this result, to achieve the rank $n_xn_d$ we will require $\ell = \lceil 1 + \frac{\log n_d}{\log n_x}\rceil$ observations, since we need to ensure $n_x^\ell \geq n_xn_d$ and we want the minimal $\ell$ which satisfies this.  The span of the selected observations is $n_d$, while their number is only logarithmic in $n_d$. For example, consider the estimation of tensor $\underset{\bm{\mathsf{O}}_{L_{t-1}}\bm{\mathsf{O}}_{R_{t-1}}}{\bm{\mathscr{M}}}$ for an HSMM with $n_x=3$ and $n_d=20$. In this case $\ell = 4$ and $\bm{\mathsf{O}}_{R_{t-1}} = \{o_{t}, o_{t+11}, o_{t+17}, o_{t+19}\}$ and $\bm{\mathsf{O}}_{L_{t-1}} = \{o_{t-21}, o_{t-19}, o_{t-13}, o_{t-2}\}$, where the set $\bm{\mathsf{O}}_{L_{t-1}}$ is defined similar to $\bm{\mathsf{O}}_{R_{t-1}}$ in Theorem \ref{mainTheorem} but for the indices to the left of time stamp $t-1$.  Figure \ref{fig:obs} illustrates this example. We note that the requirement for the span of the selected observations to be $n_d$, which is a maximum state persistence, is to ensure that for a given time stamp $t$, we select the observations far enough to the right and left of it so that those observations are likely to be sampled from  different hidden states.

\begin{figure}[!t]
\centering
   \includegraphics[width=0.9\textwidth]{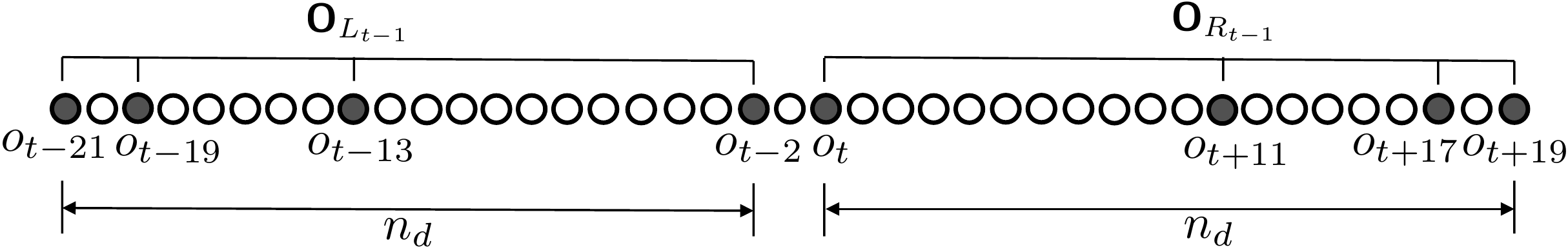}
   \caption{Observations required to estimate $\underset{\bm{\mathsf{O}}_{L_{t-1}}\bm{\mathsf{O}}_{R_{t-1}}}{\bm{\mathscr{M}}}$ from data for HSMM with $n_x=3$ and $n_d=20$.}
   \label{fig:obs}
\end{figure}

In order to prove the above Theorem, we will focus our analysis on the tensor $\underset{\bm{\mathsf{O}}_{R_{t+1}}|x_td_{t}}{\bm{\mathscr{F}}}$ instead of the tensor $\underset{\bm{\mathsf{O}}_{R_{{t-1}}}|x_{t-1}d_{t-2}}{\bm{\mathscr{F}}}$. This specific choice was only done to ensure the compactness in our notations, however the HSMM homogeneity property enables us to transfer this result for tensors for any $t$. Note that 
\begin{align*}
\underset{\bm{\mathsf{O}}_{R_{t+1}}|x_td_t}{\bm{\mathscr{F}}} = 
\underset{\bm{\mathsf{O}}_{R_{t-1}}|x_{t\hspace{-1pt}-\hspace{-1pt}2}d_{t\hspace{-1pt}-\hspace{-1pt}2}}{\bm{\mathscr{F}}} = 
 \underset{\bm{\mathsf{O}}_{R_{t-1}}|x_{t\hspace{-1pt}-\hspace{-1pt}1}d_{t\hspace{-1pt}-\hspace{-1pt}2}}{\bm{\mathscr{F}}} \times_{x_{t\hspace{-1pt}-\hspace{-1pt}1}d_{t\hspace{-1pt}-\hspace{-1pt}2}} \underset{x_{t\hspace{-1pt}-\hspace{-1pt}1}d_{t\hspace{-1pt}-\hspace{-1pt}2}|x_{t\hspace{-1pt}-\hspace{-1pt}2}d_{t\hspace{-1pt}-\hspace{-1pt}2}}{\bm{\mathscr{X}}},
\end{align*}
\noindent where the first equality is due to the homogeneity property of the model and in the second equality we embedded the HSMM transition matrix into tensor $\underset{x_{t\hspace{-1pt}-\hspace{-1pt}1}d_{t\hspace{-1pt}-\hspace{-1pt}2}|x_{t\hspace{-1pt}-\hspace{-1pt}2}d_{t\hspace{-1pt}-\hspace{-1pt}2}}{\bm{\mathscr{X}}}$ with mode $d_{t-2}$ duplicated. It can be shown that the matricized tensor $\underset{x_{t\hspace{-1pt}-\hspace{-1pt}1}d_{t\hspace{-1pt}-\hspace{-1pt}2}|x_{t\hspace{-1pt}-\hspace{-1pt}2}d_{t\hspace{-1pt}-\hspace{-1pt}2}}{{\mathbf{X}}} \in \mathbb{R}^{n_xn_d\times n_xn_d}$ has rank $n_xn_d$, i.e., it is full rank. Therefore, the rank structure of $\underset{\bm{\mathsf{O}}_{R_{t+1}}|x_td_t}{\bm{\mathscr{F}}}$ determines the rank structure of  $\underset{\bm{\mathsf{O}}_{R_{t-1}}|x_{t\hspace{-1pt}-\hspace{-1pt}1}d_{t\hspace{-1pt}-\hspace{-1pt}2}}{\bm{\mathscr{F}}}$.

The rest of Section \ref{sec:observationVars} is devoted to the proof of Theorem \ref{mainTheorem}. We first establish the rank structure of tensor $\underset{\bm{\mathsf{O}}_{R_{t+1}}|x_{t}d_{t}}{\bm{\mathscr{F}}}$ for sequential set of observations $\bm{\mathsf{O}}_{R_{t+1}}$ and then analyze the rank structure for the observations which were selected non-sequentially.

\subsection[Rank Structure of Tensor F]{Rank Structure of Tensor $\underset{\bm{\mathsf{O}}_{R_{t+1}}|x_{t}d_{t}}{\bm{\mathscr{F}}}$}
\label{sec:factorRepres}

Define by $\bm{\mathsf{X}}_{R_{t+1}} = \{x_{t+2}, x_{t+3}, \ldots\}$, the sequence of hidden states corresponding to observations $\bm{\mathsf{O}}_{R_{t+1}} = \{o_{t+2}, o_{t+3}, \ldots\}$. Then using conditional independence property of the graphical model in Figure \ref{fig:hsmm}, namely, that the variables $\bm{\mathsf{O}}_{R_{t+1}}$ and $x_td_{t}$ are independent given $\bm{\mathsf{X}}_{R_{t+1}}$, we can write:
\begin{align}
\label{eq:factorF}
\underset{\bm{\mathsf{O}}_{R_{t+1}}|x_td_{t}}{\bm{\mathscr{F}}} = \underset{\bm{\mathsf{O}}_{R_{t+1}}|\bm{\mathsf{X}}_{R_{t+1}}}{\bm{\mathscr{Q}}}\times\underset{\bm{\mathsf{X}}_{R_{t+1}}|x_td_t}{\bm{\mathscr{T}}},
\end{align}
\noindent for some tensors $\bm{\mathscr{Q}}$ and $\bm{\mathscr{T}}$, representing the appropriate probability distributions.

Denoting $\ell = |\bm{\mathsf{O}}_{R_{t+1}}| = |\bm{\mathsf{X}}_{R_{t+1}}|$, it can be verified, that the matrisized form of $\bm{\mathscr{Q}}$ in \eqref{eq:factorF} can be written as $\mathbf{Q} = \otimes_{\ell}O \in \mathbb{R}^{n_o^\ell\times n_x^\ell}$, i.e., a Kronecker product of the observation matrix $O$ with itself $\ell$ times. According to the assumption $A3$, $rank(O) = n_x$ and $n_x \leq n_o$, and using the rank property of the Kronecker product, we infer that $rank(\mathbf{Q}) = n_x^\ell$.

%Recall that $\underset{r|d_{t}x_t}{\bm{\mathscr{F}}}$ is the tensor representation of the conditional probability table of $p(r|d_{t}x_t)$, where $w$ represents a set of observable variables to the right of $d_{t}x_t$.

Combining the above conclusion with the fact that the matrisized form of the other two tensors in \eqref{eq:factorF} is $\mathbf{F} \in \mathbb{R}^{n_o^{\ell} \times n_xn_d}$ and $\mathbf{T} \in \mathbb{R}^{n_x^{\ell} \times n_xn_d}$, to ensure invertibility of $\bm{\mathscr{F}}$, we need to select a set of variables $\bm{\mathsf{X}}_{R_{t+1}}$ so that $rank\Big(\underset{\bm{\mathsf{X}}_{R_{t+1}}|x_td_t}{\mathbf{T}}\Big) = n_xn_d$ with the condition that $n_x^{\ell}\geq n_xn_d$.
Thus, the problem of the analysis of the rank structure of tensor $\underset{\bm{\mathsf{O}}_{R_{t+1}}|x_{t}d_{t}}{\bm{\mathscr{F}}}$ translates to the problem of rank structure of matrix $\underset{\bm{\mathsf{X}}_{R_{t+1}}|x_td_t}{\mathbf{T}}$. In what follows, we assume that $\bm{\mathsf{X}}_{R_{t+1}} = \{x_{t+2},\ldots,x_{t+\ell+1}\}$ are sequential and so we would be interested in determining $\ell$ which makes $rank\Big(\underset{\bm{\mathsf{X}}_{R_{t+1}}|x_td_t}{\mathbf{T}}\Big) = n_xn_d$. Later, the sequential assumption will be removed and we show how to select such variables in a more efficient way.

\subsubsection{Computation of Factor T}
In order to study the rank structure of $\underset{\bm{\mathsf{X}}_{R_{t+1}}|x_td_t}{\mathbf{T}}$ we will have to understand the mechanism how this matrix is constructed and how the rank changes as the size of $\bm{\mathsf{X}}_{R_{t+1}}$ increases.
We start by considering the following conditional independence relationships from the model in Figure \ref{fig:hsmm}:
\begin{align}
\label{example}
p(x_{t+3}, x_{t+2}|x_{t+1},d_{t+1}) &=\sum_{d_{t+2}}p(x_{t+3}|x_{t+2},d_{t+2})\underbracket{p(d_{t+2}|x_{t+2},d_{t+1})p(x_{t+2}|x_{t+1},d_{t+1})}\\
\label{example2}
p(x_{t+3}, x_{t+2}, x_{t+1}|x_t,d_t) &=\sum_{d_{t+1}}p(x_{t+3}, x_{t+2}|x_{t+1},d_{t+1})\underbracket{p(d_{t+1}|x_{t+1},d_t)p(x_{t+1}|x_{t},d_{t})}.
\end{align}

\noindent Using the model's homogeneity property, we see that the quantity underlined in \eqref{example} is the same as the one in \eqref{example2}. Moreover, equation \eqref{example} can then be thought of as transforming $p(x_{t+1}|x_{t},d_{t})$ into $p(x_{t+2}, x_{t+1}|x_{t},d_{t})$, while the expression in \eqref{example2} is, in effect, transforms $p(x_{t+2}, x_{t+1}|x_{t},d_{t})$ into $p(x_{t+3}, x_{t+2}, x_{t+1}|x_{t},d_{t})$. Thus \eqref{example} and \eqref{example2} encode the following chain of transformations:
\begin{align*}
p(x_{t+1}|x_{t},d_{t}) \rightarrow p(x_{t+2}, x_{t+1}|x_{t},d_{t}) \rightarrow p(x_{t+3}, x_{t+2}, x_{t+1}|x_{t},d_{t}).
\end{align*}

Based on the above considerations, we can rewrite \eqref{example} and \eqref{example2} in the tensor form as follows:
\begin{align}
\label{Texampl}
\underset{x_{t+3}, x_{t+2}|x_{t+1},d_{t+1}}{\bm{\mathscr{T}}} &= \underset{x_{t+3},x_{t+2}|x_{t+2},d_{t+2}}{\bm{\mathscr{T}}} \times_{x_{t+2}d_{t+2}} \underset{x_{t+2},d_{t+2}|x_{t+1}d_{t+1}}{{\bm{\mathscr{V}}}}\\
\label{Texampl2}
\underset{x_{t+3}, x_{t+2}, x_{t+1}|x_{t},d_{t}}{\bm{\mathscr{T}}} &= \underset{x_{t+3}, x_{t+2},x_{t+1}|x_{t+1},d_{t+1}}{\bm{\mathscr{T}}} \times_{x_{t+1}d_{d+1}} \underset{x_{t+1},d_{t+1}|x_td_t}{\bm{\mathscr{V}}},
\end{align}
\noindent where $\underset{x_{t+2},d_{t+2}|x_{t+1},d_{t+1}}{\bm{\mathscr{V}}} = \underset{x_{t+1},d_{t+1}|x_t,d_t}{\bm{\mathscr{V}}} = \underset{x_{t+1},d_{t+1}|x_{t+1},d_{t}}{\bm{\mathscr{D}}}\times_{x_{t+1}d_t} \underset{x_{t+1}, d_t|x_{t},d_{t}}{\bm{\mathscr{X}}}$. The homogeneity property allows us to rewrite the above as
\begin{align}
\label{Texample}
\underset{x_{t+2}, x_{t+1}|x_{t},d_{t}}{\bm{\mathscr{T}}} &= \underset{x_{t+1},x_{t}|x_{t},d_{t}}{\bm{\mathscr{T}}} \times {\bm{\mathscr{V}}}\\
\label{Texample2}
\underset{x_{t+3}, x_{t+2}, x_{t+1}, x_{t+1}|x_{t},d_{t}}{\bm{\mathscr{T}}} &= \underset{x_{t+2}, x_{t+1}|x_{t},d_{t}}{\bm{\mathscr{T}}} \times \bm{\mathscr{V}}.
\end{align}

Our next step is to represent the above tensor equations in the matrix form. First, consider tensor $\bm{\mathscr{V}}$, its matricized form can be written as:
\begin{align}
\label{eq:V}
\mathbf{V} = \underset{x_{t+1},d_{t+1}|x_{t+1}, d_{t}}{\mathbf{D}}~~\underset{x_{t+1},d_{t}|x_{t}, d_{t}}{\mathbf{X}}
\end{align}
\noindent where $\underset{x_{t+1},d_{t+1}|x_{t+1}, d_{t}}{\mathbf{D}} \in \mathbb{R}^{n_xn_d\times n_xn_d}$ and  $\underset{x_{t+1},d_{t}|x_{t}, d_{t}}{\mathbf{X}} \in \mathbb{R}^{n_xn_d\times n_xn_d}$. Next, consider the equations \eqref{Texample} and \eqref{Texample2}, its matrix version is of the form:
\begin{align}
\label{Mexample}
\underset{x_{t+2}, x_{t+1}|x_{t},d_{t}}{{\mathbf{T}}} &= \underset{x_{t+1}, x_t |x_{t},d_{t}}{{\mathbf{T}}} ~~{\mathbf{V}}\\
\label{Mexample2}
\underset{x_{t+3}, x_{t+2}, x_{t+1}|x_{t},d_{t}}{{\mathbf{T}}} &= \underset{x_{t+2}, x_{t+1}, x_t |x_{t},d_{t}}{{\mathbf{T}}} ~~{\mathbf{V}},
\end{align}
\noindent here $\underset{x_{t+1}, x_t |x_{t},d_{t}}{{\mathbf{T}}} \in \mathbb{R}^{n_x^2 \times n_xn_d}$, $\underset{x_{t+2}, x_{t+1}|x_{t},d_{t}}{{\mathbf{T}}} \in \mathbb{R}^{n_x^2\times n_xn_d}$,  and similarly $\underset{x_{t+2},x_{t+1}, x_t |x_{t},d_{t}}{{\mathbf{T}}} \in \mathbb{R}^{n_x^3 \times n_xn_d}$, and matrix  $\underset{x_{t+3}, x_{t+2}, x_t|x_{t},d_{t}}{{\mathbf{T}}} \in \mathbb{R}^{n_x^3\times n_xn_d}$.

Summarizing the above derivations, we can describe the following algorithmic approach for analyzing $\underset{\bm{\mathsf{X}}_{R_{t+1}}|x_td_t}{\mathbf{T}}$ as $\bm{\mathsf{X}}_{R_{t+1}}$ increases. We begin with $\underset{x_{t+1}|x_t, d_t}{\mathbf{T}} = \left[\mathcal{X}~\mathbf{I}~\cdots~\mathbf{I}\right] \in \mathbb{R}^{n_x\times n_xn_d}$, where the first block $\mathcal{X}\in \mathbb{R}^{n_x\times n_x}$ corresponds to $d_t=1$, and the subsequent $(n_d-1)$ blocks of $\mathbf{I}\in \mathbb{R}^{n_x\times n_x}$ correspond to $d_t > 1$ for which $x_{t+1}=x_t$. We then use \eqref{Mexample} to get $\underset{x_{t+2}, x_{t+1}|x_{t},d_{t}}{{\mathbf{T}}}$. However, notice that in \eqref{Mexample} the matrix $\underset{x_{t+1}, x_t |x_{t},d_{t}}{{\mathbf{T}}}$ has a duplicated mode $x_t$, therefore, we need to restructure $\underset{x_{t+1}|x_t, d_t}{\mathbf{T}}$, which can be accomplished with:
\begin{align*}
\underset{x_{t+1},x_{t}|x_t, d_t}{\mathbf{T}^{\prime}} = \underset{x_{t+1}|x_t, d_t}{\mathbf{T}} \odot ~~ \mathbf{E},
\end{align*} 
\noindent where $\mathbf{E} = [\mathbf{I}~\cdots ~\mathbf{I}] \in \mathbb{R}^{n_x \times n_x n_d}$ and $\odot$ denotes a Khatri-Rao product (row-wise Kronecker product)\footnote{
Let $\mathbf{P} = 
\begin{bmatrix}
\mathbf{p}_1\\
\mathbf{p}_2\\
\vdots\\
\mathbf{p}_n
\end{bmatrix} \in \mathbb{R}^{m\times n}$ and $\mathbf{Q} \in \mathbb{R}^{k\times n}$ then $\mathbf{P}\odot\mathbf{Q}=
\begin{bmatrix}
\mathbf{p}_1\otimes\mathbf{Q}\\
\mathbf{p}_2\otimes\mathbf{Q}\\
\vdots\\
\mathbf{p}_n\otimes\mathbf{Q}
\end{bmatrix} \in \mathbb{R}^{mk\times n}$, where $\otimes$ is a Kronecker product.
}. 
Then, we use \eqref{Mexample2} to transform $\underset{x_{t+2}, x_{t+1}|x_{t},d_{t}}{{\mathbf{T}}}$ into $\underset{x_{t+3}, x_{t+2}, x_{t+1}|x_{t},d_{t}}{{\mathbf{T}}}$ where, again a preliminary step is to restructure the matrix as follows:
\begin{align*}
\underset{x_{t+2},x_{t+1},x_{t}|x_t, d_t}{\mathbf{T}^{\prime}} = \underset{x_{t+2}, x_{t+1}|x_t, d_t}{\mathbf{T}} \odot ~~ \mathbf{E}.
\end{align*} 

\noindent Algorithm \ref{alg1} summarizes the above constructions for a general case.

\begin{algorithm}[tb]
   \caption{Computation of $\underset{\bm{\mathsf{X}}_{R_{t+1}}|x_td_t}{\mathbf{T}}$}
   \label{alg1}
\begin{algorithmic}
   \STATE {\bfseries Input:} $p(d_t|x_t,d_{t-1})$ and $p(x_t|x_{t-1},d_{t-1})$ - duration and transition distributions, $\ell$ - the number of sequential hidden states represented by $\bm{\mathsf{X}}_{R_{t+1}}$. 
   \STATE {\bfseries Initialization:}
   \begin{align*}
    & p(x_{t+1}|x_{t},d_{t}) \rightarrow \underset{x_{t+1}|x_t, d_t}{\mathbf{T}}\\
    & p(d_{t+1}|x_{t+1},d_{t}) \rightarrow \underset{x_{t+1},d_{t+1}|x_{t+1}, d_{t}}{\mathbf{D}}\\
    & p(x_{t+1}|x_{t},d_{t})\rightarrow \underset{x_{t+1},d_{t}|x_{t}, d_{t}}{\mathbf{X}}\\
    &\mathbf{V} = \underset{x_{t+1},d_{t+1}|x_{t+1}, d_{t}}{\mathbf{D}}~~\underset{x_{t+1},d_{t}|x_{t}, d_{t}}{\mathbf{X}}, \quad\mathbf{E} = [\mathbf{I}\cdots \mathbf{I}]
   \end{align*}
   \STATE {\bfseries for }$i=1$ {\bfseries to} $\ell-1$ {\bfseries do}
   \begin{align}
   \label{eq:expan}
   \hspace{-5pt}\underset{x_{t\hspace{-1pt}+\hspace{-1pt}i},~\ldots~,x_{t\hspace{-1pt}+\hspace{-1pt}1},x_t|x_t, d_t}{\mathbf{T}^{\prime}} &= \underset{x_{t\hspace{-1pt}+\hspace{-1pt}i},~\ldots~,x_{t\hspace{-1pt}+\hspace{-1pt}1}|x_t, d_t}{\mathbf{T}} \odot ~\mathbf{E}\\
   	\label{eq:prop}
   	 \hspace{-5pt}\underset{x_{t\hspace{-1pt}+\hspace{-1pt}i\hspace{-1pt}+\hspace{-1pt}1},~\ldots~,x_{t\hspace{-1pt}+\hspace{-1pt}2},x_{t\hspace{-1pt}+\hspace{-1pt}1}|x_t, d_t}{\mathbf{T}} &= \underset{x_{t\hspace{-1pt}+\hspace{-1pt}i},~\ldots~,x_{t\hspace{-1pt}+\hspace{-1pt}1}, x_t|x_t, d_t}{\mathbf{T}^{\prime}}~~\mathbf{V}
   \end{align}
   \STATE {\bfseries end for}
\end{algorithmic}
\end{algorithm}
 
The following Theorem characterizes the rank structure of matrix $\underset{\bm{\mathsf{X}}_{R_{t+1}}|x_td_t}{\mathbf{T}}$ in the output of the Algorithm \ref{alg1}. The proof can be found in Appendix \ref{sec:AnalysisAlg1}.

\begin{theorem}
\label{mainTheoremm}
The rank of the output matrix $\underset{\bm{\mathsf{X}}_{R_{t+1}}|x_td_t}{\mathbf{T}}$ in Algorithm \ref{alg1} is $\min(\ell n_x, n_xn_d)$.
\end{theorem}

\noindent Applying now Theorem \ref{mainTheoremm} to equation \eqref{eq:factorF} in matrix form
\begin{align*}
\underset{\bm{\mathsf{O}}_{R_{t+1}}|x_td_{t}}{{\mathbf{F}}} = \underset{\bm{\mathsf{O}}_{R_{t+1}}|\bm{\mathsf{X}}_{R_{t+1}}}{{\mathbf{Q}}}\times\underset{\bm{\mathsf{X}}_{R_{t+1}}|x_td_t}{{\mathbf{T}}},
\end{align*}
where $rank(\mathbf{Q}) = n_x^\ell $ we can now conclude the following result:

\begin{corollary}
\label{mainCorollary}
To achieve the full column rank for $\underset{\bm{\mathsf{O}}_{R_{t+1}}|x_td_t}{\mathbf{F}} \in \mathbb{R}^{n_o^\ell\times n_xn_d}$, i.e. to ensure that the rank of tensor $\underset{\bm{\mathsf{O}}_{R_{t+1}}|x_td_t}{\bm{\mathscr{F}}}$ is $n_xn_d$, the number of observations $\ell $ in $\bm{\mathsf{O}}_{R_{t+1}} = \{o_{t+2}, o_{t+3}, \ldots, o_{t+\ell+1}\}$ must be equal to the maximum state persistence i.e., $\ell =n_d$.
\end{corollary}

%\subsubsection{Analysis of Algorithm \ref{alg1}}
%\label{sec:algAnalysis}
%\input{secAlgAnalysis}

\subsubsection[Efficient Computation of Factor T]{Efficient Computation of Factor $\mathbf{T}$}
\label{sec:effFactorComp}

In Corollary \ref{mainCorollary} we established that the required number of observations in $\bm{\mathsf{O}}_{R_{t+1}} = \{o_{t+2}, o_{t+3}, \ldots, o_{t+\ell+1}\}$ is $\ell=n_d$. Therefore, the sizes of the estimated quantities $\tilde{\bm{\mathscr{D}}} \in \mathbb{R}^{n_o^{n_d} \times n_o^{n_d}}$ and $\tilde{\bm{\mathscr{X}}} \in \mathbb{R}^{n_o^{n_d} \times n_o^{n_d} \times n_o}$ in Algorithm \ref{alg1} will have exponential dependency on $n_d$. When maximum state persistence is large, the estimation of such quantity becomes impractical. Fortunately, we can modify Algorithm \ref{alg1} to significantly reduce the number of observations. The idea is to apply the step \eqref{eq:prop} multiple times in-between the applications of step \eqref{eq:expan}. Recall that in the previous construction we established that $\ell=n_d$ consecutive observations are sufficient, e.g., $\bm{\mathsf{O}}_{R_{t+1}} = \{o_{t+2},\ldots,o_{t+\ell+1}\}$. In contrast, in the proposed approach, every time we add an observation, say $o_{t+\tau}$, we skip certain number $\delta$ of time steps before adding another observation $o_{t+\tau+\delta}$, so that the observations are non-consecutive. As we illustrate next, the span of these non-consecutive observations is still $n_d$ but the number of them is only logarithmic in $n_d$. The proposed approach still achieves the full rank structure of $\underset{\bm{\bm{\mathsf{O}}}_{R_{t+1}}|x_td_t}{\mathbf{F}}$ but with smaller number of data points. Algorithm \ref{alg2}, which is a simple modification of Algorithm \ref{alg1}, summarizes the above procedure.

\begin{algorithm}[!t]
    \caption{Efficient computation of $\underset{\bm{\mathsf{X}}_{R_{t+1}}|x_td_t}{\mathbf{T}}$}
    \label{alg2}
 \begin{algorithmic}
    \STATE {\bfseries Input:} $p(d_t|x_t,d_{t-1})$ and $p(x_t|x_{t-1},d_{t-1})$ - duration and transition distributions, $\ell$ - the number of sequential hidden states represented by $\bm{\mathsf{X}}_{R_{t+1}}$
       \STATE {\bfseries Initialization:}
       \begin{align*}
        & p(x_{t+1}|x_{t},d_{t}) \rightarrow \underset{x_{t+1}|x_t, d_t}{\mathbf{T}}\\
        & p(d_{t+1}|x_{t+1},d_{t}) \rightarrow \underset{x_{t+1},d_{t+1}|x_{t+1}, d_{t}}{\mathbf{D}}\\
        & p(x_{t+1}|x_{t},d_{t})\rightarrow \underset{x_{t+1},d_{t}|x_{t}, d_{t}}{\mathbf{X}}\\
        &\mathbf{V} = \underset{x_{t+1},d_{t+1}|x_{t+1}, d_{t}}{\mathbf{D}}~~\underset{x_{t+1},d_{t}|x_{t}, d_{t}}{\mathbf{X}}, \quad\mathbf{E} = [\mathbf{I}\cdots \mathbf{I}]
       \end{align*}
    \STATE $c = 1$
   \STATE {\bfseries for }$i=1$ {\bfseries to} $\ell-1$ {\bfseries do}
    \begin{align}
    	\label{eq:prop2}
       	 \hspace{-5pt}{\mathbf{T}} = {\mathbf{T}}~~\mathbf{V}
    \end{align}
    \STATE {\bfseries \quad\quad if }$i==(n_x)^{c}-1$ {\bfseries or} $i==\ell-1$ {\bfseries do}
    \begin{align}
    	   \label{eq:expan2}
           \hspace{-5pt}{\mathbf{T}} = {\mathbf{T}} \odot ~\mathbf{E}
    \end{align}
    \STATE {\bfseries \quad\quad end if }
    \STATE \quad\quad $c = c+1$
    \STATE {\bfseries end for }
 \end{algorithmic}
\end{algorithm}
 
\noindent The following result establishes the rank structure of the matrix $\underset{\bm{\mathsf{X}}_{R_{t+1}}|x_td_t}{\mathbf{T}}$ in the output of the Algorithm \ref{alg2}. The proof can be found in Appendix \ref{sec:AnalysisAlg2}.

\begin{theorem}
\label{mainTheorem2}
The rank of the output matrix $\underset{\bm{\mathsf{X}}_{R_{t+1}}|x_td_t}{\mathbf{T}}$ in Algorithm \ref{alg2} is $\min(n_x^{\ell}, n_xn_d)$.
\end{theorem}

\noindent  Note that based on the above theorem, Algorithm \ref{alg2} increases the rank at every step exponentially rather than linearly. In order for $\underset{\bm{\mathsf{X}}_{R_{t+1}}|x_td_t}{\mathbf{T}}$ to achieve the rank $n_xn_d$ we will now require $\ell = \lceil 1 + \frac{\log n_d}{\log n_x}\rceil$ observations, since we need to ensure $n_x^\ell = n_xn_d$.  Observe that the span of the selected observations is still $n_d$, while the number of the observations is only logarithmic in $n_d$. The following Corollary summarizes the above conclusions.

\begin{corollary}
\label{mainCorollary2}
To achieve the full column rank for $\underset{\bm{\mathsf{O}}_{R_{t+1}}|x_td_t}{\mathbf{F}} \in \mathbb{R}^{n_o^\ell\times n_xn_d}$, i.e. to ensure that the rank of tensor $\underset{\bm{\mathsf{O}}_{R_{t+1}}|x_td_t}{\bm{\mathscr{F}}}$ is $n_xn_d$, the number of observations $\ell $ in $\bm{\mathsf{O}}_{R_{t+1}}$ must be equal to $\ell = \lceil 1 + \frac{\log n_d}{\log n_x}\rceil$, since we need to ensure $n_x^\ell = n_xn_d$.
\end{corollary}

\noindent Theorem \ref{mainTheorem2} together with Corollary \ref{mainCorollary2} now proves the Theorem \ref{mainTheorem} stated earlier.

%% file: Experiments.tex
In this section we evaluated the performance of the proposed algorithm both on synthetic as well as real datasets and compared its performance to a standard EM algorithm.
%The spectral algorithm was implemented in Matlab, while for EM we used an appropriately adopted version of the C++ library for Discrete Approximate Inference in graphical models \cite{libdai}.

\subsection{Synthetic Data}

\begin{figure*}[!tb]
\centering
\includegraphics[width=\textwidth]{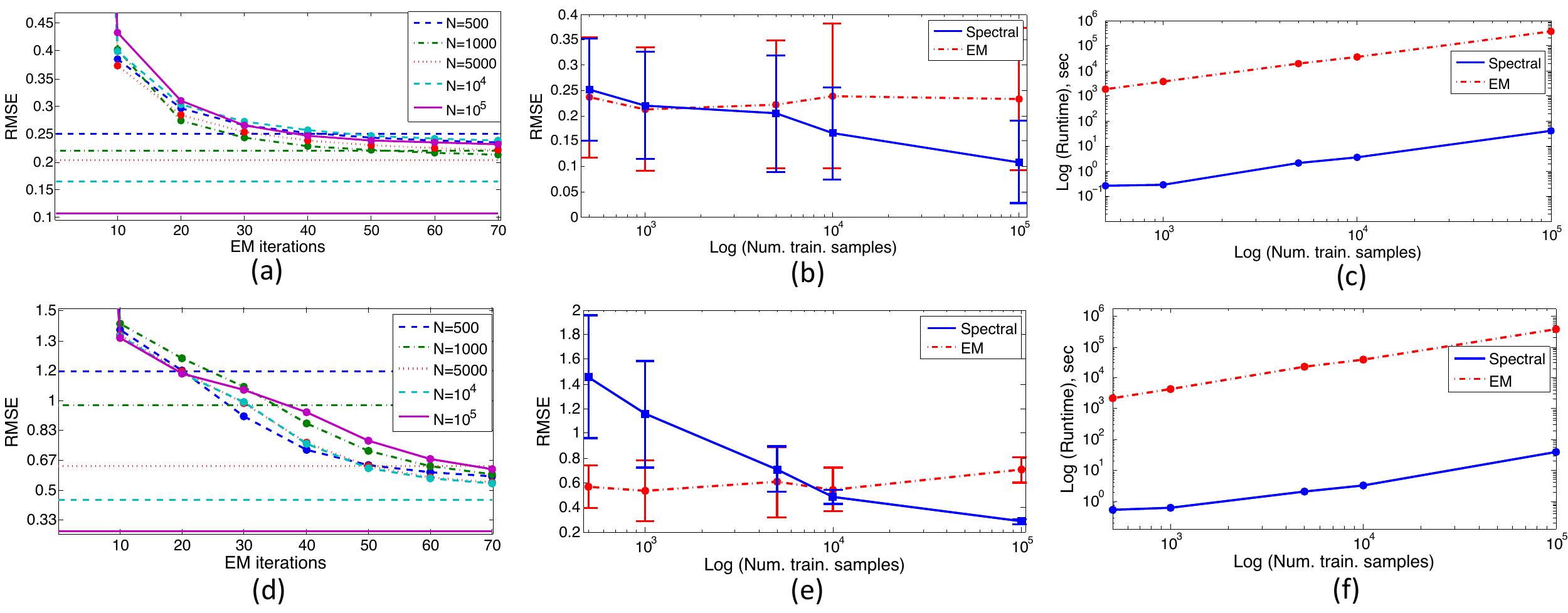}
%\vspace*{-5mm}
\caption{Performance of the spectral algorithm and EM on synthetic data generated from HSMM with $n_o=3,n_x=2,n_d=2$ (top row) and $n_o=5,n_x=4,n_d=6$ (bottom row). (a), (d): Error for EM across different iterations for various training datasets. The straight lines show the performance for spectral method. (b), (e): Average error and one standard deviation over $100$ runs for EM after convergence and spectral algorithm across different number of training data. (c), (f): Runtime, in seconds, for both methods.}
%\vspace*{-0mm}
\label{fig:sim}
\end{figure*}

Using synthetic data, we compared the estimation accuracy and the runtime of the spectral algorithm with EM. For this, we defined two HSMMs, one with $n_o=3, n_x=2, n_d=2$ and another with  $n_o=5, n_x=4, n_d=6$. For each model, we generated a set of $N_{train} = \{500,1000,5000, 10^4, 10^5\}$ training and $N_{test} = \hspace{-2pt}=\hspace{-2pt}1000$ testing sequences, each of length $T=100$. The accuracy of estimating likelihood for each testing sequence was measured using the relative deviation from the true likelihood, i.e., $\epsilon_i = \frac{|\hat{p}(\mathbf{S}_i^{test}) - p(\mathbf{S}_i^{test})|}{p(\mathbf{S}_i^{test})}$ for $i=1,\ldots, 1000$. Given $1000$ such values, we then computed the final score, which  is the root-mean-square error (RMSE) across all the testing sequences, RMSE $ = \sqrt{\frac{1}{N_{test}}\sum_{i=1}^{N_{test}}\epsilon_i^2}$. 

Figure \ref{fig:sim} shows results, where the top row of graphs corresponds to the model $n_o=3,n_x=2,n_d=2$ and the bottom row is for model $n_o=5, n_x=4, n_d=6$. The left column of graphs shows the progression of RMSE across EM iterations for both models; the middle column shows the dependence of testing error on the number of training samples and the right column shows the running times. It can be observed from plots (b) and (e) in Figure \ref{fig:sim} that with the small training set, EM achieves smaller errors, while as the number of training samples increases, the spectral method becomes more accurate, outperforming EM. Also, comparing the plots (a), (b) with (d) and (e), we can conclude that for larger models, i.e., whose $n_o$, $n_x$ and $n_d$ are larger, the spectral method requires more data in order to achieve same or better accuracy than EM. This is expected since the sizes of estimated tensors grow with the model size.
%Moreover, EM achieves lower errors with smaller number of training sets \ab{Not clear - do we mean smaller training set?}. Since it has no global optimum guarantees, its accuracy does not improve much with more data as it gets trapped in the local minima \ab{How do we know that this is a local minima issue}. \ab{Providing alternatives to the previous two lines:} 
%Moreover, EM achieves lower errors, compared to spectral, for small training sets. Interestingly, the performance of EM does not seem to improve with increase in data size, whereas spectral improves steadily and outperforms EM on larger training sets.  
Moreover, the plots (c) and (f) in Figure \ref{fig:sim} show that spectral method is several orders of magnitude faster than EM.

Given the above results, we can conclude that (i) for small datasets EM is a preferable algorithm, (ii) for large data, the spectral algorithm is a better choice, since it achieves higher accuracy and (iii) across all datasets the spectral algorithm requires significantly less computations as compared to EM.

\subsection{Application to Aviation Safety Data}

\begin{figure*}[!tb] %[!th]
\centering
\includegraphics[width=\textwidth]{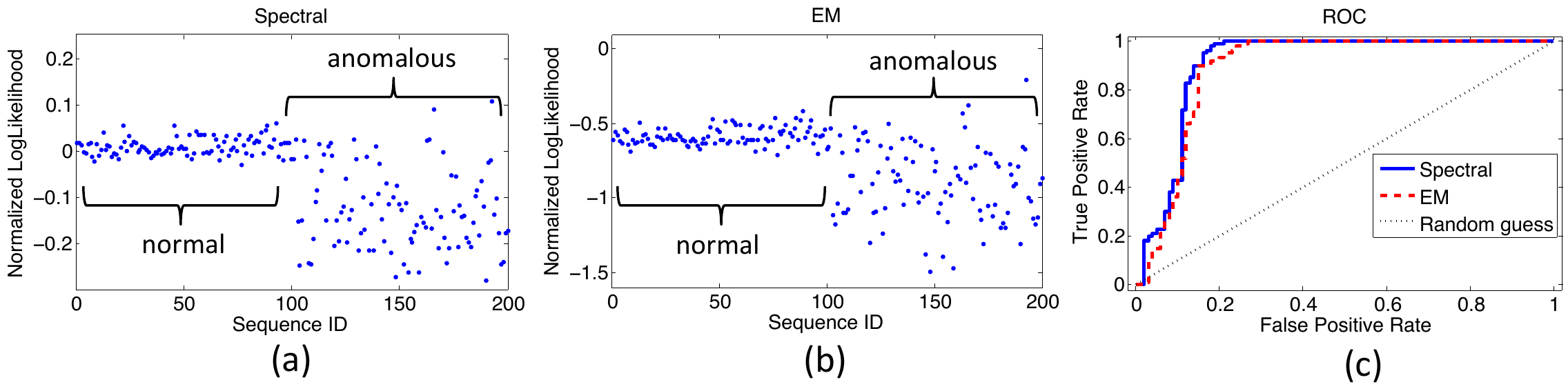}
\caption{Evaluation of the spectral algorithm and EM on aviation safety data. (a) and (b): Normalized joint loglikelihood computed by spectral algorithm (a) and EM (b) for a set of $200$ test flights, with $100$ normal and $100$ anomalous. HSMM parameters: $n_o=9, n_x=8, n_d=40$ (c): The Receiver Operating Characteristic (ROC) curve, illustrating classification accuracy of the algorithms. Area Under Curve (AUC) for spectral algorithm is 0.91 and for EM is 0.89.}
%\vspace*{-5mm}
\label{fig:real}
\end{figure*}

We also compared the performance of the spectral algorithm and EM on real NASA flight dataset \cite{nasadata}, containing over $180000$ flights of $35$ aircrafts from a defunct mid-western airline company. For each flight, the data has a record of $186$ parameters, sampled at $1$ Hz, including sensor readings and pilot actions. We considered a problem of anomaly detection in aviation systems \cite{bsoo09, gomm12, matthews13} and used HSMM to detect abnormal flights based on pilot actions. Our idea is based on the observation that a flight can be partitioned into a number of phases, e.g., initial descent, touch down, or braking on the runway, and where within each phase the pilot performs certain actions. For example, during the initial descent, the pilot reduces throttle, lowers the flaps, and uses the ailerons and elevator to stabilize the aircraft. On the other hand, in the braking stage, the pilot uses brakes as well as rudder to keep the aircraft in the middle of the runway. Using HSMM as a model, we represented the flight phases as hidden states and the pilot actions as the observations from these states (see \cite{melnyk13} for more details).

In our experiments, we focused on a part of flight related to the approach phase ($15-60$ minutes in duration before the touch down on the runway) for a subset of flights landing at the same airport. We chose $9$ pilot commands, among which are ``selected altitude'', ``selected heading'', ''selected throttle level'',  etc. A simple data filter, based on the histogram of the pilot actions, was applied to select $10020$ normal flights for training. A test set contained $200$ flights, with $100$ of them being similar to the training set and the rest were selected from the flights rejected by the filter. Most of abnormal flights contained low occurrence events, such as fast descent, unusual usage of air brakes, etc., and few significant anomalies, e.g., the aborted descent in order to delay the flight. The length of the considered sequences varied anywhere from $500$ to $4000$ seconds.

We applied EM and spectral algorithm to compute the normalized joint log-likelihood
\begin{align*}
\frac{1}{T_i}\log p(o_1,o_2,\ldots, o_{T_i}),
\end{align*}
for $i=1,\ldots,200$, where $o_i$ are the observed pilot actions. Figure \ref{fig:real} shows the results. The high-likelihood sequences were considered normal and low-likelihood ones classified as anomalous (see plots (a) and (b)). Both algorithms achieved similar detection accuracy, with the spectral algorithm having the Area Under Curve (AUC) score of $0.91$ and the EM had AUC $=0.89$. On the other hand, the computational time of the spectral algorithm was orders of magnitude smaller as compared to EM (see plot (c) on Figure \ref{fig:real}). We also compared performance of both algorithm on the same flight data while varying the dimensionality of the HSMM parameters (see Figure \ref{fig:real2} and Table \ref{tb:table}). We can see that although the performance of EM and spectral algorithm is similar across many models, the latter offers significant computational savings.

\begin{figure*}[!tb] %[!th]
\centering
\includegraphics[width=0.7\textwidth]{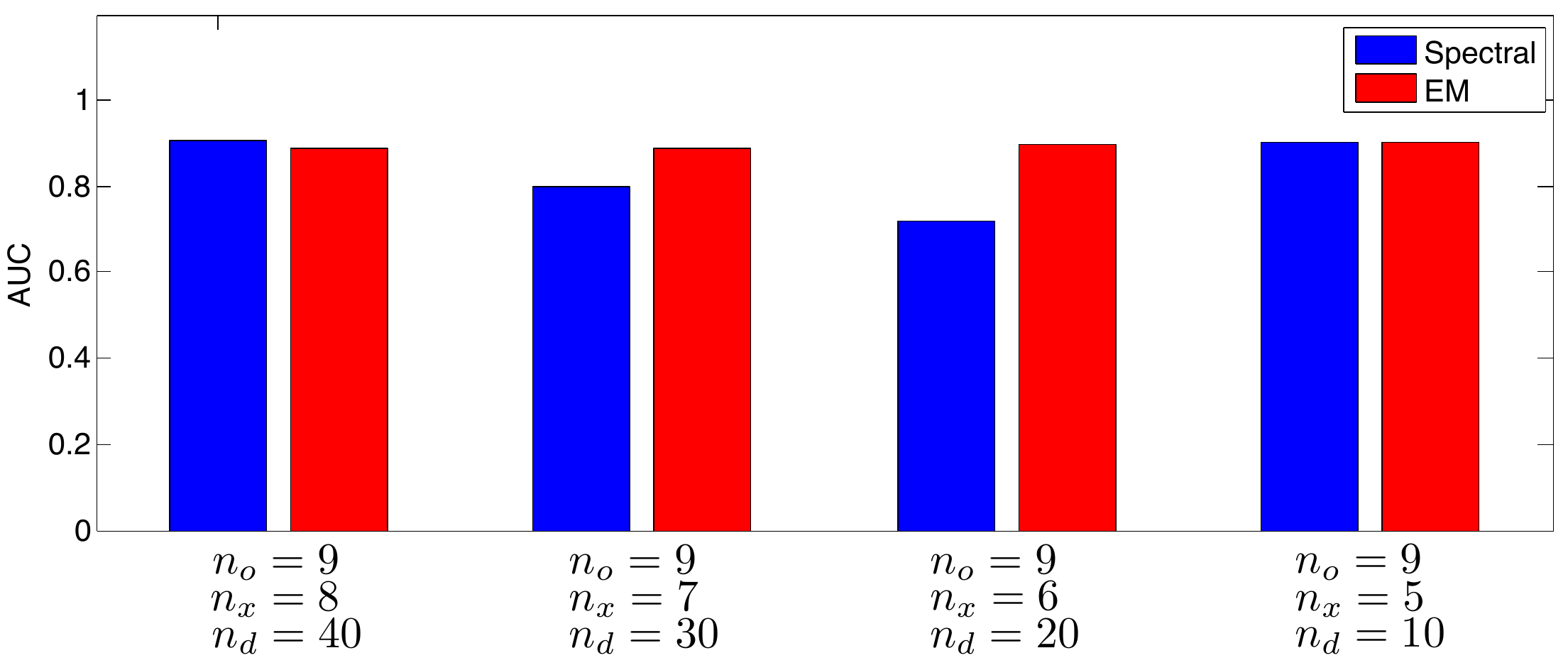}
\caption{Comparison of AUC scores for EM and spectral algorithm for various model parameters when evaluated on aviation safety data. Both algorithms achieve similar high accuracy across different models.}
%\vspace*{-5mm}
\label{fig:real2}
\end{figure*}

\begin{table*} %[t]
\centering % used for centering table
\begin{tabular}{|c|c|c|c|c|c|} % centered columns (4 columns)
\cline{1-6}
\multicolumn{2}{|c|}{Parameters} &
$\!
\begin{aligned}
n_o&=9\\
n_x&=8\\
n_d&=40
\end{aligned}$
&
$\!
\begin{aligned}
n_o&=9\\
n_x&=7\\
n_d&=30
\end{aligned}$
&
$\!
\begin{aligned}
n_o&=9\\
n_x&=6\\
n_d&=20
\end{aligned}$
&
$\!
\begin{aligned}
n_o&=9\\
n_x&=5\\
n_d&=10
\end{aligned}$\\
\hline
\multirow{2}{*}{Running Time} &
Spectral & ~~6.8 hours~~ & ~~6.4 hours~~ & ~~6.4 hours~~ & ~~6.3 hours~~\\
& EM & $>2$ days & $>2$ days &  $>2$ days & $>2$ days \\
\hline
%\multirow{2}{*}{AUC} &
%Spectral & 0.9066 & 0.8010 & 0.7215 & 0.9019\\
%& EM & 0.8884 & 0.8873 & 0.8959 & 0.9015\\
%\hline
\end{tabular}
\vspace*{-1mm}
\caption{Comparison of running time for EM and spectral algorithm for multiple model parameters. Spectral algorithm is several orders of magnitude faster as compared to EM, offering significant computational savings.} % title of Table
\label{tb:table}
\end{table*}

%We applied the spectral algorithm to compute the normalized joint log-likelihood of the observed pilot actions. The high-likelihood sequences were considered normal and low-likelihood sequences were classified as anomalous (see center panel in Figure \ref{fig:real}). The algorithm learnt the observable representation model in $189$ seconds and achieved a high classification accuracy, which can be observed on the right panel, showing the ROC curve with Area Under Curve (AUC) score of $0.96$.

%% file: Conclusion.tex
In this paper, we present a novel spectral algorithm to perform inference in HSMM. We derive an observable representation of the model which can be computed from the data sample moments of size logarithmic in the maximum length of latent state persistence. Based on the representation and exploiting the homogeneity of the model, we present an efficient approach to inference, which ensures that the number of matrix multiplications and inverses needed to estimate the probability of an observed sequence is fixed and independent of its length. Moreover, the empirical evaluation on synthetic and real flight datasets illustrate the promise of the proposed spectral algorithm. In particular, the spectral method gets similar or better performance than EM as the size of the training dataset increases, and at the same time the spectral method is orders of magnitude faster than EM providing significant computational savings. Going forward, we plan to explore if similar spectral methods can be developed for inference in more general dynamic Bayesian networks.

%% file: proofs.tex
In this Section we analyze the properties of Algorithms \ref{alg1} and \ref{alg2} and present proofs for Theorems \ref{mainTheoremm} and \ref{mainTheorem2}.

\subsection{Analysis of Algorithm \ref{alg1}}
\label{sec:AnalysisAlg1}

Here we provide analysis of the Algorithm \ref{alg1} and study the rank structure of matrix $\mathbf{T}$ in order to prove Theorem \ref{mainTheoremm}. To understand the analysis, it is important to know how the structure of matrix $\underset{\bm{\mathsf{X}}_{R_{t+1}}|x_td_t}{\mathbf{T}}$ evolves across iterations. For this, we present in Figure \ref{fig:Talg} a schematic description of a few steps of the algorithm.
\begin{figure*}[!b]
\centering
   \includegraphics[width=\textwidth]{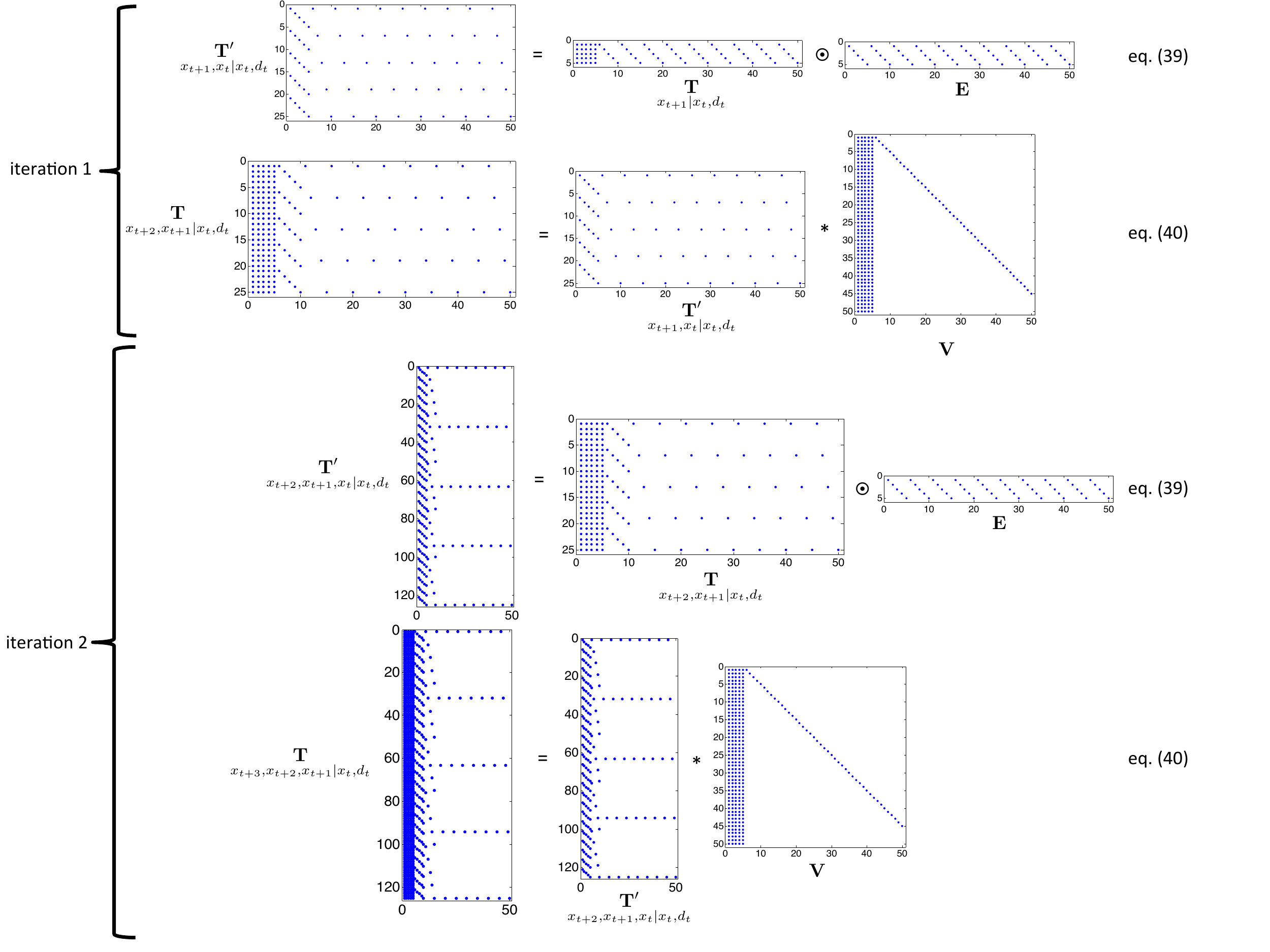}
   \caption{Schematic representation of Algorithm \ref{alg1}. This example illustrates the HSMM with $n_x=5$ and $n_d=10$. The non-zero matrix elements are displayed as dots.}
   \label{fig:Talg}
\end{figure*}
For the analysis we will need to establish certain auxiliary results.

\begin{lemma}
\label{KhatriRao1}
Let $\mathbf{A} \in \mathbb{R}^{m\times n}$ be a matrix with no all-zero columns then $rank\left(\mathbf{I}\odot\mathbf{A}\right) = rank\left(\mathbf{A}\odot\mathbf{I}\right) = n$, where $\odot$ denotes Khatri-Rao product and $\mathbf{I} \in \mathbb{R}^{n\times n}$. 
\end{lemma}

\begin{proof}
Let $\mathbf{K} = \left(\mathbf{I}\odot\mathbf{A}\right) \in \mathbb{R}^{mn\times n}$. By definition of Khatri-Rao product, $\mathbf{K}(:,j) = \mathbf{e}_j\otimes\mathbf{A}(:,j)$, for $j=1,\ldots,n$, which consists of zeros, except for rows $(j-1)m+1,\ldots,(j-1)m+m$, containing the column $\mathbf{A}(:,j)$. Here $\otimes$ denotes Kronecker product and $\mathbf{e}_j$ is everywhere zero except for position $j$ which is $1$. As long as there is no all-zero columns in $\mathbf{A}$, each column of $\mathbf{K}$ is independent of each other and therefore the rank is $n$. Moreover, since the matrix $\mathbf{A}\odot\mathbf{I}$ is a row-permuted version of $\mathbf{A}\odot\mathbf{I}$, their ranks are the same.
\end{proof}

\begin{lemma}
\label{KhatriRao2}
Define a block-row matrix $\mathbf{M} = \left[\mathbf{A}_1 ~\mathbf{A}_2 ~\cdots 
~\mathbf{A}_k\right]\in\mathbb{R}^{m\times kn}$, where each  $\mathbf{A}_i \in \mathbb{R}^{m\times n}$. Define by $r_j, ~j=1,\ldots,n$ the rank of matrix $\left[\mathbf{A}_1(:,j) ~\cdots ~\mathbf{A}_k(:,j)\right]$ composed of $j$th columns of $\mathbf{A}$'s, and let $\mathbf{E} = \left[\mathbf{I} ~\mathbf{I} ~\cdots 
~\mathbf{I}\right]\in\mathbb{R}^{n\times kn}$, where $\mathbf{I} \in \mathbb{R}^{n\times n}$. Then the rank of matrix $\mathbf{W} = \mathbf{M}\odot\mathbf{E} \in \mathbb{R}^{mn\times kn}$, obtained using a Khatri-Rao product, is $\min(mn, \sum_jr_j)$.
\end{lemma}

\begin{proof}
First note that $\mathbf{M}\odot\mathbf{E}$ and $\mathbf{E}\odot\mathbf{M}$ are row permuted version of each other, so they have the same rank. Therefore, consider $\mathbf{W}^{\prime} = \mathbf{E}\odot\mathbf{M} = \left[\mathbf{I}\odot\mathbf{A}_1 \cdots \mathbf{I}\odot\mathbf{A}_k\right]$. Also, note that $\mathbf{e}_j\otimes\left[\mathbf{A}_1(:,j) \cdots ~\mathbf{A}_k(:,j)\right]$, $j=1,\ldots,n$ is a matrix which consists of zeros except for rows $(j-1)m+1,\ldots,(j-1)m+m$ where it contains the columns $\left[\mathbf{A}_1(:,j) ~\cdots ~\mathbf{A}_k(:,j)\right]$. The rank of these columns is $r_j$ and all other columns in $\mathbf{W}$ are independent of them due to the structure of the Khatri-Rao product. Therefore, each set of such columns adds $r_j$ to the total rank. Since the overall rank of $\mathbf{W}$ cannot exceed either the number of rows or columns, we conclude that $rank(\mathbf{W})=\min(mn, \sum_jr_j)$.
\end{proof}

\begin{lemma}
\label{SubsetIndep}
Let $V = \{\mathbf{v}_1,\ldots,\mathbf{v}_n\}$ be a set of linearly independent vectors. Define $\mathbf{u} = \sum_{i=1}^nc_i\mathbf{v}_i$, where coefficients $c_i \neq 0, i=1,\ldots,n$. Define $U$ to be a strict subset of $V$, i.e., $U \subset V $, then a set of vectors $\mathbf{u} \cup U$ is independent.
\end{lemma}

\begin{proof}
Define $\{1,\ldots,n\} = \alpha \cup \bar{\alpha}$, where $\alpha$ denotes a subset of indices for vectors corresponding to $U$. Then we can write $\mathbf{u} = \sum_{i:i\in\alpha}c_i\mathbf{v}_i + \sum_{j:j\in\bar{\alpha}}c_j\mathbf{v}_j$. 

Assuming the opposite, i.e., $\mathbf{u} \cup U$ are dependent, we can write $k_0\mathbf{u} + \sum_{i:i\in\alpha}k_i\mathbf{v}_i = 0$ where $k_0 \neq 0$ and some of $k_i, i\in\alpha$ are also must be non-zero. Substituting the definition of $\mathbf{u}$ and rearranging the terms, we get:
\begin{align*}
k_0\sum_{i:i\in\alpha}(c_i+k_i)\mathbf{v}_i + k_0\sum_{j:j\in\bar{\alpha}}c_j\mathbf{v}_j  = 0.
\end{align*}
\noindent Since $c_j \neq 0, j\in\bar{\alpha}$, the above equation claims the linear dependence of vectors in $V$, which is a contradiction of our assumption and so $\mathbf{u} \cup U$ are independent.
\end{proof}

We are now ready to analyze Algorithm \ref{alg1}. It can be verified that \eqref{eq:V} is of the form: 
\begin{align}
\label{eq:matrixV}
\mathbf{V} = 
\begin{bmatrix}
\Psi & \vline
\begin{array}{ccc}
\mathbf{I} & &\\
& \ddots &\\
&&\mathbf{I}\\
\hline
\mathbf{0} & \cdots & \mathbf{0}
\end{array}
\end{bmatrix} \in \mathbb{R}^{n_xn_d~\times ~n_xn_d}
\quad\quad\text{where}\quad\Psi = 
\begin{bmatrix}
\text{diag}\left[D(1,:)\right]\mathcal{X}\\
\text{diag}\left[D(2,:)\right]\mathcal{X}\\
\vdots\\
\text{diag}\left[D(n_d,:)\right]\mathcal{X}\\
\end{bmatrix} \in \mathbb{R}^{n_xn_d~\times ~n_x},
\end{align} 
\noindent where $\text{diag}\left[D(i,:)\right]$ is the diagonal matrix with $i$th row from $D$ on the diagonal. Note that we can also write $\Psi = \left(D\odot\mathbf{I}\right)\mathcal{X}$. Observe that the rank of $\mathbf{V}$ is $n_xn_d$ because  the $n_x(n_d-1) \times n_x(n_d-1)$ block diagonal matrix delineated in \eqref{eq:matrixV} and the last $n_x\times n_x$ block matrix $\text{diag}\left[D(n_d,:)\right]\mathcal{X}$ in $\Psi$ together comprising $n_xn_d$ independent columns of $\mathbf{V}$. Note that $\text{diag}\left[D(n_d,:)\right]\mathcal{X}$ has rank $n_x$ because $\mathcal{X}$ is full rank and $D(n_d,:)$ is non-zero, which follows from assumptions $A1$ and $A2$. As a side note observe that the requirement to have $D(n_d,:)$ non-zero implies that there is a non-zero probability of maximum state persistence.  

In analyzing the Algorithm \ref{alg1}, it would be useful to denote the matrices at iteration $i$ in \eqref{eq:expan} and \eqref{eq:prop} as
\begin{align*} \underset{x_{t\hspace{-1pt}+\hspace{-1pt}i},~\ldots~,x_{t\hspace{-1pt}+\hspace{-1pt}1}|x_t, d_t}{\mathbf{T}} &= [\mathbf{A}_1^{(i)} ~\cdots ~\mathbf{A}_{n_d}^{(i)}]\\ \underset{x_{t\hspace{-1pt}+\hspace{-1pt}i},~\ldots~,x_{t\hspace{-1pt}+\hspace{-1pt}1},x_t|x_t, d_t}{\mathbf{T}^{\prime}} &= [\mathbf{B}_1^{(i)} ~\cdots~ \mathbf{B}_{n_d}^{(i)}]\\
\underset{x_{t\hspace{-1pt}+\hspace{-1pt}i\hspace{-1pt}+\hspace{-1pt}1},\ldots,x_{t\hspace{-1pt}+\hspace{-1pt}2},x_{t\hspace{-1pt}+\hspace{-1pt}1}|x_t, d_t}{\mathbf{T}} &= [\mathbf{C}_1^{(i)} ~\cdots ~\mathbf{C}_{n_d}^{(i)}].
\end{align*}

\noindent Moreover, utilizing the structure of matrix $\mathbf{V}$ from \eqref{eq:matrixV}, the operations involved in step \eqref{eq:prop} are as follows: 
\begin{align}
\label{eq:Vtrans}
\Big[\mathbf{C}_1^{(i)} ~~\mathbf{C}_2^{(i)} ~~\mathbf{C}_3^{(i)} ~~\cdots ~\mathbf{C}_{n_d}^{(i)}\Big] = \Big[[\mathbf{B}_1^{(i)} ~\cdots~ \mathbf{B}_{n_d}^{(i)}]\Psi ~~~\mathbf{B}_1^{(i)} ~~\mathbf{B}_2^{(i)} ~~\cdots ~~\mathbf{B}_{n_d-1}^{(i)}\Big].
\end{align}

\noindent With the above information we can now present the proof of Theorem \ref{mainTheoremm}:

\begin{proof2}
At the start of the algorithm, we have $\underset{x_{t+1}|x_t, d_t}{\mathbf{T}} = \left[\mathcal{X}~\mathbf{I}~\cdots~\mathbf{I}\right] = [\mathbf{A}_1^{(1)} \cdots \mathbf{A}_{n_d}^{(1)}]$, which has rank $n_x$. The rank of matrix $\left[\mathbf{A}_1^{(1)}(:,l)\cdots \mathbf{A}_{n_d}^{(1)}(:,l)\right]$ for $l=1,\ldots, n_x$ is $r_l = 2$ since among all the columns only two of them are independent. Therefore, according to Lemma \ref{KhatriRao2}, the result of operations in \eqref{eq:expan}, has rank $\sum_{l}r_l = 2n_x$. Moreover, we note that since $[\mathbf{B}_1^{(1)} ~\mathbf{B}_2^{(1)} ~\cdots ~\mathbf{B}_{n_d}^{(1)}] = [\mathcal{X}\hspace{-1pt}\odot\hspace{-1pt}\mathbf{I} ~~\mathbf{I}\hspace{-1pt}\odot\hspace{-1pt}\mathbf{I}~ \cdots~ \mathbf{I}\hspace{-1pt}\odot\hspace{-1pt}\mathbf{I}]$, it can be seen that its $2n_x$ independent vectors can be formed by the columns $[\mathbf{B}_1^{(1)} ~\mathbf{B}_2^{(1)}]$, so that the rank of $\left[\mathbf{B}_1^{(1)}(:,l) \cdots \mathbf{B}_{n_d}^{(1)}(:,l)\right]$ for $l=1,\ldots, n_x$ is $2$.

Next, since the rank of $\mathbf{V}$ is $n_xn_d$, the operations in \eqref{eq:prop} produce matrix $[\mathbf{C}_1^{(1)}  ~\mathbf{C}_2^{(1)} ~\cdots ~\mathbf{C}_{n_d}^{(1)}]$ with the rank still being $2n_x$. Moreover, the columns of $\mathbf{C}_1^{(1)}$ are linearly dependent on the rest of the columns, $[\mathbf{C}_2^{(1)} ~\cdots ~\mathbf{C}_{n_d}^{(1)}]$, due to \eqref{eq:Vtrans}. However, the rank of $\left[\mathbf{C}_1^{(1)}(:,l) \cdots \mathbf{C}_{n_d}^{(1)}(:,l)\right]$ is now $r_l = 3$ for $l=1,\ldots, n_x$. To understand this, note that 
\begin{align*}
[\mathbf{B}_1^{(1)} ~~\mathbf{B}_2^{(1)}~\cdots ~ \mathbf{B}_{n_d}^{1}] &= [\mathcal{X}\hspace{-2pt}\odot\hspace{-2pt}\mathbf{I} ~~\mathbf{I}\hspace{-2pt}\odot\hspace{-2pt}\mathbf{I}~ \cdots~ \mathbf{I}\hspace{-2pt}\odot\hspace{-2pt}\mathbf{I}]\\
[\mathbf{C}_1^{(1)}~~\mathbf{C}_2^{(1)}~~\mathbf{C}_3^{(1)}~\cdots~ \mathbf{C}_{n_d}^{(1)}] &= [\mathbf{C}_1^{(1)}~ ~\mathcal{X}\hspace{-2pt}\odot\hspace{-2pt}\mathbf{I} ~~\mathbf{I}\hspace{-2pt}\odot\hspace{-2pt}\mathbf{I}~ \cdots~ \mathbf{I}\hspace{-2pt}\odot\hspace{-2pt}\mathbf{I}],
\end{align*}
\noindent where, according to \eqref{eq:Vtrans}, $\mathbf{C}_1^{(1)} = [\mathbf{B}_1^{(1)} \cdots \mathbf{B}_{n_d}^{(1)}]\Psi$. As we established before, the rank of the matrix  $\left[\mathbf{C}_2^{(1)}(:,l) \cdots \mathbf{C}_{n_d}^{(1)}(:,l)\right] = \left[\mathbf{B}_1^{(1)}(:,l) \cdots \mathbf{B}_{n_{d-1}}^{(1)}(:,l)\right]$ is $r_l=2$. Moreover, it can also be checked that $\mathbf{C}_1^{(1)}(:,l)$ is independent of $\left[\mathbf{C}_2^{(1)}(:,l) \cdots \mathbf{C}_{n_d}^{(1)}(:,l)\right]$ due to Lemma \ref{SubsetIndep}. Clearly, then the cumulative rank of $\left[\mathbf{C}_1^{(1)}(:,l) \cdots \mathbf{C}_{n_d}^{(1)}(:,l)\right]$ is $3$ for $l=1,\ldots, n_x$.

To generalize, if at the iteration $i$ the rank of $\left[\mathbf{A}_1^{(i)} \cdots \mathbf{A}_{n_d}^{(i)}\right]$ is $in_x$ while the rank of $\left[\mathbf{A}_1^{(i)}(:,l) \cdots \mathbf{A}_{n_d}^{(i)}(:,l)\right]$ is $(i+1)$, then the operations in step \eqref{eq:expan} produce $\left[\mathbf{B}_1^{(i)} \cdots \mathbf{B}_{n_d}^{(i)}\right]$ having rank $(i+1)n_x$ due to Lemma \ref{KhatriRao2}. The step in \eqref{eq:prop} keeps the rank of $\left[\mathbf{C}_1^{(i)} \cdots \mathbf{C}_{n_d}^{(i)}\right]$ at $(i+1)n_x$ due to the full rank structure of $\mathbf{V}$. At the same time, this step increases the rank of $\left[\mathbf{C}_1^{(i)}(:,l) \cdots \mathbf{C}_{n_d}^{(i)}(:,l)\right]$ to $(i+2)$ due to Lemma \ref{SubsetIndep}, i.e., independence of $\mathbf{C}_1^{(i)}(:,l)$ from $\left[\mathbf{C}_2^{(i)}(:,l) \cdots \mathbf{C}_{n_d}^{(i)}(:,l)\right]$ with the latter having the rank $(i+1)$. Therefore, each iteration increases the rank of matrix $\mathbf{T}$ by $n_x$ and so after $2 \leq \ell  \leq n_d$ steps the rank of the resulting matrix $\underset{\bm{\mathsf{X}}_{R_{t+1}}|x_td_t}{\mathbf{T}}$ is $\ell n_x$. 

Note that if $\ell =1$ then the Algorithm \ref{alg1} is not executed and returns the trivial $\underset{x_{t+1}|x_t, d_t}{\mathbf{T}}$ with rank $n_x$. On the other hand, if $\ell  > n_d$ then the rank of $\underset{\bm{\mathsf{X}}_{R_{t+1}}|x_td_t}{\mathbf{T}}$ is $n_xn_d$ since this is the number of columns in that matrix and so is the maximum achievable rank.
\end{proof2}

\subsection{Analysis of Algorithm \ref{alg2}}
\label{sec:AnalysisAlg2}

In this Section we provide analysis of the Algorithm \ref{alg2} in order to prove Theorem \ref{mainTheorem2}. Similarly as in Section \ref{sec:AnalysisAlg1}, it is instructive to visualize the progress of Algorithm \ref{alg2}. Figure \ref{fig:Talgeff} shows a schematic description of a few steps of the algorithm.

\begin{figure*}[!tp]
\centering
   \includegraphics[width=0.58\textwidth]{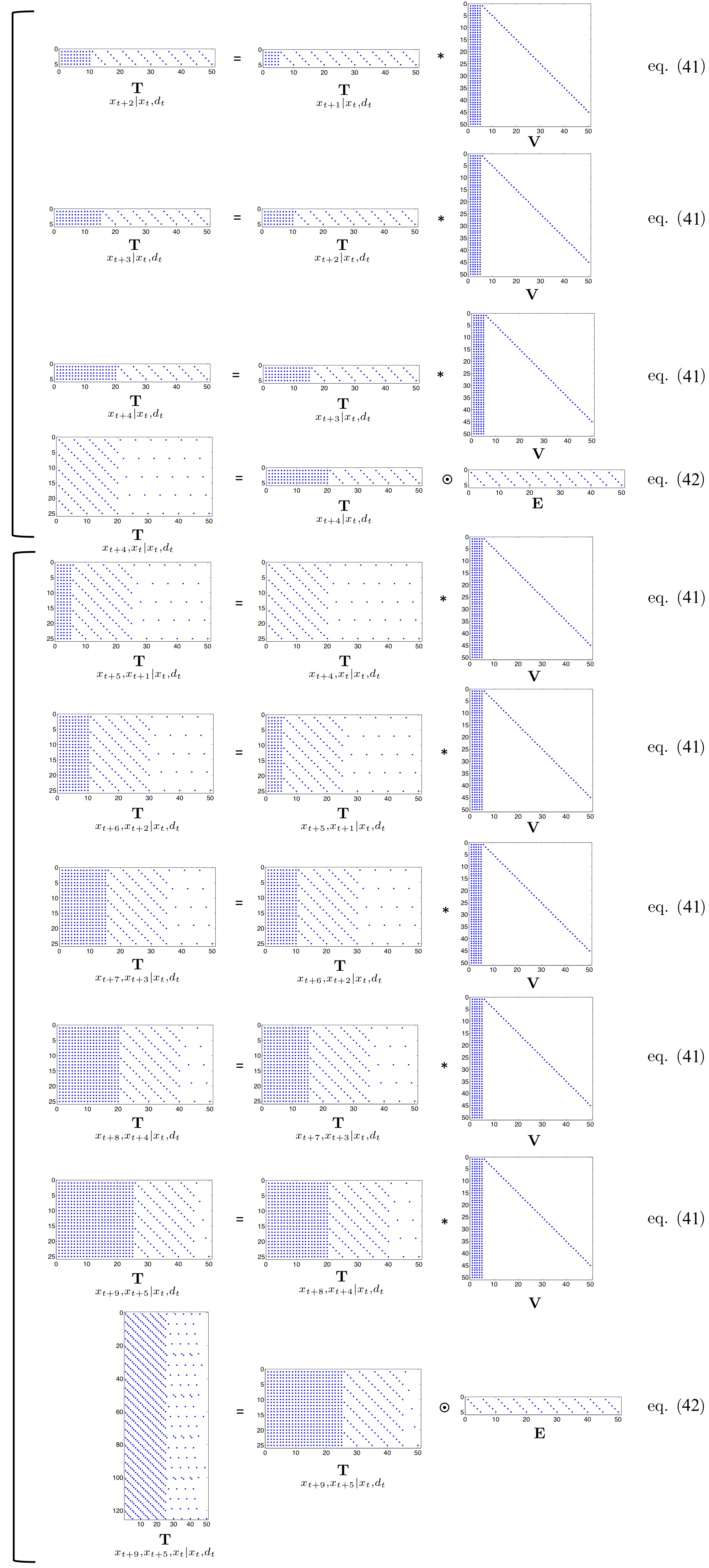}
   \caption{Schematic representation of Algorithm \ref{alg2}. This example illustrates the HSMM with $n_x=5$ and $n_d=10$. The non-zero matrix elements are displayed as dots.}
   \label{fig:Talgeff}
\end{figure*}
\noindent We are now ready to present the proof of Theorem \ref{mainTheorem2}.

\begin{proof4}
For the proof, we refer back to Algorithm \ref{alg1} and the proof of Theorem \ref{mainTheoremm}. Recall, that at iteration $i=1$, the result of step \eqref{eq:expan} is a matrix $[\mathbf{B}_1^{(1)} \cdots \mathbf{B}_{n_d}^{(1)}] \in \mathbb{R}^{n_x^{2} \times n_xn_d}$, whose rank is $2n_x$, since $\left[\mathbf{A}_1^{(1)}(:,l) \cdots \mathbf{A}_{n_d}^{(1)}(:,l)\right] = \left[\mathcal{X}~ \mathbf{I}\cdots\mathbf{I}\right]\in \mathbb{R}^{n_x\times n_xn_d}$ for $l=1,\ldots, n_x$ had two independent columns. Then, the transformations in step \eqref{eq:prop} produced $\left[\mathbf{C}_1^{(1)}(:,l) \cdots \mathbf{C}_{n_d}^{(1)}(:,l)\right]$ for $l=1,\ldots,n_x$ with rank $3n_x$.

Note that if $n_x > 2$ then $\left[\mathbf{A}_1^{(1)}(:,l) \cdots \mathbf{A}_{n_d}^{(1)}(:,l)\right]$ potentially can have a rank up to $n_x$, while in Algorithm \ref{alg1} we only have it equal to $2$. It turns out that if we apply step \eqref{eq:prop} multiple times and use Lemma \ref{SubsetIndep}, we can increase the rank of $\left[\mathbf{C}_1^{(1)}(:,l) \cdots \mathbf{C}_{n_d}^{(1)}(:,l)\right]$ for $l=1,\ldots,n_x$ to $n_x$. 

Specifically, consider the step \eqref{eq:prop2}. Then at iteration $i=1$ we have $[\mathbf{A}_1^{(1)} \cdots \mathbf{A}_{n_d}^{(1)}] = [\mathbf{B}_1^{(1)} \cdots \mathbf{B}_{n_d}^{(1)}]$ and for $l=1,\ldots,n_x$ the two independent columns are $\left[\mathbf{B}_1^{(1)}(:,l)~~\mathbf{B}_2^{(1)}(:,l)\right] = \left[\mathcal{X}(:,l) ~~\mathbf{I}(:,l)\right]$. The result of step  $\eqref{eq:prop2}$ gives us then three independent columns
\begin{align*}
\left[\mathbf{C}_1^{(1)}(:,l)~~\mathbf{C}_2^{(1)}(:,l) ~~\mathbf{C}_3^{(1)}(:,l)\right] = \left[\mathbf{C}_1^{(1)}(:,l) ~~\mathcal{X}(:,l) ~~\mathbf{I}(:,l)\right],
\end{align*}
\noindent where $\mathbf{C}_1^{(1)} = [\mathcal{X} ~\mathbf{I} ~\cdots ~\mathbf{I}]\Psi$. The independence follows from Lemma \ref{SubsetIndep}. The repeated application of step  $\eqref{eq:prop2}$ one more time gives four independent columns
\begin{align*}
\left[\mathbf{C}_1^{(2)}(:,l)~~\mathbf{C}_2^{(2)}(:,l) ~~\mathbf{C}_3^{(2)}(:,l)~~\mathbf{C}_4^{(2)}(:,l)\right] = \left[\mathbf{C}_1^{(2)}(:,l) ~~\mathbf{C}_1^{(1)}(:,l) ~~\mathcal{X}(:,l) ~~\mathbf{I}(:,l)\right],
\end{align*}
\noindent where $\mathbf{C}_1^{(2)} = [\mathbf{C}_1^{(1)} \cdots \mathbf{C}_{n_d}^{(1)}]\Psi$. Observe that since the number of rows is $n_x$, we can increase the rank at most up to $n_x$. Therefore, if in the beginning we had $two$ independent columns and we want to get $n_x$ independent columns, we would need to apply the step \eqref{eq:prop2} $n_x-2$ times, so as to have the matrix $[\mathbf{C}_1^{(n_x-2)}(:,l) ~\cdots ~\mathbf{C}_{n_d}^{(n_x-2)}(:,l)]$ with rank $n_x$.

If we now apply step \eqref{eq:expan2} it will give us $[\mathbf{A}_1^{(1)} ~\cdots ~\mathbf{A}_{n_d}^{(1)}] \in \mathbb{R}^{n_x^{2}\times n_xn_d}$ with rank $n_x^2$ due to Lemma \ref{KhatriRao2}. Continuing in this manner, we can again repeatedly apply the step \eqref{eq:prop2} to create a matrix with a rank at most $n_x^2$, since there are $n_x^2$ rows and assuming that $n_xn_d \geq n_x^2$. The number of times we need to apply \eqref{eq:prop2} is now $n_x^2 - n_x$ since we need to go from $n_x$ to $n_x^2$ independent columns.

In general, the step \eqref{eq:prop2} needs to be applied $n_x^c - n_x^{c-1}$, in order to obtain $n_x^c$ independent columns. The application of step \eqref{eq:expan2} then creates $\mathbf{T}$ with rank $n_x^{c+1}$. Note, that since $\mathbf{T}$ has $n_xn_d$ columns, the maximum achievable rank is $n_xn_d$.
\end{proof4}

Observe that the above proof also provided the method for selecting the non-sequential observations $\bm{\mathsf{X}}_{R_{t+1}}$. Specifically, since the set of observations $\bm{\mathsf{X}}_{R_{t+1}} = \{o_{t+2}, \ldots\}$ must start from observation $o_{t+2}$ and $|\bm{\mathsf{X}}_{R_{t+1}}| = \ell$, we denote $s = t+2$. Then, $i$th added observation is $o_{s+(n_d-1)-(n_x^i -1)}$ for $i=0,\ldots,\ell-2$ and the $\ell$th observation is $o_{s} = o_{t+2}$. For tensor  $\underset{\bm{\mathsf{O}}_{R_{t+1}}|x_td_t}{\bm{\mathscr{F}}}$ to achieve rank $n_xn_d$ we need to add $\ell = \lceil 1 + \frac{\log n_d}{\log n_x}\rceil$ observations.

%% file: startEndModel.tex
In this Section we present the derivations for the initial and final steps of HSMM, which were omitted from the main text. Specifically, this amounts to computing the factor ${\bm{\mathscr{X}}}$ for two parts of the model, corresponding to $\mathbb{X}_{root}$ and $\mathbb{X}_{T}$ in Figures \ref{fig:hsmmStart} and \ref{fig:hsmmEnd}. The derivations for all other parts of HSMM were presented in the main text and this supplement. 

\begin{figure}[!ht]
\centering
   \includegraphics[width=0.7\textwidth]{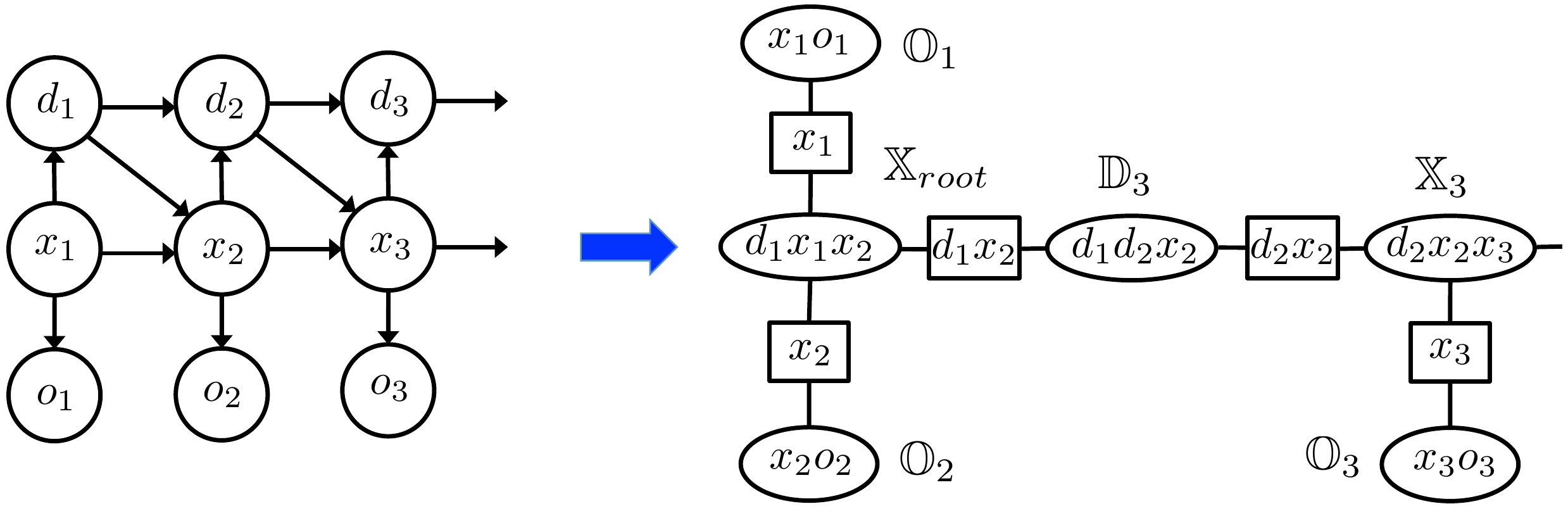}
   \caption{Part of HSMM corresponding to the initial time stamps and the related part of junction tree.}
   \label{fig:hsmmStart}
\end{figure}

\begin{figure}[!ht]
\centering
   \includegraphics[width=0.9\textwidth]{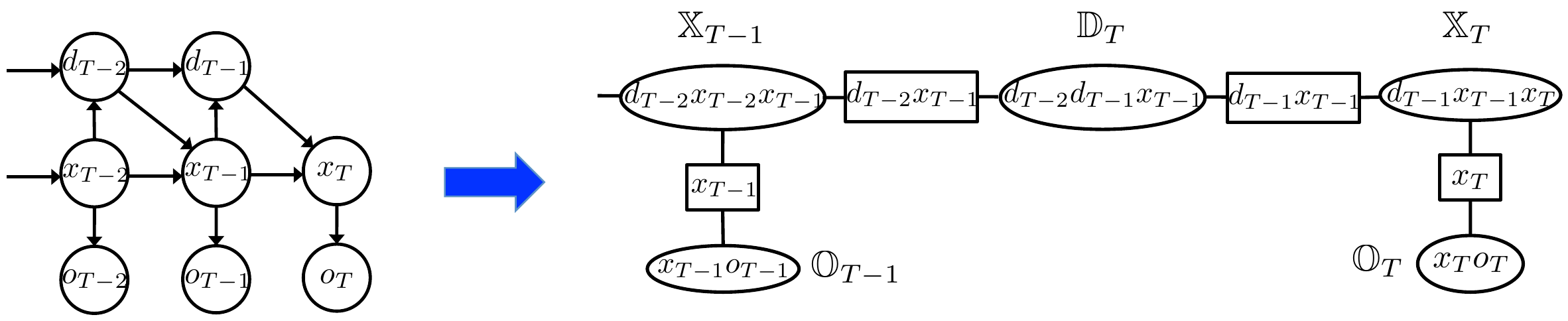}
   \caption{Part of HSMM corresponding to the final time stamps and the related part of junction tree.}
   \label{fig:hsmmEnd}
\end{figure}

\noindent To begin, recall the expression for the joint likelihood of the observed sequence:
\begin{align*}
\underset{o_1,\ldots,o_T}{\bm{\mathscr{P}}} = 
\prod_{t} \underset{d_{t-1}|x_{t-1}d_{t-2}}{\bm{\mathscr{D}}}
\times_{x_{t-1}d_{t-1}}\left(\underset{x_{t}|x_{t-1}d_{t-1}}{\bm{\mathscr{X}}}\times_{x_t}\underset{o_t|x_t}{\bm{\mathscr{O}}}\right)
\end{align*}
\noindent and rewrite the above expression by keeping only the initial and final factors:
\begin{align}
\underset{o_1,\ldots,o_T}{\bm{\mathscr{P}}} =& 
\left(\underset{o_{1}|{x_{1}}}{\bm{\mathscr{O}}}\times_{x_1}\left(\underset{x_{2}x_2|x_{1}d_{1}}{\bm{\mathscr{X}}}\times_{x_2}\underset{o_{2}|{x_{2}}}{\bm{\mathscr{O}}}\right)\right)\times_{x_2d_1}\underset{{d_2|x_2x_2d_{1}}}{\bm{\mathscr{D}}}\times \cdots\nonumber\\ &\cdots\times\underset{d_{T-1}|x_{T-1}x_{T-1}d_{T-2}}{\bm{\mathscr{D}}}\times_{x_{T-1}d_{T-1}}\left(\underset{x_{T}|x_{T-1}d_{T-1}}{\bm{\mathscr{X}}}\times_{x_T}\underset{o_{T}|{x_{T}}}{\bm{\mathscr{O}}}\right).
\label{eq:PstartEnd}
\end{align}

\noindent Introduce the identity tensors into \eqref{eq:PstartEnd}, regroup the terms and extract the factors ${\bm{\mathscr{X}}}$:
\begin{align}
\label{eq:X1}
\underset{\omega_{x_1}\omega_{x_2}\omega_{x_2d_1}}{\tilde{\bm{\mathscr{X}}}} &=
\underset{\omega_{x_1}|x_{1}}{\bm{\mathscr{F}}}\times_{x_1}
\left(\underset{x_2x_{2}|x_{1}d_{1}}{\bm{\mathscr{X}}}\times_{x_2}\underset{\omega_{x_2}|x_{2}}{\bm{\mathscr{F}}}\right)
\times_{x_2d_1}\underset{\omega_{x_2d_{1}}|x_2d_{1}}{\bm{\mathscr{F}}}\\
\label{eq:X2}
\underset{\omega_{x_{T-1}d_{T-1}}\omega_{x_{T}}}{\tilde{\bm{\mathscr{X}}}} &=
\underset{\omega_{x_{T-1}d_{T-1}}|x_{T-1}d_{T-1}}{\bm{\mathscr{F}}^{-1}}\times_{x_{T-1}d_{T-1}}
\left(\underset{x_{T}|x_{T-1}d_{T-1}}{\bm{\mathscr{X}}}\times_{x_T}\underset{\omega_{x_T}|x_{T}}{\bm{\mathscr{F}}}\right).
\end{align} 

\noindent Defining the observable sets $\omega_{x_1} = o_1$, $\omega_{x_2} = o_2$ and $\omega_{x_2d_1} ={\bm{\mathsf{O}}}_{R_3}$ we can rewrite \eqref{eq:X1} as follows:
\begin{align}
\label{eq:X11}
\underset{o_{1}o_{2}{\bm{\mathsf{O}}}_{R_3}}{\tilde{\bm{\mathscr{X}}}} =
\underset{o_{1}|x_{1}}{\bm{\mathscr{F}}}\times_{x_1}
\left(\underset{x_{2}x_2|x_{1}d_{1}}{\bm{\mathscr{X}}}\times_{x_2}\underset{o_{2}|x_{2}}{\bm{\mathscr{F}}}\right)
\times_{x_2d_1}\underset{{\bm{\mathsf{O}}}_{R_3}|x_2d_{1}}{\bm{\mathscr{F}}}.
\end{align} 

Note that since all the factors participating in \eqref{eq:X11} are valid probability distributions, the resulting factor, i.e., $\underset{o_{1}o_{2}{\bm{\mathsf{O}}}_{R_3}}{\tilde{\bm{\mathscr{X}}}}$ is also a valid probability distribution, so it can be estimated directly from data. This is in contrast to the derivations we made for other parts of the model, where we had to perform additional transformations such as, for example in \eqref{eq:DtensorObs}, in order to bring to the form, which could be estimated from the data samples.

In order to estimate \eqref{eq:X2}, we compare it to the similar factor we considered in the main paper:
\begin{align}
\label{eq:X22}
\underset{\omega_{x_{t\hspace{-1pt}-\hspace{-1pt}1}d_{t\hspace{-1pt}-\hspace{-1pt}1}}\omega_{x_t}\omega_{x_{t}d_{t\hspace{-1pt}-\hspace{-1pt}1}}}{\tilde{\bm{\mathscr{X}}}} =\hspace{-5pt}
\underset{\omega_{x_{t\hspace{-1pt}-\hspace{-1pt}1}d_{t\hspace{-1pt}-\hspace{-1pt}1}}|x_{t\hspace{-1pt}-\hspace{-1pt}1}d_{t\hspace{-1pt}-\hspace{-1pt}1}}{\bm{\mathscr{F}}^{-1}}\hspace{-5pt}\times_{x_{t\hspace{-1pt}-\hspace{-1pt}1}d_{t\hspace{-1pt}-\hspace{-1pt}1}}
\hspace{-2pt}\left(\underset{x_{t}x_t|x_{t\hspace{-1pt}-\hspace{-1pt}1}x_{t\hspace{-1pt}-\hspace{-1pt}1}d_{t\hspace{-1pt}-\hspace{-1pt}1}}{\bm{\mathscr{X}}}\hspace{-2pt}\times_{x_t}\hspace{-2pt}\underset{\omega_{x_t}|x_{t}}{\bm{\mathscr{F}}}\right)
\hspace{-2pt}\times_{x_td_{t-1}}\hspace{-3pt}\underset{\omega_{x_td_{t\hspace{-1pt}-\hspace{-1pt}1}}|x_td_{t-1}}{\bm{\mathscr{F}}},
\end{align} 
and observe that the last factor $\underset{\omega_{x_td_{t\hspace{-1pt}-\hspace{-1pt}1}}|x_td_{t-1}}{\bm{\mathscr{F}}}$ in \eqref{eq:X22} is a conditional probability distribution, which has the following marginalization property
\begin{align}
\label{eq:marg}
\underset{\omega_{x_td_{t\hspace{-1pt}-\hspace{-1pt}1}}|x_td_{t-1}}{\bm{\mathscr{F}}} \times_{\omega_{x_td_{t\hspace{-1pt}-\hspace{-1pt}1}}}~~\underset{\omega_{x_td_{t\hspace{-1pt}-\hspace{-1pt}1}}}{{\mathbf{1}}} = \underset{x_td_{t-1}}{{\mathbf{1}}},
\end{align}
\noindent where $\mathbf{1}$ is the tensor, which has all elements equal to $1$. The above can also be written in the scalar notations, $\sum_{\omega_{x_td_{t\hspace{-1pt}-\hspace{-1pt}1}}}p(\omega_{x_td_{t\hspace{-1pt}-\hspace{-1pt}1}}|x_td_{t-1}) = 1$ for each value of $x_td_{t-1}$. Therefore, if we apply \eqref{eq:marg} to \eqref{eq:X22}, we get $\underset{\omega_{x_{t\hspace{-1pt}-\hspace{-1pt}1}d_{t\hspace{-1pt}-\hspace{-1pt}1}}\omega_{x_t}}{\tilde{\bm{\mathscr{X}}}}$, which is the time-shifted version of $\underset{\omega_{x_{T-1}d_{T-1}}\omega_{x_{T}}}{\tilde{\bm{\mathscr{X}}}}$. Therefore, to compute \eqref{eq:X2}, we estimate the tensor in \eqref{eq:tensorXobs}, i.e., 
\begin{align*}
\underset{\bm{\mathsf{O}}_{R_{t}}o_t\bm{\mathsf{O}}_{R_{t}}}{\tilde{\bm{\mathscr{X}}}} = \underset{\bm{\mathsf{O}}_{L_{t}}\bm{\mathsf{O}}_{R_{t}}}{\bm{\mathscr{M}}^{-1}} \times_{\bm{\mathsf{O}}_{L_{t}}}\underset{\bm{\mathsf{O}}_{L_{t}}\bm{\mathsf{O}}_{R_{t}}o_t}{\bm{\mathscr{M}}},
\end{align*}
and marginalize out the right set of modes, corresponding to $\bm{\mathsf{O}}_{R_{t}}$. Alternatively, we can use the batch estimate
\begin{align*}
\tilde{\bm{\mathscr{X}}} = \left(\sum_{t}\underset{\bm{\mathsf{O}}_{L_{t}}\bm{\mathsf{O}}_{R_{t}}}{\bm{\mathscr{M}}}\right)^{-1}\times_{\bm{\mathsf{O}}_{L}}\left(\sum_{t}\underset{\bm{\mathsf{O}}_{L_{t}}\bm{\mathsf{O}}_{R_{t}}o_t}{\bm{\mathscr{M}}}\right),
\end{align*}
\noindent and similarly perform the marginalization. This concludes our derivations.